%% file: ms.tex
\documentclass{article}

\usepackage{amsfonts}
\usepackage{graphicx}
\usepackage{float}
\usepackage{subcaption}
\usepackage{hyperref}
\usepackage{xcolor}
\usepackage{makecell}
\usepackage{printlen}
\usepackage[export]{adjustbox}
\usepackage{booktabs}

\usepackage{hyperref}

\usepackage[accepted]{icml2023}

\usepackage{amsmath}
\usepackage{amssymb}
\usepackage{mathtools}
\usepackage{amsthm}

\usepackage[capitalize,noabbrev]{cleveref}

\theoremstyle{plain}
\newtheorem{theorem}{Theorem}[section]
\newtheorem{proposition}[theorem]{Proposition}

\newtheorem{corollary}[theorem]{Corollary}
\theoremstyle{definition}

\newtheorem{assumption}[theorem]{Assumption}
\theoremstyle{remark}

\usepackage[textsize=tiny]{todonotes}

\icmltitlerunning{On Sampling with Approximate Transport Maps}

\input{commands}

\begin{document}

\twocolumn[
\icmltitle{On Sampling with Approximate Transport Maps}

\icmlsetsymbol{equal}{*}

\begin{icmlauthorlist}
\icmlauthor{Louis Grenioux}{x}
\icmlauthor{Alain Oliviero Durmus}{x}
\icmlauthor{Éric Moulines}{x}
\icmlauthor{Marylou Gabrié}{x}
\end{icmlauthorlist}

\icmlaffiliation{x}{École Polytechnique}

\icmlcorrespondingauthor{Louis Grenioux}{louis.grenioux@polytechnique.edu}
\icmlcorrespondingauthor{Alain Oliviero Durmus}{alain.durmus@polytechnique.edu}
\icmlcorrespondingauthor{Éric Moulines}{eric.moulines@polytechnique.edu}
\icmlcorrespondingauthor{Marylou Gabrié}{marylou.gabrie@polytechnique.edu}

\icmlkeywords{Normalizing flows, MCMC, Sampling, Bayesian inference}

\vskip 0.3in
]

\printAffiliationsAndNotice{}  %

\begin{abstract}
  Transport maps can ease the sampling of distributions with non-trivial geometries by transforming them into distributions that are easier to handle. The potential of this approach has risen with the development of Normalizing Flows (NF) which are maps parameterized with deep neural networks trained to push a reference distribution towards a target. NF-enhanced samplers recently proposed blend (Markov chain) Monte Carlo methods with either (i) proposal draws from the flow or (ii) a flow-based reparametrization.
  In both cases, the quality of the learned transport %
  conditions performance.
  The present work clarifies for the first time the relative strengths and weaknesses of these two approaches. Our study concludes that multimodal targets can be reliably handled with flow-based proposals 
  up to moderately high dimensions. In contrast, methods relying on reparametrization struggle with multimodality but are more robust otherwise in high-dimensional settings and under poor training.
  To further illustrate the influence of target-proposal adequacy, we also derive a new quantitative bound for the mixing time of the Independent Metropolis-Hastings sampler. 
\end{abstract}

\section{Introduction}
Creating a transport map between an intractable distribution of interest and a tractable reference distribution can be a powerful strategy for facilitating inference. Namely, if a bijective map $T: \rset^d \to \rset^d$ transports a tractable distribution $\rho$ on $\rset^d$ to a target $\pi$ on the same space, then the expectation of any test function $f: \rset^d \to \R$ under the target distribution can also be written as the expectation value under the reference distribution
\begin{align*}
    \pi(f) = \int_{\rset^d} f(x) \rmd\pi(x) = \int_{\rset^d} f(T(x)) |J_T(x)| \rmd\base(x)\, ,
\end{align*}
where $J_T$ is the Jacobian determinant of $T$.
However, the intractability of the target is traded here with the difficult task of finding the map $T$.

According to Brenier's theorem \cite{Brenier1991}, if $\rho$ is absolutely continous 
then an exact mapping $T$ between $\rho$ and $\pi$ always exists.
Such a map may be known analytically, as in some field theories in physics \cite{luscher_trivializing_2009}. Otherwise, it can be approximated by optimizing a parameterized version of the map.
While learned maps typically suffer from approximation and estimation errors, a sufficiently accurate approximated map is still valuable when combined with a reweighting scheme such as a Markov chain Monte Carlo (MCMC) or Importance Sampling (IS).

The choice of the parametrization 
must ensure that the mapping is invertible and that the Jacobian determinant
remains easy to compute. Among the first works combining approximate transport and Monte Carlo methods, \cite{Parno2018transport_mcmc} proposed using triangular maps. Over the years, the term Normalising Flow, introduced initially to refer to a Gaussianizing map \cite{Tabak2010}, has become a common name for highly flexible transport maps, usually parameterized with neural networks, that allow efficient computations of inverses and Jacobians \cite{Papamakarios2021,Kobyzev2021}. NFs were developed in particular for generative modelling and are now also a central tool for Monte Carlo algorithms based on transport maps.

While the issue of estimating the map is of great interest in the context of sampling, it is not the focus of this paper; see e.g. \cite{rezende_variational_2015, Parno2018transport_mcmc,muller_neural_2019,Noe2019}. Instead, we focus on comparing the performance of algorithmic trends among NF-enhanced samplers developed simultaneously. On the one hand, flows have been used as reparametrization maps that improve the geometry of the target before running local traditional samplers such as Hamiltonian Monte Carlo (HMC) \cite{Parno2018transport_mcmc,Hoffman2019neutra,Noe2019,Cabezas2022}. We refer to these strategies as \emph{\reparamalgs} methods. On the other hand, the push-forward of the NF base distribution through the map has also been used as an independent proposal in IS \cite{muller_neural_2019,Noe2019}, an approach coined \emph{\isalgs}, and in MCMC updates \cite{Albergo2019,Gabrie2022,Samsonov2022} among others. We refer to the latter as \emph{\propalgs} methods.

Despite the growing number of applications of NF-enhanced samplers, such as in Bayesian inference \cite{karamanis_accelerating_2022,wong_flowmc_2022}, Energy Based Model learning \cite{Nijkamp2020}, statistical mechanics, \cite{McNaughton2020}, lattice QCD \cite{abbott_gauge-equivariant_2022} or chemistry \cite{mahmoud_accurate_2022},
a comparative study of the methods of \reparamalgs, \propalgs~and \isalgs~is lacking. The present work fills this gap:
\vspace{-0.3cm}
\begin{itemize}
    \item We systematically compare the robustness of algorithms with respect to key performance factors: imperfect flow learning, poor conditioning and complex geometries of the target distribution,  multimodality and high-dimensions (Section \ref{sec:exp_intuitions}).
    \vspace{-0.2cm}
    \item We show that \propalgs~and \isalgs~can handle mutlimodal distributions up to moderately high-dimensions while \reparamalgs~is hindered in mixing between modes by the approximate nature of learned flows. 
    \vspace{-0.2cm}
    \item For unimodal targets, we find that \reparamalgs~is more reliable than \propalgs~and \isalgs~ given low-quality flows.
    \vspace{-0.2cm}
    \item We provide a new theoretical result on the mixing time of the independent Metropolis-Hastings (IMH) sampler by leveraging for the first time, to the best of our knowledge, a local approximation condition on the importance weights (Section \ref{sec:theory}). 
    \vspace{-0.2cm}
    \item Intuitions formed on synthetic controlled cases are confirmed in real-world applications (Section \ref{sec:realtasks}).
\end{itemize}

The code to reproduce the experiments is available at \href{https://github.com/h2o64/flow_mcmc}{https://github.com/h2o64/flow\_mcmc}.

\section{Background}

\subsection{Normalizing flows}
\label{sec:nf-background}

Normalizing flows are a class of probabilistic models combining a $C^1$-diffeomorphism $T : \rset^d \to \rset^d$ and a fully tractable probability distribution $\base$ on $\rset^d$. The \emph{push-forward} of $\base$ by $T$ is defined as the distribution of $X = T(Z)$ for $Z \sim \base$ and has a density given by the change of variable as
\begin{align}
    \pushforward{T}{\base}(x) = \base(T^{-1}(x)) |J_{T^{-1}}(x)| \, ,
\end{align}
where $|J_T|$ denotes the Jacobian determinant of $T$. Similarly, given a probability density $\pi$ on $\rset^d$, the \emph{push-backward} of $\pi$ through $T$ is defined as the distribution of $Z = T^{-1}(X)$ for $X \sim \pi$ and has a density given by $\pushforward{T^{-1}}{\pi}$. 

A parameterized family of $C^1$-diffeomorphisms $\{\Talpha\}_{\alpha \in \mathbb{A}}$ then defines a family of distributions $\{\pushforwardalpha\}_{\alpha \in \mathbb{A}}$. This construction has recently been popularised by the use of neural networks \cite{Kobyzev2021,Papamakarios2021} for generative modelling and sampling applications.

In the context of sampling, the flow is usually trained so that the push-forward distribution approximates the target distribution %
i.e., $\pushforwardalpha \approx \target$. As will be discussed in more detail below, the flow can then be used as a reparametrization map or for drawing proposals. Note that we are interested in situations where samples from $\target$ are not available a priori when training flows for sampling applications. In that case, the reverse Kullback-Leiber divergence (KL)
\begin{align}
    \mathrm{KL}(\pushforwardalpha || \pi) = \int \log(\pushforwardalpha(x) / \pi(x)) \pushforwardalpha(x) \rmd x
\end{align}
can serve as a training target, since it can be efficiently estimated with i.i.d. samples from $\base$. Minimizing the reverse KL amounts to variational inference with a NF candidate model \cite{rezende_variational_2015}. This objective is also referred to in the literature as \emph{self-training} or \emph{training by energy}; it is notoriously prone to mode collapse \cite{Noe2019,jerfel_variational_2021,hackett_flow-based_2021}. On the other hand, the forward KL
\begin{align}
    \mathrm{KL}(\pi || \pushforwardalpha) = \int \log(\target(x)/\pushforwardalpha(x)) \target(x) \rmd x
\end{align}
is more manageable but more difficult to estimate because it is an expectation over $\target$. Remedies include importance reweighting \cite{muller_neural_2019}, adaptive MCMC training \cite{Parno2018transport_mcmc, Gabrie2022}, and sequential approaches to the target distribution \cite{McNaughton2020,arbel_annealed_2021,karamanis_accelerating_2022,Midgley2022}.
Regardless of which training strategy is used, the learned model $\pushforwardalpha$ always suffers from approximation and estimation errors with respect to the target $\target$. However, the approximate transport map $\Talpha$ can be used to produce high quality Monte Carlo estimators using the strategies described in the next Section.

\subsection{Sampling with transport maps} \label{sec:sampling_with_T}
Since NFs can be efficiently sampled from, they can easily be integrated in Monte Carlo methods relying on proposal distributions, such as IS and certain MCMCs. 

\paragraph{\isalgs}
Importance Sampling uses a tractable proposal distribution, here denoted by $\lambda$, to calculate expected values with respect to $\pi$ \cite{Tokdar2010}. Assuming that the support of $\pi$ is included in the support of $\lambda$, we denote 
\begin{equation}
\label{eq:def_weight_function}
    w(x) = \pi(x)/\lambda(x)
\end{equation}
the importance weight function, and define the self-normalized importance sampling estimator (SNIS) \cite{Robert2005} of the expectation of $f$ under $\pi$ as
$$ \hat{\pi}_N(f) = \sum_{i=1}^N w_N^i f(X^i)$$
\noindent where $X^{1:N}$ $N$ are i.i.d. samples from $\lambda$ and
\begin{align}
    \label{eq:self-normed-weights}
     w_N^i = \left. w(X^i) \middle/ \sum_{j=1}^N w(X^j) \right.
\end{align} are the self-normalized importance weights.
For IS to be effective, the proposal $\lambda$ must be close enough to $\pi$ in $\chi$-square distance (see \citep[Theorem~1]{Agapiou2017}), which makes IS also notably affected by the curse of dimensionality (e.g., \citep[Section 2.4.1]{Agapiou2017}).  

Adaptive IS proposed to consider parametric proposals adjusted to match the target as closely as possible. NFs are  suited to achieve this goal: they define a manageable push-forward density, can be easily sampled, and are very expressive. IS using flows as proposals is known as Neural-IS \cite{muller_neural_2019} or Boltzmann Generator \cite{Noe2019}. NFs were also used in methods building on IS to specifically estimate normalization constants \cite{jia_normalizing_2019,ding_deepbar_2021,wirnsberger_targeted_2020}.

\paragraph{Flow-based independent proposal MCMCs}
Another way to leverage a tractable $\pushforwardalpha \approx \target$ is to use it as a proposal in an MCMC with invariant distribution $\target$. In particular, the flow can be used as a proposal for IMH.

Metropolis-Hastings is a two-step iterative algorithm relying on a proposal Markov kernel, here denoted by $Q(x^{(k)}, \rmd x) = q(x^{(k)},x) \rmd x$. At iteration $k + 1$ a candidate $x$ is sampled from $Q$ conditionally to the previous state $x^{(k)}$ and the next state is set according to the rule 
\begin{align}
    x^{(k+1)} = \left\{ \begin{array}{ll}
         x & \text{w. prob. }  \mathrm{acc}(x^{(k)}, x)   \\
         x^{(k)} & \text{w. prob. }  1 -\mathrm{acc}(x^{(k)}, x)
    \end{array}
    \right. ,
\end{align}
where, given a target $\target$, the acceptance probability is
\begin{align}
    \mathrm{acc}(x^{(k)}, x) = \min\left(1,\frac{q(x,x^{(k)}) \target(x)}{q(x, x^{(k)})\target(x^{(k)})}\right).
\end{align}
To avoid vanishing acceptance probabilities, the Markov kernel is usually chosen to propose local updates, as in Metropolis-adjusted Langevin (MALA) \cite{Roberts1996} or Hamiltonian Monte Carlo (HMC) \cite{duane_hybrid_1987, Neal2011}, which exploit the local geometry of the target. These local MCMCs exhibit a tradeoff between update size and acceptability leading in many cases to a long decorrelation time for the resulting chain. Conversely, NFs can serve as non-local proposals in the \emph{independent} Metropolis Hastings setting $q(x, x') = \pushforwardalpha(x')$.

Since modern computer hardware allows a high degree of parallelism, it may also be advantageous to consider Markov chains with multiple proposals at each iteration, such as multiple-try Metropolis \cite{Liu2000, Craiu2007} or iterative sampling-importance resampling (i-SIR) \cite{Tjelmeland2004, Andrieu2010}. The latter has been combined with NFs in 
\cite{Samsonov2022}. Setting the number of parallel trials to $N > 1$ in i-SIR, $N-1$ propositions are drawn at iteration $(k+1)$:
$$
    x_{l} \sim \pushforwardalpha \text{ for } l=2 \cdots N
$$
and $x_{1}$ is set equal to the previous state $x^{(k)}$. The next state $x^{(k+1)}$ is drawn from the set $\{x_l\}_{l=1}^N$ according to the self-normalized weights $w^l_N$ computed as in \eqref{eq:self-normed-weights}. If $x^{(k+1)}$ is not equal to $x^{(k)}$, the MCMC state has been fully refreshed. 

In what follows, we refer to MCMC methods that use NFs as independent proposals, in IMH or in i-SIR, as \emph{\propalgs}. These methods suffer from similar pitfalls as IS. If the proposal is not ``close enough" to the target, the importance function $w(x) = \target(x) / \pushforwardalpha(x)$ fluctuates widely, leading to long rejection streaks that become all the more difficult to avoid as the dimension increases. 

\paragraph{Flow-reparametrized local-MCMCs}
A third strategy of NF-enhanced samplers is to use the reverse map $\Talpha^{-1}$ defined by the flow to transport $\pi$ into a distribution 
which, it is hoped, will be easier to capture with standard local samplers such as MALA or HMC. This strategy was discussed by \cite{Parno2018transport_mcmc} in the context of triangular flows and by \cite{Noe2019} and \cite{Hoffman2019neutra} with modern normalizing-flows, we keep the denomination of \emph{\reparamalgs}~from the latter. 
\reparamalgs~amounts to sampling the push-backward of the target $\pushbackwardalpha$ with local MCMCs before mapping the samples back through $\Talpha$. It can be viewed as a reparametrization of the space or a spatially-dependent preconditioning step that has similarities to Riemannian MCMC \cite{Girolami2011,Hoffman2019neutra}. Indeed, local MCMCs notoriously suffer from poor conditioning. For example, one can show that MALA has a mixing time\footnote{See Section \ref{sec:theory} for a definition of the mixing time.} scaling as $O(\kappa \sqrt{d})$ - where $\kappa$ is the conditioning number of the target distribution - provided the target distribution is log-concave \cite{Wu2021, Chewi2020, Dwivedi2018}.

Nevertheless, \reparamalgs~does not benefit from the fast decorrelation of \propalgs~methods, since updates remain local.
Still, local-updates might be precisely the ingredient making \reparamalgs~escape the curse of dimensionality in some scenarios. Indeed, if the distribution targeted is log-concave, the mixing time of MALA mentioned above depends only sub-linearly on the dimension. The question becomes, when do scaling abilities of \reparamalgs s provided by locality allow to beat \isalgs~and \propalgs s?

In a recent work, \cite{Cabezas2022} also used a flow reparametrization for the Elliptical Slice Sampler (ESS) \cite{Murray2010} which is gradient free and parameter free\footnote{\cite{Cabezas2022} uses ESS with a fixed covariance parameter $\Sigma = I_d$. This choice is justified by the fact that using the \reparamalgs~trick amounts to sampling something close to the base of the flow which is $\mathcal{N}(0,I_d)$}. Notably, ESS is able to cross energy barriers to tackle multimodal targets more efficiently than MALA or HMC  \cite{Natarovskii2021}.

In our experiments we include different versions of \reparamalgs s and indicate in parentheses the sampler used on the push-backward: either MALA, HMC or ESS.

\section{Synthetic case studies} 
\label{sec:exp_intuitions}

\isalgs, \reparamalgs~and \propalgs~build on 
well-studied Monte Carlo scheme with known strengths and weaknesses (see \cite{rubinstein_simulation_2017} for a textbook).
Most of their limitations would be lifted if an exact transport between the base distribution and the target were available, as sampling from the flow would directly amounts to sampling from the target distribution. However, learned maps are imperfect, which leaves open a number of questions about the expected performance of NF-enhanced samplers: Which of the methods is most sensitive to the quality of the transport approximation? How do they work on challenging multimodal destinations? And how do they scale with dimension? In this Section, we present systematic synthetic case studies answering the questions above.

In all of our experiments, we selected the samplers' hyperparameters by optimizing case-specific performance metrics. The length of chains was chosen to be twice the number of steps required to satisfy the $\hat{ R}$ diagnosis \cite {Gelman1992} at the 1.1-threshold for the fastest converging algorithm. We used MALA as local sampler as it is suited for the log-concave distributions considered, faster and easier to tune than HMC. Evaluation metrics are fully descibed in App. \ref{app:sliced_metrics}.

\subsection{\reparamalgs s are robust to imperfect training} 
\label{sec:unimodal}

\begin{figure*}[t]
    \begin{subfigure}{0.40\linewidth}
        \centering
        \raisebox{-2cm}{\includegraphics[width=\linewidth]{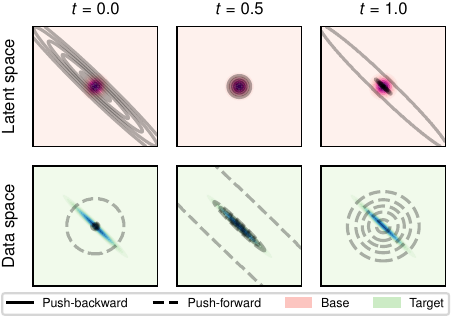}}
        \label{fig:3flows:t}
    \end{subfigure}
    \begin{subfigure}{0.55\linewidth}
        \centering
        \raisebox{-2cm}{\includegraphics[width=\linewidth]{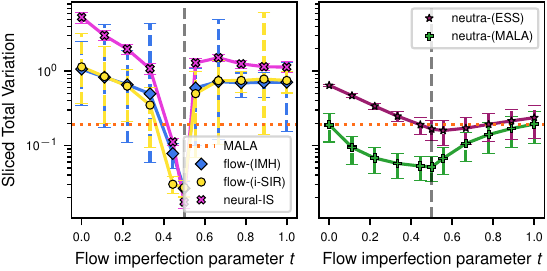}}
        \label{fig:3flows:sliced_tv_dim_128}
    \end{subfigure}
    \caption{\textbf{(Left)} \textbf{Push-backwards $\pushforward{T_t}{\base}$ and push-forwards
    $\pushforward{T^{-1}_t}{\pi}$ as a function of the flow imperfection parameter $t$.}
    \textbf{(Right) Sliced TV distances of different samplers depending on the quality of the flow $t$}, using 256 chains of length 1400 initialized with draws from NF with $T_t$. \isalgs~was evaluated with 14000 samples. Results were qualitatively unchanged for $d=16,32,64,256$.
    }
    \label{fig:3flows:main}
\end{figure*}

Provided an exact transport is available, drawing independent samples from the base and pushing them through the maps generates i.i.d. samples from the target. This is equivalent to running \isalgs~and finding that importance weights \eqref{eq:def_weight_function} are uniformly equal. 
With an approximate transport map, it is not clear which sampling strategy to prefer. 

In practice, the quality of the learned flow depends on many parameters: expressiveness, training objective, optimization procedure etc.
In the first experiments that we now present, we design a framework in which the quality of the flow can be adjusted manually.

Our first target distribution is a multivariate Gaussian ($d=128$) with an ill-conditioned covariance. We define an analytical flow with a scalar quality parameter $t \in [0,1]$ leading to a perfect transport at $t=0.5$ and an over-concentrated (respectively under-concentrated) push-forward at $t=0$ (resp. $t=1$)\footnote{This mimics the behavior when fitting the closest Gaussian of type $\mathcal{N}(0,\sigma I_d)$ with the forward KL if $t =0$ and with the backward KL if $t =1$.}, as shown on Fig. \ref{fig:3flows:main}. All the experimental details are reported in App \ref{app:3flows}.

For the perfect flow, \isalgs~yields the most accurate samples as expected, closely followed by \propalgs s. More interestingly, \reparamalgs~(MALA) quickly outperforms \propalgs~as the flow shifts away from the exact transport, both towards over-spreading and over-concentrating the mass (Fig \ref{fig:3flows:main}). A low-quality flow leads rapidly to zero-acceptance of NF proposals or very poor participation ratios\footnote{The participation ratio of a sample of $N$ IS proposal  $1 / {\sum_{i=1}^{N} {w^i_N}^2} \in [1, N]$ tracks the number of samples contributing in practice to the computation of an SNIS estimator.} for IS which translates into \isalgs~and \propalgs~being even less efficient that MALA (see Fig. \ref{fig:3flows_funnel:wierd_energies} in App. \ref{app:wierd_energies}). Conversely, \reparamalgs s are found to be robust as imperfect pre-conditioning is still an improvement over a simple MALA. 
These findings are confirmed by repeating the experiment on Neal's Funnel distribution in App. \ref{app:funnel} (Fig. \ref{fig:funnel:main}), 
for which the conditioning number of  the target 
distribution highly fluctuates over its support.

Finally note that \propalgs~methods using a multiple-try scheme, here i-SIR, remain more efficient than \reparamalgs~for a larger set of flow imperfection compared to the single-try IMH scheme. This advantage is understandable: an acceptable proposal is more likely to be available among multiple tries (see Fig. \ref{fig:3flows_funnel:wierd_energies} in App. \ref{app:3flows}). If  multiple-try schemes are more costly per iteration, wall-clock times may be comparable thanks to parallelization (see Fig. \ref{fig:real_exp:speedup} in App. \ref{app:computational}).

\subsection{\reparamalgs~may not mix between modes} 
\label{sec:mog}
MCMCs with global updates or IS can effectively capture multiple modes, provided the proposal distribution is well matched to the target. MCMCs with local updates, on the other hand, usually cannot properly sample multimodal targets due to a prohibitive mixing time\footnote{Indeed, the Eyring-Kramers law shows that the exit time of a bassin of attraction of the Langevin diffusion has expectation which scales exponentially in the depth of the mode; see e.g. \cite{Bovier2005} and the reference therein.}. Therefore, the performance in multimodal environments of \reparamalgs~methods coupled with local samplers depends on the ability of the flow to lower energy barriers while keeping a simple geometry in the push-backward of the target. %

We first examined the push-backward of common flow architectures trained by likelihood maximization on 2d target distributions. Both for a mixture of 4 isotropic Gaussians (Fig. \ref{fig:highdim_mog:recap} left) and for the Two-moons distribution (Fig. \ref{fig:many_flows} in App. \ref{app:two_moons}), the approximate transport map is unable to erase energy barriers and creates an intricate push-backward landscape. Visualizing chains in latent and direct spaces makes it is clear that: \reparamalgs~(MALA) mixing is hindered by the complex geometry, the gradient free \reparamalgs~(ESS) mixes more successfully, while \propalgs~(i-SIR) is even more efficient.

We systematically extended the experiment on 4-Gaussians by training a RealNVP \cite{Dinh2017density} with maximum likelihood for a scale of increasing dimensions (see App. \ref{app:highdim_mog} for all experiment details).
We tested the relative ability of the \isalgs, \reparamalgs~and \propalgs~algorithms to %
represent the relative weights of each Gaussian component
by building histograms of the visited modes within a chain and comparing them with the perfect uniform histogram using a median squared error (Fig. \ref{fig:highdim_mog:recap} middle). As dimension increases and the quality of the transport map presumably decreases, the ability of \reparamalgs~(MALA) to change bassin worsens. Using \reparamalgs~with ESS enables an approximate mixing between modes but only \propalgs s recover the exact histograms robustly up to $d=256$. In other words, dimension only heightens the performance gap between NF-enhanced samplers observed in small dimension. Note however that the acceptance of independent proposals drops with dimension such that \propalgs~is also eventually limited by dimension (Fig. \ref{fig:highdim_mog:global_acc_avg_kl} in App. \ref{app:highdim_mog}). Not included in the plots for readability, \isalgs~behaves similarly to \propalgs.

In fact, it is expected that exact flows mapping a unmiodal base distribution to a multimodal target distribution are difficult to learn \cite{Cornish2020}. More precisely, Theorem 2.1 of the former reference shows that the bi-Lipschitz constant\footnote{$\mathrm{BiLip}(f)$ is the bi-Lipschitz constant of $f$ is defined as the infimum over $M \in [1,\infty]$ such that $M^{-1} \norm{z - z'} \leq \norm{f(z) - f(z')} \leq M \norm{z - z'}$ for all $z$ and $z'$ different.} of a transport map between two distributions with different supports is infinite. Here we provide a complementary result on the illustrative uni-dimensional case of a Gaussian mixture target $\target$ and standard normal base $\base$:  

\begin{proposition}\label{th:prop_bogachev_bilip}
    Let $\pi = \mathcal{N}(-a,\sigma^2)/2 + \mathcal{N}(a, \sigma^2) / 2$ with $a > 0$, $\sigma > 0$ and $\base = \mathcal{N}(0, 1)$. The unique increasing flow mapping $\pi$ to $\base$ denoted $T_{\pi,\base}$ verifies that
    \begin{equation}\label{eq:bogachev_1d_lower_bound}
        \mathrm{BiLip}(T_{\pi,\base}) \geq \frac{\rmd T^{-1}_{\pi,\base}}{\rmd z}(0) = \sigma \exp\left(\frac{a^2}{2\sigma^2}\right)\eqsp.
    \end{equation}
\end{proposition}
The proof of Proposition \ref{th:prop_bogachev_bilip}, showing the exponential scaling of the bi-Lipschitz constant in the distance between modes, is postponed to Appendix \ref{app:bogachev}.

Overall, these results show that \isalgs~and \propalgs~are typically more effective for multimodal targets than \reparamalgs. Note further that it has also been proposed to use mixture base distributions \cite{izmailov_semi-supervised_2020} or mixture of NFs \cite{Noe2019, Gabrie2022,hackett_flow-based_2021} to accommodate for multimodal targets provided that prior knowledge of the modes' structure is available. These mixture models can be easily plugged in \isalgs~and \propalgs, however it is unclear how to combine them with \reparamalgs. 
Finally, mixing \reparamalgs~and \propalgs~schemes by alternating between global updates (crossing energy barriers) and flow-preconditioned local-steps (robust withing a mode) seems promising, in particular when properties of the target distributions are not known a priori. It will be referred below as \emph{\neutraflow}. A proposition along these lines was also made in \cite{Grumitt2022}. 

\begin{figure*}
    \centering
    \begin{subfigure}{0.40\linewidth}
        \centering
        \raisebox{-2cm}{\includegraphics[width=\linewidth]{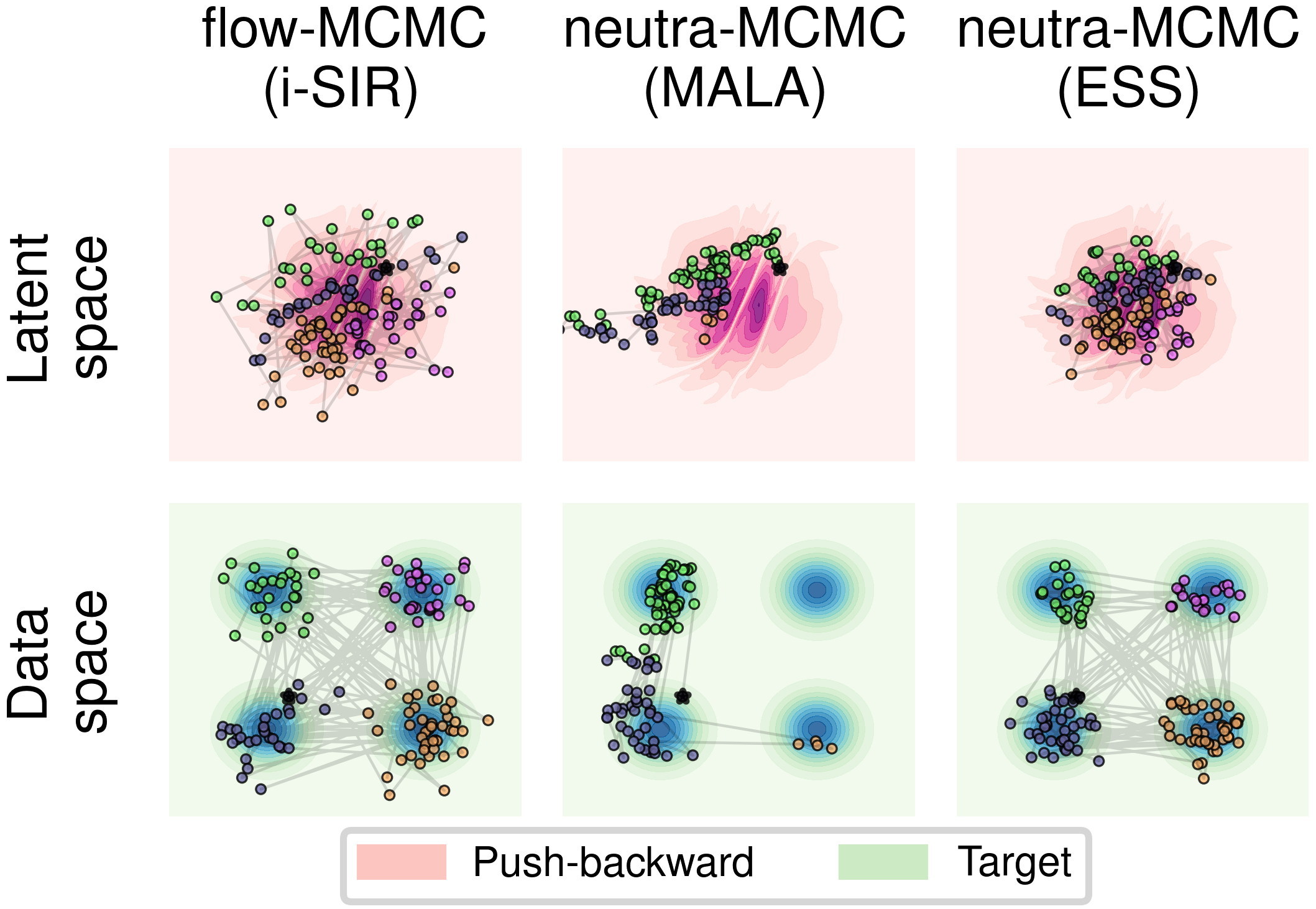}}
        \label{fig:highdim_mog:samples_neutra}
    \end{subfigure}%
    \begin{subfigure}{0.30\linewidth}
        \centering
        \raisebox{-2cm}{\includegraphics[width=\linewidth]{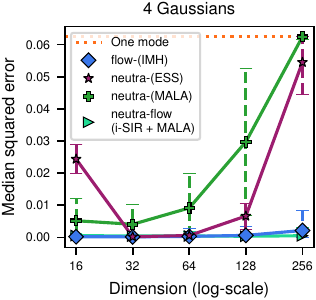}}
        \label{fig:highdim_mog:l2_dist}
    \end{subfigure}%
    \begin{subfigure}{0.30\linewidth}
        \centering
        \raisebox{-2cm}{\includegraphics[width=\linewidth]{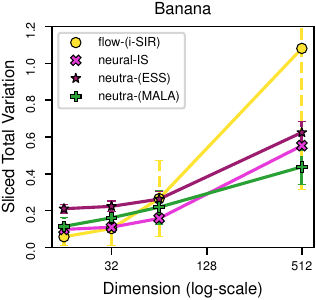}}
        \label{sec:dim_scaling}
    \end{subfigure}
    \caption{\textbf{(Left) Example chains of NF-enhanced walkers with a 2d target mixture of 4 Gaussians.} 
    The 128-step MCMC chain is colored according to the closest mode in the data space (bottom row) with corresponding location in the latent space (top row). The complex geometry of the push-backward $\pushforward{T_\alpha}{\pi}$ hinders the mixing of local-update algorithms. MALA's step-size was chosen to reach 75\% of acceptance. \textbf{(Middle)} \textbf{Median squared error of the histograms of visited modes of 4 Gaussians per chain against the perfect uniform histogram as a function of dimension.} 512 chains of 1000-steps on average were used. \textbf{(Right) Sliced total variation in sampling the Banana distribution in increasing dimension} using a RealNVP. 128 chains of 1024-steps were used.}
    \label{fig:highdim_mog:recap}
\end{figure*}

\subsection{\propalgs s are the most impacted by dimension}

To investigate the effect of dimension, we ran a systematic experiment on the Banana distribution, a unimodal distribution with complicated geometry (details in App. \ref{app:banana}). We compared NF-enhanced samplers powered by RealNVPs
trained to optimize the backward KL and found that  
a crossover occurs in moderate dimensions
: \isalgs~and \propalgs~algorithms are more efficient in small dimension but are more affected by the increase in dimensions compared to \reparamalgs~algorithms (Fig. \ref{fig:highdim_mog:recap} Right).

\section{New mixing rates for IMH}
 \label{sec:theory}

As illustrated by the previous experiment, learning and sampling are expected to be more challenging when dimension increases. To better understand the interplay between flow-quality and dimension, 
we examined the  case of the IMH sampler applied to a strongly log-concave target $\pi$ and proposal $\mathcal{N}(0,\sigma^2 I_d)$\footnote{Nevertheless, we develop a theory for a generic proposal in Section~\ref{app:proof_cond} trough Theorem~\ref{th:cond_imh}}, $\sigma >0$, with density denoted by $q_{\sigma}$. %

To this end, we consider the following assumption on the importance weight function  $w_\sigma(x) = \pi(x)/q_{\sigma}(x)$: 

\begin{assumption}
\label{ass:main_flow_quality}
For any $R \geq 0$, there exists $C_R \geq 0$ such that for any $x,y \in \ball(0,R) = \{z \, :\, \norm{z} < R\}$:
\begin{equation}
    \abs{\log w_\sigma(x) - \log w_\sigma(y)}
    \leq C_R \abs{x-y} \eqsp.
\end{equation}
\end{assumption}
 The constant $C_R$ appearing in Assumption~\ref{ass:main_flow_quality}, for $R \geq 0$, represents the quality of the  proposal with respect to the target $\pi$ locally on $\ball(0,R)$. Indeed, if $q_{\sigma} = \pi$ on $\ball(0,R)$,  this constant is zero. 
In particular, $q_{\sigma} \approx \pi$ with $q_{\sigma}$ and $\pi$ smooth and $\nabla \log w_\sigma(x) \approx 0$ on $\ball(0,R)$ would result in Assumption~\ref{ass:main_flow_quality} holding with a small constant $C_R$. 
In contrast to existing analyses of IMH which assume $w_{\sigma}$ to be uniformly bounded to obtain explicit convergence rates, here  we only assume a smooth local condition on $\log(w_{\sigma})$, namely that it is locally Lipschitz. Note this latter condition is milder than assuming the former. 
 To the best of our knowledge, it is the first time that such a condition is considered for IMH; a thorough comparison of our contribution with the literature is postponed to Section~\ref{sec:related}. 

While we relax existing conditions on $w_{\sigma}$, we restrict our study to the following particular class of targets:
\begin{assumption}
\label{ass:strong_convex_main}
The target $\pi$ is positive and $- \log \pi$ is $m$-strongly convex on $\rset^d$ and attains its minimum at $0$.
\end{assumption}

Denote by $P_{\sigma}$, the IMH Markov kernel with target $\pi$ and proposal $\mathcal{N}(0,\sigma^2 I_d)$. In our next result, we analyze the mixing time of $P_{\sigma}$, defined for an accuracy $\epsilon > 0$ and an initial distribution $\mu$ as 
\begin{equation}
   \tau_{mix}(\mu,\epsilon) = \inf \{n \in \mathbb{N} ~:~ \norm{ \mu P^n_{\sigma} - \pi}_{\TV} \leq \epsilon \} \eqsp.
\end{equation}
$\tau_{mix}$ quantifies the number of MCMC steps needed to bring the total variation distance \footnote{For two distribution $\mu,\nu$ on $\rset^d$, $\norm{\mu - \nu}_{\TV} = \sup_{\msa \in \mathcal{B}(\rset^d)} \abs{\mu(\msa) - \nu(\msa)}$.} between the Markov chain and its invariant distribution bellow $\epsilon$.

\begin{theorem}[Explicit mixing time bounds for IMH]\label{th:asympto_bound_main}
    Assume Assumption-\ref{ass:main_flow_quality}-\ref{ass:strong_convex_main} hold. Let $0 < \epsilon < 1$ and $\mu$ a $\beta$-warm distribution with respect to $\pi$ \footnote{For any Borel set $\mse$, $\mu(\mse) \leq \beta \pi(\mse)$}. Suppose in addition that $C_{R} \leq \log 2 \sqrt{m} / 32$ with
    \begin{equation}
        R \geq   C \sqrt{d}\max\left(\sigma, \frac{1}{\sqrt{m}}\right) (1+ \abs{\log^{\alpha}(\epsilon / \beta)}/d^{\alpha/2}) \eqsp,
    \end{equation}
    for some explicit numerical constant $C \geq 0$ and exponent $\alpha >0$. 
    Then the mixing time of IMH is bounded as
    \begin{equation}
        \tau_{mix}(\mu,\epsilon) \leq 128 \log\left(\frac{2\beta}{\epsilon}\right) \max\left(1, \frac{128^2 C_R^2}{\log(2)^2 m}\right)\eqsp.
    \end{equation}
\end{theorem}

The proof of Theorem \ref{th:asympto_bound_main} is postponed to App. \ref{app:proof_cond}.
It shows that if $C_R$ is bounded by a constant independent of the dimension for $R$ of order at least $\sqrt{d}$, then the mixing time is also independent of the dimension, which recovers easy consequences of existing analyses \cite{Roberts2011,Wang2022}.
In contrast to these works, Theorem~\ref{th:asympto_bound_main} can be applied to the illustrative case where $\pi = \mathcal{N}(0,I_d)$ and $\sigma = 1 + \lambda$ considering the error term $\lambda$ either positive \emph{or} negative (for which $w_\sigma$ is not uniformly bounded). In that case, Theorem~\ref{th:asympto_bound_main} shows that reaching a precision $\epsilon$ with a fixed number of MCMC steps $n$ 
requires $\lambda$ to decrease as $\mathcal{O}(1/d)$ (the detailed derivation is postponed to App~\ref{app:proof_cond}). Finally, note that we do not assume that $\pi$ is $L$-smooth, i.e., $\nabla \log \pi$ is Lipschitz in Theorem~\ref{th:asympto_bound_main}. This condition  is generally considered 
in existing results on MALA and HMC for strongly log-concave target distributions; see \cite{Dwivedi2018,Chen2020}.

\section{Related Works}
\label{sec:related}

\textbf{Comparison of NF-enhanced samplers} Several papers have investigated the difficulty of \propalgs~algorithms in scaling with dimension \cite{del_debbio_efficient_2021,abbott_aspects_2022}. Hurdles arising from multimodality were also discussed in \cite{hackett_flow-based_2021,nicoli_machine_2022} in the context of \propalgs~methods. Meanwhile, the authors of \cite{Hoffman2019neutra} argued that the success of their \reparamalgs~was bound to the quality of the flow but did not provide experiments in this direction. To the best of our knowledge, no thorough comparative study of the different NF-enhanced samplers was performed prior to this work. 

As previously mentioned, \cite{Grumitt2022} proposed to mix local NF-preconditioned steps with NF Metropolis-Hastings steps, i.e., to combine \reparamalgs~and \propalgs~methods. However, the focus of these authors was on the aspect of performing deterministic local updates using an instantaneous estimate of the density of walkers provided by the flow. More related to the present work, they present a rapid ablation study in their Appendix D.

Enhancing Sequential Monte Carlo \cite{del_moral_sequential_2006} with NFs has also been investigated by
\cite{arbel_annealed_2021, karamanis_accelerating_2022}. These methods are more involved and require the definition of a collection of target distributions approaching the final distribution of interest. They could not be directly compared to \isalgs, \reparamalgs~and \propalgs. We also note that \cite{invernizzi_skipping_2022} recently proposed another promising method to assist sampling with flows, in the context of replica exchange MCMCs. 

\textbf{IMH analysis} Most analyses establishing quantitative convergence bounds rely on the condition that the ratio $\pi/q$ be uniformly bounded \cite{Yang2021,Brown2021,Wang2022}. In these works, it is shown that IMH is  uniformly geometric in total variation or Wasserstein distances.  Our contribution relaxes the uniform boundedness condition on $\pi/q$ by restricting our study to the class of strongly log-concave targets.

The analysis of local MCMC samplers, such as MALA or HMC for sampling from a strongly log-concave target is now well developed; see e.g., \cite{Dwivedi2018, Chen2020, Chewi2020, Wu2021}. These works rely on the notion of $s$-conductance for a reversible Markov kernel and on the results developed in \cite{Lovasz1993} connecting this notion to the kernel's mixing time. This strategy has been successively applied to numerous MCMC algorithm since then; e.g., \cite{Lovasz1999, Vempala2005, Vempala2007, Karthekeyan2010, Mou2019, Cousins2014, Laddha2020, Narayanan2022}.
We follow the same path in the proof of Theorem~\ref{th:asympto_bound_main}.

Finally, while \cite{Roberts2011} establish a general convergence for IMH under mild assumption, exploiting this result turns out to be difficult. In particular, we believe it cannot be made quantitative if $\pi/q$ is unbounded since their main convergence result involves an intractable expectation with respect to the target distribution. 
 
\section{Benchmarks on real tasks}
\label{sec:realtasks}

In this Section we compare NF-enhanced samplers beyond the previously discussed synthetic examples. Our main findings hold for real world use-cases. An extra experiment on high dimension with image dataset is available in App. \ref{app:cifar10}.

\subsection{Molecular system} \label{sec:aldp}
Our first experiment is the alanine dipeptide molecular system, which consists of 22 atoms in an implicit solvent. Our goal is to capture the Boltzmann distribution at temperature $T = 300 K$ of the atomic 3D coordinates, which is known to be multimodal. We have used the flow trained in \cite{Midgley2022} to drive the samplers and generated 2d projections of the outputs in Figure~\ref{fig:aldp:samples_energies}. \reparamalgs~methods are not perfectly mixing between modes, while \propalgs~properly explores the weaker modes.
For more details, see App. \ref{app:aldp}.
\begin{figure}[t]
    \centering
        \centering
        \includegraphics[width=\linewidth]{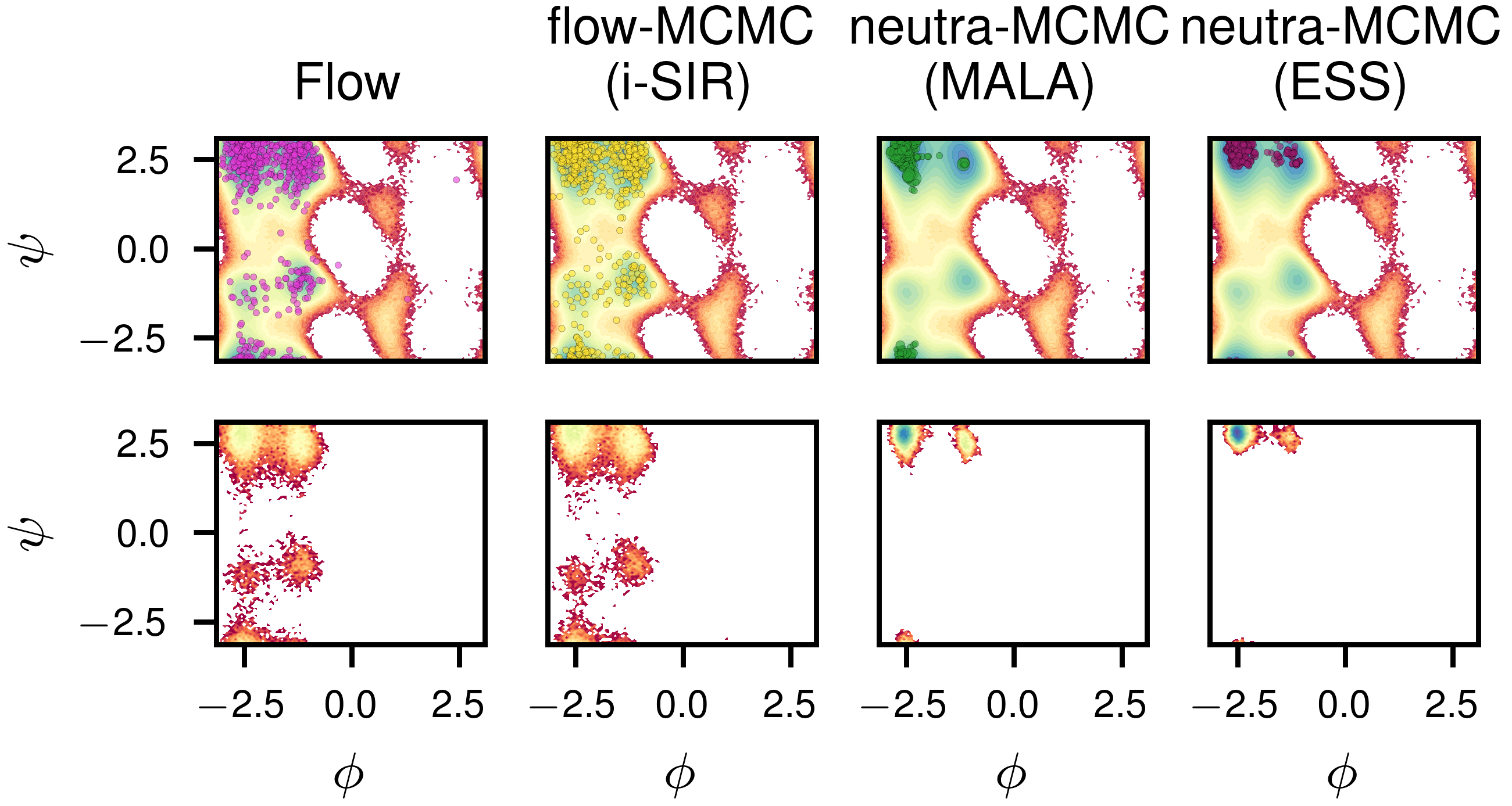}
    \caption{\textbf{Sampled configurations of alanine-dipeptide projected from 66 Cartesian coordinates to 2 dihedral angles $\phi$ and $\psi$} (see 
    App. \ref{app:aldp}). \textbf{(Top)} Samples from the flow (left) and samples from a single MCMC chain of the different NF-samplers are shown as bright-colored points on colored background displaying the log histogram of exact samples at $T=300 K$ obtained by a Replica Exchange Molecular Dynamics simulation of \cite{vincent_stimper_2022_6993124}.  \textbf{(Bottom)} Log-histograms of samples from the flow (left) and from 256 MCMC chains started at the same location. \label{fig:aldp:samples_energies}}
    \vspace{-0.3cm}
\end{figure}
\subsection{Sparse logistic regression}

\begin{table}[t]
\caption{Predictive posterior distribution for Bayesian sparse logistic regression on the German credit dataset. \label{tbl:log_reg:pred_post}}
\begin{center}
\begin{small}
\begin{tabular}{lc}
\hline
    \toprule
    \thead{Sampler} & \thead{Average predictive \\ log-posterior distribution} \\
    \midrule
    \reparamalgs~(HMC) & -191.1 $\pm$ 0.1 \\
    \neutraflow~(i-SIR + HMC)  & -194.1 $\pm$ 1.6 \\
    \propalgs~(i-SIR) & -208.5 $\pm$ 2.1 \\
    HMC & -209.7 $\pm$ 1.0 \\
    \bottomrule
\end{tabular}
\end{small}
\end{center}
\vspace{-0.5cm}
\end{table}

Our second experiment is a sparse Bayesian hierarchical logistic regression on the German credit dataset \cite{Dua2019}, which has been used as a benchmark in recent papers \cite{Hoffman2019neutra, Grumitt2022, Cabezas2022}.
We trained an Inverse Autoregressive Flow (IAF) \cite{Papamakarios2017} using the procedure described in \cite{Hoffman2019neutra}. More details about the sampled distribution and the construction of the flow are given in App.~\ref{app:log_reg}. We sampled the posterior predictive distribution on a test dataset and reported the log-posterior predictive density values for these samples in Table~\ref{tbl:log_reg:pred_post}. \reparamalgs~methods achieve higher posterior predictive values compared to \propalgs~methods, which differ little from HMC. Note that \neutraflow, alternating between \propalgs~and \reparamalgs, does not improve upon \reparamalgs.

\subsection{Field system}

In our last experiment we investigate the 1-d $\phi^4$ model used as a benchmark in \cite{Gabrie2022}. This field system has two well-separated modes at the chosen temperature. Defined at the continuous level, the field can be discretized with different grid sizes, leading to practical implementations in different dimensions. We trained a RealNVP in $64$, $128$, and $256$ dimensions by minimizing an approximated forward KL (more details on this procedure in App.~\ref{app:phi_four}). Consistent with the results of Section~\ref{sec:mog}, \reparamalgs~(MALA)~chains remain in the modes in which they were initialized, \reparamalgs~(ESS)~ crosses over to the other mode rarely while \propalgs~ is able to mix properly (see Fig.~\ref{fig:phi_four:recap} left).
To further examine performance as a function of dimension, we considered
the distribution restricted to the initial mode only and calculated the sliced total variation of samplers' chain compared to exact samples (Fig. \ref{fig:phi_four:recap} right). \reparamalgs~methods appear to be less accurate here than \propalgs. Even within a mode, the global updates appear to allow for more effective exploration. Both approaches suffer as dimensions grow.

\begin{figure}[t]
    \centering
    \begin{subfigure}{0.5\linewidth}
        \centering
        \raisebox{0.6cm}{\includegraphics[width=0.90\linewidth]{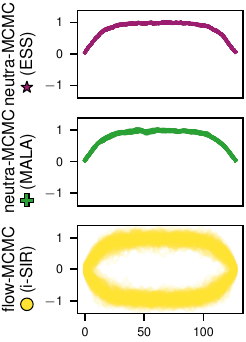}}
        \label{fig:phi_four:samples_128}
    \end{subfigure}%
    \begin{subfigure}{0.5\linewidth}
        \centering
        \includegraphics[width=\linewidth]{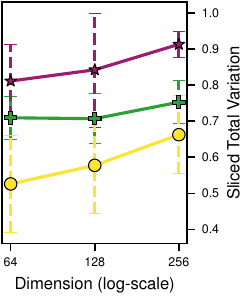}
        \label{fig:phi_four:sliced_tv}
    \end{subfigure}
    \vspace*{-10mm}
    \caption{\textbf{(Left) Sampled $\phi^4$ configurations} in dimension 128. \textbf{(Right)} Within-mode Sliced TV as a function of dimension.}
    \label{fig:phi_four:recap}
    \vspace{-0.5cm}
\end{figure}

\begin{table*}[h]
    \caption{\textbf{Informal summary of findings.} The symbol \checkm~indicates where algorithms can be expected to produce reliable samples, the symbol is starred \checkm$^*$~when faster than another suited algorithm for the task. The symbol $\mathbf{\sim}$ ~is used when the algorithm reaches its limits and may fail. A cross \textbf{\crossm}~indicates failure to sample the target. \label{table:summary}}
    \begin{center}
        \begin{small}
        \begin{tabular}{lcccccc}
        \toprule
         & 
       \multicolumn{3}{c}{Unimodal target} & \multicolumn{3}{c}{Multimodal target} \\
       \cmidrule(lr){2-4} \cmidrule(lr){5-7}
         & low-dim & mid-dim & high-dim & low-dim & mid-dim & high-dim \\
         & good map & fair map & poor map & good map &fair map & poor map \\
        \cmidrule(lr){2-4} \cmidrule(lr){5-7}
        \makecell{\propalgs \\ 
        \isalgs} & \checkm$^*$ &  \checkm & \textbf{\crossm} & \checkm &  \checkm & \textbf{\crossm} \\
        \midrule
        \reparamalgs & \checkm  & \checkm$^*$ &  $\mathbf{\sim}$ & $\mathbf{\sim}$  & \textbf{\crossm} &  \textbf{\crossm}  \\
        \bottomrule
        \end{tabular}
        \end{small}
    \end{center}
\end{table*}
 
\subsection{Run time considerations} 

In all experiments, algorithms were compared with a fixed sample-size,
yet wall-clock time and computational costs per iteration vary between samplers: \isalgs~and single-try \propalgs~require two passes through the flow per iteration, \reparamalgs s~require typically more. Multiple-try \propalgs~computational cost scales linearly with the number of trials yet can be parallelized. In App. \ref{app:computational} we report the run-time per iteration for the experiments of this section. Results show that \reparamalgs s are usually significantly slower per iteration than other methods. Nevertheless, expensive target evaluation such as in Molecular Dynamics impact in particular multiple-try \propalgs.

\section{Conclusion}
As a conclusion, we gather our findings in an informal summary table (\ref{table:summary}) of heuristics depending on the type of target and with respect to the quality of the map and dimension of the problem which typically go together (the lower the dimension, the better the learned map and vice-versa). In synthetic and real experiments, we show that 
NF provide significant advantage in sampling provided the method is chosen accordingly to the properties of the target. 
However, high-dimensional multimodal targets remain a challenge for NF-assisted samplers.

\section*{Acknowledgements}

We thank the anonymous reviewers for their valuable feedback on the paper. L.G. and M.G. acknowledge funding from Hi! Paris. The work was partly supported by ANR-19-CHIA-0002-01 “SCAI”. Part of this research has been carried out under the auspice of the Lagrange Center for Mathematics and Computing.
A.O.D. would like to thank the Isaac Newton Institute for Mathematical Sciences for support and hospitality during the programme \emph{The mathematical and statistical foundation of future data-driven engineering} when work on this paper was undertaken. This work was supported by: EPSRC grant number EP/R014604/1
\newpage
\clearpage
\bibliography{refs_marylou,refs_louis}
\bibliographystyle{icml2023}

\newpage
\appendix
\onecolumn
\input{supplement}

\end{document}

%% file: commands.tex
\usepackage{pifont}
\newcommand{\checkm}{\ding{51}}
\newcommand{\crossm}{\ding{53}}

\newcommand{\reparamalgs}{neutra-MCMC}
\newcommand{\propalgs}{flow-MCMC}
\newcommand{\isalgs}{neural-IS}
\newcommand{\neutraflow}{neutra-flow-MCMC}

\newcommand{\R}{\mathbb{R}}

\newcommand{\rset}{\mathbb{R}}

\newcommand{\base}{\rho}
\newcommand{\target}{\pi}
\newcommand{\pushforward}[2]{\lambda_{#1}^{#2}}

\newcommand{\Talpha}{T_{\alpha}}
\newcommand{\Talphainv}{\Talpha^{-1}}
\newcommand{\pushforwardalpha}{\pushforward{\Talpha}{\base}}
\newcommand{\pushbackwardalpha}{\lambda_{\Talphainv}^{\target}}

\DeclareMathOperator\erf{erf}
\newcommand{\norm}[1]{\left\lVert#1\right\rVert}
\newcommand{\abs}[1]{\left\lvert#1\right\rvert}

\def\rmd{\mathrm{d}}

\def\eqsp{\,}

\def\msa{\mathsf{A}}
\def\msb{\mathsf{B}}
\def\mse{\mathsf{E}}
\def\mss{\mathsf{S}}

\def\ball{\mathrm{B}}
\def\mix{\mathrm{mix}}
\def\TV{\mathrm{TV}}
\def\KS{\mathrm{KS}}
\def\gauss{\mathcal{N}}
\def\Idd{\operatorname{I}_d}
\def\dist{\mathrm{dist}}

\newcommand{\isEquivTo}[1]{\underset{#1}{\sim}}

\DeclareMathOperator*{\argmax}{arg\,max}

%% file: supplement.tex
\section{Perfect flows between multimodal distributions and unimodal distributions} \label{app:bogachev}

In \cite{Bogachev2005}, the authors provide a recipe to build a triangular mapping \footnote{A function $T: \mathbb{R}^n \to \mathbb{R}^n$ is triangular if $T_i(x)$ only depends on $x_1, \ldots, x_i$ for $x \in \mathbb{R}^n$ and $1 \leq i \leq n$} from one distribution to another. They do this by using the inverse CDF method for the first coordinate and then iterate it on the conditional distribution of the other coordinates. \cite{Bogachev2005} shows that this bijection is the unique increasing one. We will illustrate this method by building a triangular map between a mixture of two Gaussians and a single Gaussian in 1D and 2D.

\paragraph{Unidimensional example}

Consider $\mu = \mathcal{N}(-a,\sigma^2)/2 + \mathcal{N}(a, \sigma^2) / 2$ with $a > 0$ and $\sigma > 0$ and $\nu = \mathcal{N}(0, \tilde{\sigma}^2)$ with $\tilde{\sigma} > 0$. We can build a bijective mapping $T_{\mu,\nu}$ between the two distribution by taking $T = F^{-1}_{\nu} \circ F_{\mu}$ where $F_{\nu}$ and $F_{\mu}$ are the cumulative distribution functions of $\nu$ and $\mu$ respectively. In our case, we have 

\begin{center}
    \begin{tabular}{c c}
     $F_{\mu}(z) = \frac{1}{4} \left(2 + \erf\left(\frac{z + a}{\sigma \sqrt{2}}\right) + \erf\left(\frac{z - a}{\sigma \sqrt{2}}\right)\right)$ & $F_{\nu}(z) = \frac{1}{2} \left(1 + \erf\left(\frac{z}{\tilde{\sigma} \sqrt{2}}\right)\right)$ \\
     $F^{-1}_{\nu}(y) = \tilde{\sigma} \sqrt{2} \erf^{-1}(2y - 1)$ & $T_{\mu,\nu}(z) = F^{-1}_{\nu}(F_{\mu}(z))$
    \end{tabular}
\end{center}

\begin{figure}[t]
    \centering
    \includegraphics[width=\linewidth]{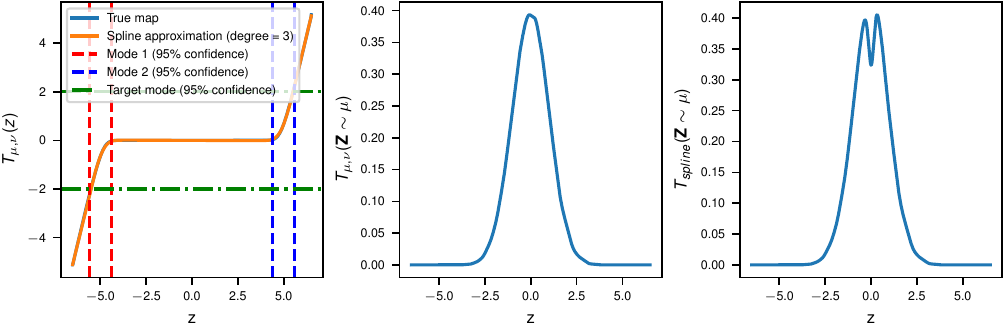}
    \caption{Flow $T_{\mu,\nu}$ in 1D - \textbf{(Left)} Map $T_{\mu,\nu}$ from the latent space (on the x-axis) to the data space (on the y-axis). Each pair of dotted lines highlight a mode in the latent space while the horinzontal line show the mode in the data space. \textbf{(Middle)} Kernel density estimation of the push forward of $\mu$ through the flow $T_{\mu,\nu}$ \textbf{(Right)} Kernel density estimation of the push forward of the base ($\mathcal{N}(0,1)$) through the flow a smooth cubic spline approximation of $T_{\mu,\nu}$}
    \label{fig:bogachev:1d_flow}
\end{figure}

Figure \ref{fig:bogachev:1d_flow} shows $T_{\mu,\nu}$ : it takes the mass from the left and push it to the middle and does the symmetric task on the other side. Note that at the edge of the modes, the curve is very sharp. This sharpness increases as the modes are getting further apart (i.e., $a \to \infty$) which makes smooth approximation (like the spline approximation shown here) more and more difficult. Those small errors on the edges of the mode lead to multimodality in the push-forward.

\paragraph{Bidimensional example}

We take a similar setup in 2 dimensions. We consider $\mu = \mathcal{N}(-a \mathbf{1}_2,\sigma^2 I_2) / 2 + \mathcal{N}(a \mathbf{1}_2, \sigma^2 I_2) / 2$ with $a > 0$ and $\sigma > 0$ and $\nu = \mathcal{N}(0, \tilde{\sigma}^2 I_2)$ with $\tilde{\sigma} > 0$. Following \cite{Bogachev2005} protocol, we define $T^{(1)}_{\mu,\nu}$ as our previous flow

$$T^{(1)}_{\mu,\nu}(z) = \tilde{\sigma} \sqrt{2} \erf^{-1}\left(\frac{1}{2} \left(2 + \erf\left(\frac{z + a}{\sigma \sqrt{2}}\right) + \erf\left(\frac{z - a}{\sigma \sqrt{2}}\right)\right) - 1\right)\eqsp,$$

because it is the canonical mapping between $\mu_1$ and $\nu_1$ which are the projections of $\mu$ and $\nu$ on the first axis. Let $\mu_x$ and $\nu_x$ be the projections of $\mu$ and $\nu$ on the second axis. We first compute $\rho_{\mu}^{x_1}$ which is the density of $\mu_{x_1}$, 
\begin{align*}
    \rho_{\mu}^{x_1}(x_2) &= \frac{\rho_{\mu,\nu}(x_1,x_2)}{\int_{\mathbb{R}} \rho_{\mu,\nu}(x_1,x_2) dx_2} \\
    &= \frac{\mathcal{N}(x_1, -a, \sigma) \mathcal{N}(x_2, -a, \sigma) / 2 + \mathcal{N}(x_1, a, \sigma) \mathcal{N}(x_2, a, \sigma) / 2}{\left(\mathcal{N}(x_1, -a, \sigma) + \mathcal{N}(x_1, -a, \sigma) \right)/2} \\
    &= w_{x_1,a,\sigma} \mathcal{N}(x_2, -a, \sigma) + (1 - w_{x_1,a,\sigma}) \mathcal{N}(x_2, a, \sigma)\eqsp,
\end{align*}

where $w_{x_1,a,\sigma} = \mathcal{N}(x_1, -a, \sigma) / (\mathcal{N}(x_1, -a, \sigma) + \mathcal{N}(x_1, -a, \sigma))$. So $\mu_{x_1}$ is a mixture between two Gaussians $\mathcal{N}(-a, \sigma)$ and $\mathcal{N}(a, \sigma)$ with weight $w_{x_1,a,\sigma}$. On the other hand, $\nu_{x_1}$ is simply $\mathcal{N}(0,\tilde{\sigma}^2)$ because it is isotropic. Following the same recipe as the unidimensional example, we can build $T^{(2)}_{\mu,\nu} = T_{\mu_{x_1},\nu_{x_1}}$ to map $\mu_{x_1}$ on $\nu_{x_1}$. Finaly, we define the bidimensional map $T_{\mu,\nu}$ to be $T_{\mu,\nu}(x,y) = (T^{(1)}_{\mu,\nu}(x), T^{(2)}_{\mu,\nu}(y))$. Figure \ref{fig:bogachev:2d_flow} show how the samples from $\mu$ are transported to $\nu$. Again, we find a very sharp border between the two modes which would lead to multimodality if not well approximated.

\begin{figure}[t]
    \centering
    \includegraphics[width=0.70\linewidth]{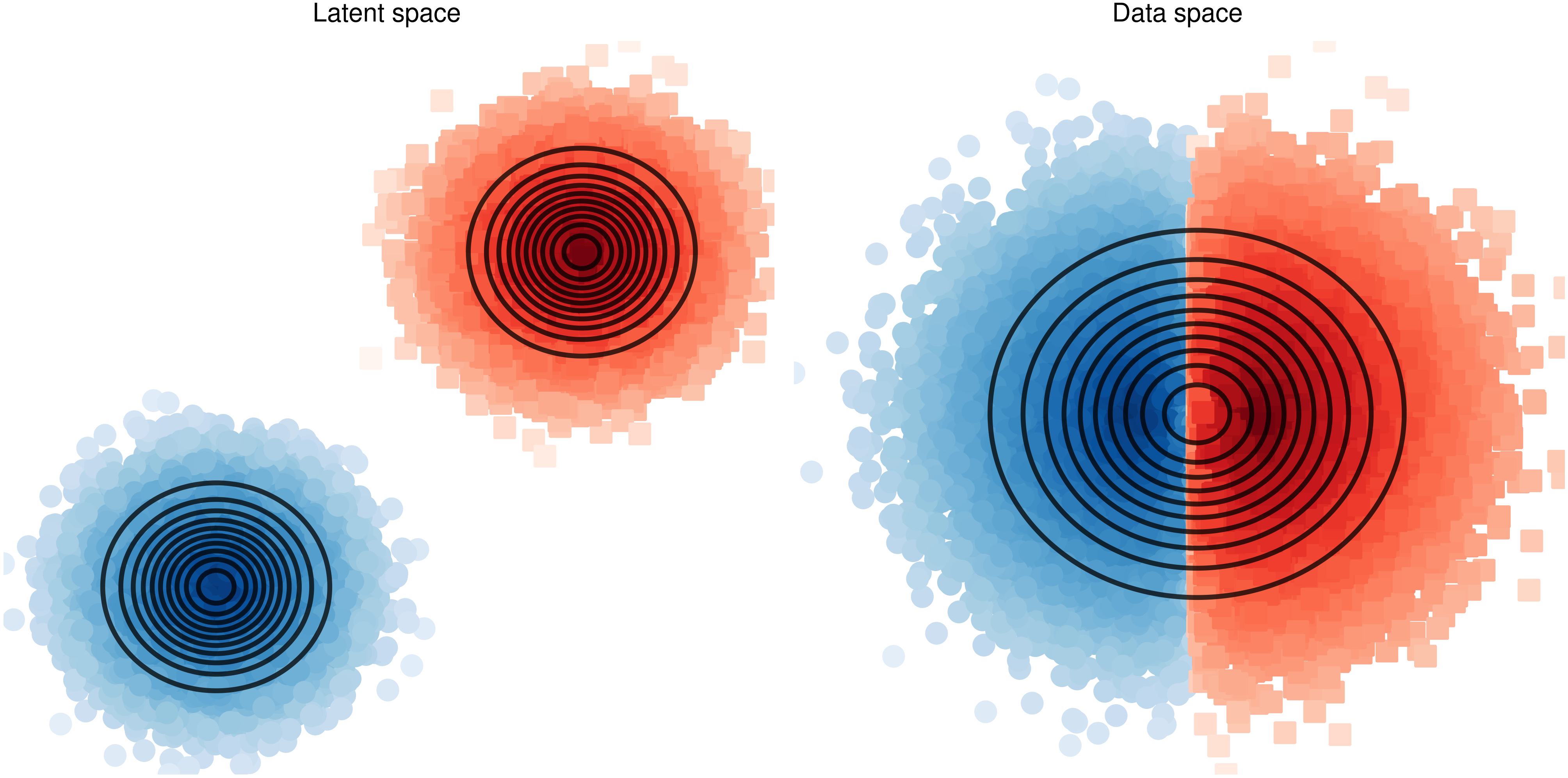}
    \caption{Flow $T_{\mu,\nu}$ in 2D - \textbf{(Left)} Samples from $\mu$ in the latent-space colored. The color of the samples is based on the closest mode. \textbf{(Right)} Samples from $\mu$ pushed through the flow $T_{\mu,\nu}$. The color of the samples correspond to their origin mode in the latent space.}
    \label{fig:bogachev:2d_flow}
\end{figure}

\paragraph{Theoretical explanation}

In \citep[Theorem~2.1]{Cornish2020}, the authors explain that if the support of the base and the target of a flow are different then in order to have a sequence $T_{\alpha_n}$ of homeomorphisms such $\pushforward{T_{\alpha_n}}{\base} \xrightarrow{\mathcal{D}} \pi$ (with the notations of Sec. \ref{sec:nf-background}) then
$$\lim_{n \to \infty} \mathrm{BiLip}(T_{\alpha_n}) = \infty\eqsp,$$
where $\mathrm{BiLip}(f)$ is the bi-Lipschitz constant of $f$ is defined as the infemum over $M \in [1,\infty]$ such that
$$M^{-1} \norm{z - z'} \leq \norm{f(z) - f(z')} \leq M \norm{z - z'}\eqsp.$$
Because classic deep normalizing flows use contraction mappings like ReLu as activation functions, their bi-Lipschitz function is necessarily bounded thus forbidding a perfect approximation of $\pi$. As mentionned by \cite{Cornish2020}, this is especially true for ResFlows \cite{Chen2019} which are built upon Lipschitz transformations.

The intuition of this theorem can be seen in the previous example, as the sharp edges at the border of the modes would require unbounded derivatives.

\paragraph{Proof of Proposition \ref{th:prop_bogachev_bilip}}

Because $d=1$, all real functions are triangular. Using \citep[Lemma~2.1]{Bogachev2005} we deduce the transport map $T_{\mu,\nu}$ that we built in the previous paragraph is actually the unique increasing flow. We now compute its bi-Lipschitz constant. We have that
\begin{align}
    \frac{\rmd F_{\mu}}{\rmd z}(z) &= \frac{1}{2\sigma\sqrt{2\pi}} \left(\exp\left(-\left(\frac{z+a}{\sigma \sqrt{2}}\right)^2\right) + \exp\left(-\left(\frac{z-a}{\sigma \sqrt{2}}\right)\right)^2\right)\eqsp, \label{eq:unidim_flow_diff_one}\\
    \frac{\rmd F_{\nu}^{-1}}{\rmd y}(y) &= \tilde{\sigma} \sqrt{2\pi} \exp([\erf^{-1}(2y-1)]^2)\eqsp, \label{eq:unidim_flow_diff_two}\\
    \frac{\rmd T_{\mu,\nu}}{\rmd z}(z) &=  \frac{\rmd F_{\nu}^{-1}}{\rmd y}(F_{\mu}(z)) \frac{\rmd F_{\mu}}{\rmd z}(z)\eqsp.\label{eq:unidim_flow_diff_three}
\end{align}
Using equations (\ref{eq:unidim_flow_diff_one})-(\ref{eq:unidim_flow_diff_three}) with $z = 0$, we have that
\begin{align*}
     \frac{\rmd T_{\mu,\nu}}{\rmd z}(0) &= \tilde{\sigma} \sqrt{2\pi} \exp\left(\left[\erf^{-1}\left(\frac{2}{4} \left(2 + \erf\left(\frac{a}{\sigma \sqrt{2}}\right) + \erf\left(\frac{a}{\sigma \sqrt{2}}\right)\right) -1\right)\right]^2\right) \\
     &\times \frac{1}{2\sigma\sqrt{2\pi}} \left(\exp\left(-\left(\frac{a}{\sigma \sqrt{2}}\right)^2\right) + \exp\left(-\left(\frac{a}{\sigma \sqrt{2}}\right)\right)^2\right) \\
     &= \frac{\tilde{\sigma}}{\sigma} \exp\left(\frac{-a^2}{2\sigma^2}\right)\eqsp.
\end{align*}
Using this result, we can derive a lower bound on $\mathrm{BiLip}(T_{\mu,\nu})$
\begin{align*}
    \mathrm{BiLip}(T_{\mu,\nu}) &\geq \mathrm{Lip}(T_{\mu,\nu}^{-1}) = 1 / \mathrm{Lip}(T_{\mu,\nu}) \geq \left(\frac{\rmd T_{\mu,\nu}}{\rmd z}(0)\right)^{-1} \\
                                &= \frac{\sigma}{\tilde{\sigma}} \exp\left(\frac{a^2}{2\sigma^2}\right)\eqsp.
\end{align*}
This shows that $\lim_{a \to \infty} \mathrm{BiLip}(T_{\mu,\nu}) = +\infty$. Proposition \ref{th:prop_bogachev_bilip} can be recovered by taking $\tilde{\sigma} = 1$.

\section{Local samplers can't cross energy barriers} \label{sec:two_modes:local}

We can illustrate that local samplers can't cross energy barriers by taking a simple uni-dimensional mixture of Gaussians as in Fig. \ref{fig:two_modes:dist_and_cv}. Figure \ref{fig:two_modes:dist_and_cv} shows that, as the energy barrier increases, it gets more and more difficult for a local sampler to sample the mixture. Independent proposal methods are able to overcome this issue despite a poorly chosen proposal.

\begin{figure}[t]
    \centering
    \begin{subfigure}{0.35\linewidth}
		\centering
		\raisebox{-2cm}{\includegraphics[width=0.95\linewidth]{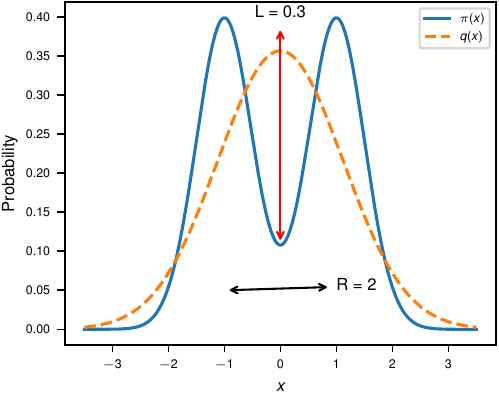}}
		\label{fig:two_modes:distributions}
    \end{subfigure}%
    \begin{subfigure}{0.55\linewidth}
		\centering
		\raisebox{-2cm}{\includegraphics[width=0.95\linewidth]{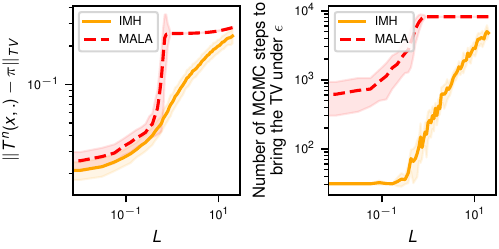}}
		\label{fig:two_modes:exp_convergence}
    \end{subfigure}
    \caption{Toy experiment with a mixture of two uni-dimensional Gaussians - \textbf{(Left)} Target distribution $\pi$ and proposal $Q = \mathcal{N}(0, \sqrt{\sigma^2 + a^2})$. The distance $L$ is the depth of the gap between the two modes. \textbf{(Middle)} Total variation metrics of 256 MCMC chains sampling $\pi$ with different values of $L$ after 8192 steps. \textbf{(Right)} Number of steps needed to bring the total variation distance of the chain being built under $\epsilon = 8 \times 10^{-2}$ for different values of $L$. Note that curves eventually plateau due to the fact that we can't sample infinitely long MCMC chains.}
    \label{fig:two_modes:dist_and_cv}
\end{figure}

\section{Flows typically do not erase potential barriers in the latent space} \label{app:two_moons}

The toy experiment from App. \ref{sec:two_modes:local} explained the difficulty of sampling multi-modal distributions directly in the data space, but could the flow kill this multi-modality in the latent space like it killed the bad conditioning before ? We trained popular deep normalizing flows on the two moons multimodal target and observed the push-forward and push-backward space. Figure \ref{fig:many_flows} shows that the flows are not erasing energy barriers and can even worsen the conditioning of the modes in the latent space compared to the data space. However, they successfully put all the mass under the support of the base of the flow which should improve the quality flow proposal based methods. The intuition behind this difficulty is hinted in App. \ref{app:bogachev}.

\begin{figure}[t]
    \centering
    \includegraphics[width=\linewidth]{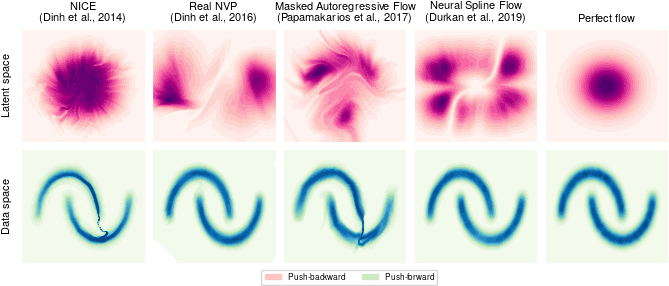}
    \caption{Push-backward and push-forward of popular deep normalizing flows targeting the two moons distribution}
    \label{fig:many_flows}
\end{figure}

\section{Additional details on synthetic examples}

\subsection{Computation of sliced distribution metrics} \label{app:sliced_metrics}

In section \ref{sec:exp_intuitions}, we used two different distances : the Total Variation (TV) distance and the Kolmogorov-Smirnov (KS) distance. They are computed between distributions $\mu$ and $\nu$ as follows
\begin{align}
    D_{\TV}(\mu,\nu) &= \sup_{A \in \mathcal{F}} \abs{\mu(A) - \nu(A)} \label{app:def_tv}\eqsp,\\
    D_{\KS}(\mu, \nu) &= \sup_x \abs{F_{\mu}(x) - F_{\nu}(x)} \label{app:def_ks}\eqsp,
\end{align}
where $F_{\mu}$ and $F_{\nu}$ are the cumulative distribution functions of $\mu$ and $\nu$ respectively. If we consider sliced versions of those metrics in Sec. \ref{sec:exp_intuitions}, it's because both of them cannot be computed in high dimension : (\refeq{app:def_tv}) require the computation of an intractable integral and (\refeq{app:def_ks}) require the multidimensional generalisation of the cumulative distribution function which is also intractable. Moreover, since $\mu$ will always be the distribution of the MCMC samples which is unknown it requires density estimation which is notoriously hard in high dimension.

To solve this problem, we project the samples in 1D before computing the metric. Let $(X_i)_{i = 1}^N$ and $(Y_i)_{i = 1}^M$ be samples from $\mu$ and $\nu$ respectively, for every random normal projection $P : \mathbb{R}^d \to \mathbb{R}$ (i.e., $P_i(x) = p_i x$ where $p_i \sim \mathcal{N}(0,1)$), we compute $\tilde{X}_i = P(X_i)$ and $\tilde{Y}_i = P(Y_i)$ and then
\begin{itemize}
    \item \emph{Sliced} Total Variation
    \begin{itemize}
        \item we perform density estimation on the obtained samples leading to densities $\tilde{\mu}$ and $\tilde{\nu}$;
        \item we compute $D_{\TV}(\tilde{\mu},\tilde{\mu})$ by doing
            $$D_{\TV}(\tilde{\mu},\tilde{\mu}) = \frac{1}{2} \int \abs{\tilde{\mu}(x) - \tilde{\nu}(x)} dx\eqsp.$$
    \end{itemize}
    \item \emph{Sliced} Kolmogorov-Smirnov
    \begin{itemize}
        \item we compute the empirical cumulative distributions of those projected samples $\hat{F}_{\mu_{proj}}(x) = 1/n \sum_{i=1}^n 1_{(-\infty,x]}(x_i)$ where $x$ browse the union support of $\mu_{proj}$ and $\nu_{proj}$ (same goes for $\hat{F}_{\nu_{proj}}(x)$);
        \item we compute the supremum of $\abs{\hat{F}_{\mu_{proj}} - \hat{F}_{\nu_{proj}}}\eqsp.$
    \end{itemize}
\end{itemize}

To be accurate this procedure requires many random projections. In this work, we always use 128 random projections. Moreover, the density estimation task needed for the sliced total variation is performed using kernel density estimation (with a gaussian kernel) where the bandwitch is selected with Scott method for all sampling algorithms except IMH and IS which use the Sheather-Jones algorithm \footnote{This is because when the flow is very far from the target, those samples tend to stagnate breaking Scott or Silverman rules.}. This density estimation task (just like the estimation of the cumulative distribution function) require to have $M \gg N$ if $\nu$ represent the true distribution and $\mu$ is an approximate distribution. In pratice, we use $M = 10 N$.

Finally, we have chosen the use of the Kolmogorov-Smirnov distance for the experiments involving Neal's Funnel as we wanted to highlight the sampling behavior in the tails of the distribution. 

To circumvent the need of slicing dimensions, a reviewer suggested to use the \emph{Maximum Mean Discrepancy} (MMD). Given a feature map $\phi : \R^d \to \mathcal{F}$ and two distributions $P$ and $Q$, the MMD is defined as
$$
	\operatorname{MMD}^2(P,Q) = \norm{\mu_P - \mu_Q}_{\mathcal{F}}^2\eqsp,
$$
where $\mu_P = (\mathbb{E}_{P}[\phi(X)_1], \ldots, \mathbb{E}_{P}[\phi(X)_m])^T$. Using the kernel trick, one can find a kernel $k : \R^d \times \R^d \to \R$ that satisfies
$$
	<\mu_P,\mu_Q>_{\mathcal{F}} = \mathbb{E}_{P,Q}[k(X,Y)]\eqsp.
$$
In the following, we'll empirically compute the MMD between two datasets $X$ and $Y$ (as feature matrix) with the following formula
$$
	MMD^{2}(X,Y) = \frac{1}{m (m-1)} \sum_{i} \sum_{j \neq i} k(x_{i}, x_{j}) - 2 \frac{1}{m.m} \sum_{i} \sum_{j} k(x_{i}, y_{j}) + \frac{1}{m (m-1)} \sum_{i} \sum_{j \neq i} k(y_{i}, y_{j})
$$
with a Gaussian kernel $k(x,y) = \exp(-\norm{x-y}^2 / (2\sigma^2)$. In practice, we choose the bandwidth $\sigma$ as the mean of the pairwise distances between each sample of $X$ and $Y$.

\subsection{Experimental details on the interpolated Gaussians} \label{app:3flows}

\paragraph{Design of the target $\pi$}

The target $\pi$ is a unimodal centered Gaussian $\mathcal{N}(0,\Sigma)$ with a badly conditioned covariance matrix $\Sigma$. We take $\Sigma = R_{\pi/4} \mathrm{diag}(\sigma_1^2, \ldots, \sigma_d^2) R_{\pi/4}^T$ where $R_{\pi/4}$ is the rotation of angle $\pi/4$ in the plan $(x_1,x_d)$ (first and last axis) and $\sigma_i$ are logarithmically evenly distributed between $10^{-1}$ and $10^1$. 

\paragraph{Design of $T_t$}

The imperfect flows are built with a piecewise linear interpolation $T_t$ using multiple canonical transformations between different multivariate Gaussians. If $\sqrt{A}$ denotes the Cholesky factor of a matrix $A$ (i.e., $\sqrt{A}\sqrt{A}^T = A$), then the flow indexed by $t$ can be expressed as  

$$
T_t : z \mapsto \begin{cases}
    ((1-2t) \sqrt{\sigma_1^2 I_d} + 2t \sqrt{\Sigma}) z & \text{ if } t < 1/2 \\
    \sqrt{\Sigma} z & \text{ if } t = 1/2 \\
    ((2t-1) \sqrt{\sigma_d^2 I_d} + 2 (1-t) \sqrt{\Sigma})z & \text{ if } t > 1/2
\end{cases}\eqsp,
$$

$$
T_t^{-1} : x \mapsto \begin{cases}
    [(1-2t) \sqrt{\sigma_1^2 I_d} + 2t \sqrt{\Sigma}]^{-1} x & \text{ if } t < 1/2 \\
    \sqrt{\Sigma}^{-1} x & \text{ if } t = 1/2 \\
    [(2t-1) \sqrt{\sigma_d^2 I_d} + 2 (1-t) \sqrt{\Sigma}]^{-1} x & \text{ if } t > 1/2
\end{cases}\eqsp.
$$

The flow $T_t$ is a linear map and its jacobians are the determinant of the scaling factors. This flow is built so that the conditioning number of the push-backward is canceled when the flow is perfect and can be as high as the one from the target if $t \neq 0.5$ (Fig. \ref{fig:3flows:cond}).

\begin{figure}[t]
    \centering
    \includegraphics[width=0.35\linewidth]{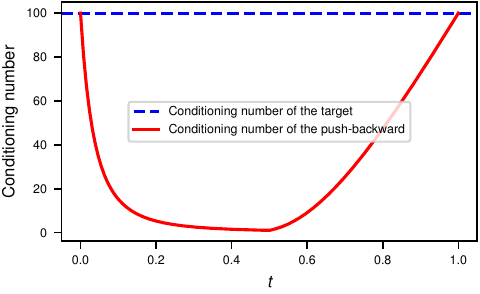}
    \caption{Conditioning of the push-backward of $\pi$ through $T_t$}
    \label{fig:3flows:cond}
\end{figure}

\paragraph{Sampling details}\label{sec:sampling_details_general}

Sampling hyperparameters where chosen by doing a grid-search for each algorithm and maximizing the median (over different quality of flows) of the sliced TV. The number of MCMC steps $n$ was also chosen so that $n$ is twice the number of steps needed to bring the $\hat{R}$ diagnostic of the faster algorithm (to converge) bellow $1.1$. The number of particles used in importance sampling is a quarter of the total number of particles used by i-SIR. This division was chosen to avoid memory problems when using non-sequential importance sampling. All details are available in table \ref{table:3flows:spl_params}. 256 chains were sampled in parallel to compute the metrics and were started by samples from the flow. 128 random projections are used to compute the sliced total variation (more details about this metric in App. \ref{app:sliced_metrics}). The step size of the local steps was chosen to maintain the acceptance rate at 75\%.

\begin{table}[t]
    \caption{Sampling hyper-parameters for the interpolated Gaussians experiment - $n$ is the number of MCMC steps in the chain, $n_{local}$ is the number of interleaved local steps in global/local samplers and $N$ is the number of particles used in importance sampling.}
    \label{table:3flows:spl_params}
    \begin{center}
        \begin{small}
            \begin{tabular}{c | c | c | c | c | c | c }
                \toprule
                \thead{Dimension} & \thead{$n$} & \thead{$N$ \\ (\propalgs)} & \thead{$N$ \\ (\neutraflow)} & \thead{$N$ \\ (IS)} & \thead{$n_{local}$ \\ (\propalgs)} & \thead{$n_{local}$ \\ (\neutraflow)} \\ 
                \midrule
                16 & 1100 & 80 & 80 & 22000 & 50 & 50 \\
                32 & 1200 & 80 & 80 & 24000 & 50 & 50 \\
                64 & 1300 & 80 & 80 & 26000 & 50 & 50 \\
                128 & 1400 & 80 & 80 & 28000 & 50 & 50 \\
                256 & 1500 & 80 & 80 & 30000 & 50 & 50 \\
                \bottomrule
            \end{tabular}
        \end{small}
    \end{center}
\end{table}

\paragraph{Maximum Mean Discrepancy distance}

\begin{figure}[t]
	\centering
	\includegraphics[width=0.5\linewidth]{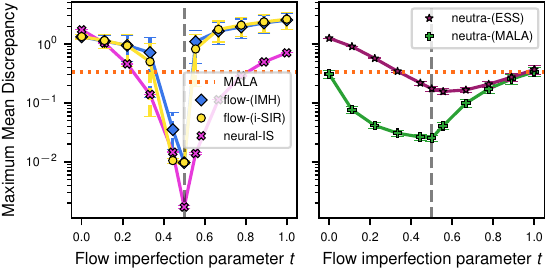}
	\caption{Replicate of Fig. \ref{fig:3flows:main} (Left) on the interpolated Gaussians with MMD distance.}
	\label{fig:3flows:mmd}
\end{figure}

As suggested by one of the reviewer, we computed the MMD distance - see end of App. \ref{app:sliced_metrics} - instead of the Sliced TV - see Fig. \ref{fig:3flows:main} (Left). Fig. \ref{fig:3flows:mmd} shows no apparent differences.

\subsection{Experimental details on Neal's Funnel} \label{app:funnel}

\begin{figure}[t]
    \centering
    \begin{subfigure}{0.40\linewidth}
        \centering
        \raisebox{-2cm}{\includegraphics[width=\linewidth]{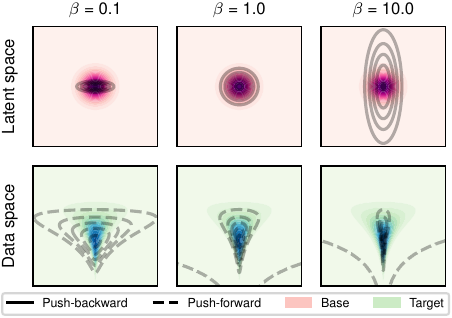}}
        \label{fig:funnel:beta}
    \end{subfigure}
    \begin{subfigure}{0.55\linewidth}
        \raisebox{-2cm}{\includegraphics[width=\linewidth]{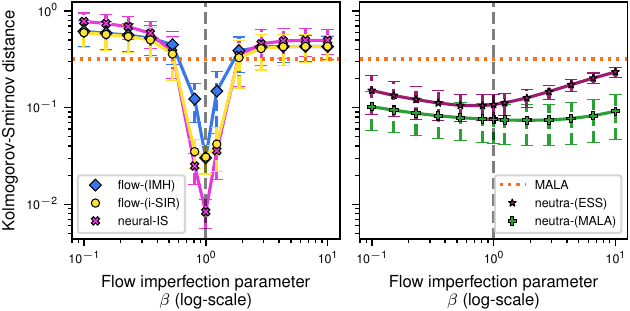}}
        \label{fig:funnel:ks_dist_dim_128}
    \end{subfigure}
    \caption{\textbf{(Left)} \textbf{Push-backwards $\pushforward{T_t}{\base}$ and push-forwards
    $\pushforward{T^{-1}_t}{\pi}$ as a function of the flow imperfection parameter $t$.} Bottom row compares ill-conditioned Gaussian target (blue levels) and flow push-forward (black lines) projected on smallest and largest variance axes. Top row compares base distribution $\mathcal{N}(0,I_d)$ (pink level) with target push-backward (black lines). \textbf{(Right)} \textbf{Kolmogorov-Smirnov (KS) distances between target and empirical samples depending on the quality of the flow $\beta$} using 512 chains of length 600 initialized with the flow $T_{\beta}$. \isalgs~was evaluated with 9750 samples. }
    \label{fig:funnel:main}
\end{figure}

\paragraph{Design of the target $\pi$}

Neal's funnel is a distribution which closely resembles the posterior of hierarchical Bayesian models. It depends on 2 positive parameters $a$, $b$ and is defined as follows

$$
    X \sim \pi(a,b) \iff \begin{cases}
        X_1 &\sim \mathcal{N}(0,a) \\
        X_i &\sim \mathcal{N}(0,\mathrm{e}^{b X_1}), \forall i \in \{2, \ldots, d\}
    \end{cases}\eqsp.
$$

\paragraph{Design of the flow $T_{\beta}$}

We define an analytical parametric flow $T_{a,b,\alpha}$ as follows 

$$T_{a,b,\alpha} : z \mapsto \begin{pmatrix}
\sqrt{\frac{a}{\alpha}} z_1\\
\vdots\\
\frac{1}{\sqrt{\alpha}} \exp(\frac{b \sqrt{\frac{a}{\alpha}} z_1}{2}) z_i \\
\vdots
\end{pmatrix}\eqsp, ~~ T_{a,b,\alpha}^{-1} : x \mapsto \begin{pmatrix}
\sqrt{\frac{\alpha}{a}} x_1 \\
\vdots\\
\sqrt{\alpha} \exp(-\frac{b}{2}x_1) x_i \\
\vdots
\end{pmatrix}\eqsp,
$$

\begin{align*}
    \log \abs{\det(J_{T_{a,b,\alpha}}(z))} = \log(\sqrt{a}) - \frac{d}{2} \log(\alpha) + \frac{d-1}{2} b \sqrt{\frac{a}{\alpha}} z_1\eqsp, \\
    \log \abs{\det(J_{T_{a,b,\alpha}^{-1}}(x))} = -\log(a) + \frac{d}{2} \log(\alpha) - \frac{d-1}{2} b x_1\eqsp.
\end{align*}

$T_{a,b,\alpha=1}$ is the natural flow which transports $N(0,I_d)$ to $\pi(a,b)$. The parameter $\alpha$ was introduce for two purposes : control the variance of the push-backward and calibrate the mode of the push-forward. Indeed, it's easy to see that

$$T^{-1}_{a,b,\alpha}(T_{a,b,1}(z)) = T^{-1}_{a,b,\alpha}\left(\begin{pmatrix}
\sqrt{a} z_1\\
\vdots\\
 \exp(\frac{b \sqrt{a} z_1}{2}) z_i \\
\vdots
\end{pmatrix}\right) = \begin{pmatrix}
\sqrt{\frac{\alpha}{a}}\left(\sqrt{a} z_1\right)\\
\vdots\\
\sqrt{\alpha} \exp\left(-\frac{b}{2} \sqrt{a} z_1\right)\left(\exp(\frac{b \sqrt{a} z_1}{2}) z_i\right) \\
\vdots
\end{pmatrix} = \sqrt{\alpha} z\eqsp.$$

So if $X \sim \pi(a,b)$, then $T_{a,b,\alpha}^{-1}(X) \sim \mathcal{N}(0,\alpha I_d)$. Moreover, we can compute the mode of $T_{a,b,\alpha}(\mathcal{N}(0,I_d))$ using the change of variable formula,

\begin{align*}
    \log \lambda_{T_{a,b,\alpha}}^{\mathcal{N}(0,I_d)}(x) &=
    \log(\mathcal{N}(T^{-1}_{a,b,\alpha}(x); 0, I_d)) + \log \det \abs{J_{T^{-1}_{a,b,\alpha}}(x)} \\
    &= -\frac{\norm{T^{-1}_{a,b,\alpha}(x)}^2}{2} - \frac{d}{2} \log(2\pi) - \log(a) + \frac{d}{2} \log(\alpha) - \frac{d-1}{2} b x_1 \\
    &= -\frac{1}{2} \left(\left(\sqrt{\frac{\alpha}{a}} x_1\right)^2 + \sum_{i=2}^d \left(\sqrt{\alpha} \exp\left(-\frac{b}{2} x_1\right) x_i\right)^2\right) - \frac{d-1}{2} x_1 + K \\
    &= -\frac{1}{2} \left(\frac{\alpha}{a} x_1^2 + \alpha \exp(-b x_1) \sum_{i=2}^d  x_i^2 \right) - \frac{d-1}{2} x_1 + K \eqsp,\\
\end{align*}

where $K$ is a constant. Canceling the gradient of $\log \lambda_{T_{a,b,\alpha}}^{\mathcal{N}(0,I_d)}$ leads to

\begin{align*}
    \begin{cases}
        \frac{\partial \log \lambda_{T_{a,b,\alpha}}^{\mathcal{N}(0,I_d)}(x)}{\partial x_1}(x) = 0 &\\
        \frac{\partial \log \lambda_{T_{a,b,\alpha}}^{\mathcal{N}(0,I_d)}(x)}{\partial x_j}(x) = 0 &, (j \neq 1)
    \end{cases}
    &\iff
    \begin{cases}
        -\frac{1}{2}\left(\frac{2\alpha}{a}x_1 - b\alpha\left(\sum_{i=2}^d x_i^2\right) \exp(-bx_1)\right) - \frac{d-1}{2} b = 0 &\\
        -\alpha \exp(-bx_1) x_j = 0 & (j \neq 0)
    \end{cases} \\
    &\iff \begin{cases}
        -\frac{\alpha}{a}x_1 = \frac{d-1}{2} b &\\
        x_j = 0 & (j \neq 0)
    \end{cases} \\
    &\iff z = \begin{pmatrix}
        -\frac{ab}{2\alpha} (d-1) \\
        0 \\
        \vdots \\
        0
    \end{pmatrix}\eqsp.
\end{align*}

Let $\beta > 0$, $a^{\star} = 3$, $b^{\star} = 1$ then if $a = \beta a^{\star}$, $b = b^{\star}$ and $\alpha = \beta$ the mode of the push-forward of $\mathcal{N}(0,I_d)$ through $T_{a,b,\alpha}$ and the mode of $\pi(a^{\star}, b^{\star})$ will always coincide for any $\beta > 0$. We use this principle to create the flow $T_{\beta} = T_{\beta a^{\star}, b^{\star}, \beta}$ which is a perfect mapping from $\mathcal{N}(0,I_d)$ to $\pi(a^{\star}, b^{\star})$ if $\beta = 1$.

\paragraph{Sampling details}

Sampling procedure is the same as in Sec. \ref{sec:sampling_details_general} with the Kolmogorov-Smirnov distance as target metric, only some hyper-parameters change (see table \ref{table:funnel:spl_params}).

\begin{table}[t]
    \caption{Sampling hyper-parameters for Neal's Funnel experiment.}
    \label{table:funnel:spl_params}
    \begin{center}
        \begin{small}
            \begin{tabular}{c | c | c | c | c | c | c }
                \toprule
                \thead{Dimension} & \thead{$n$} & \thead{$N$ \\ (\propalgs)} & \thead{$N$ \\ (\neutraflow)} & \thead{$N$ \\ (IS)} & \thead{$n_{local}$ \\ (\propalgs)} & \thead{$n_{local}$ \\ (\neutraflow)} \\ 
                \midrule
                16 & 500 & 60 & 60 & 7500 & 10 & 40 \\
                32 & 550 & 60 & 100 & 8250 & 10 & 40 \\
                64 & 600 & 60 & 100 & 9000 & 10 & 40 \\
                128 & 650 & 60 & 100 & 9750 & 10 & 40 \\
                256 & 700 & 60 & 120 & 10500 & 10 & 40 \\
                \bottomrule
            \end{tabular}
        \end{small}
    \end{center}
\end{table}

\subsection{High acceptance rates of pure \propalgs~algorithms} \label{app:wierd_energies}

\begin{figure}[t]
    \centering
    \begin{subfigure}{\linewidth}
		\centering
		\includegraphics[width=\linewidth]{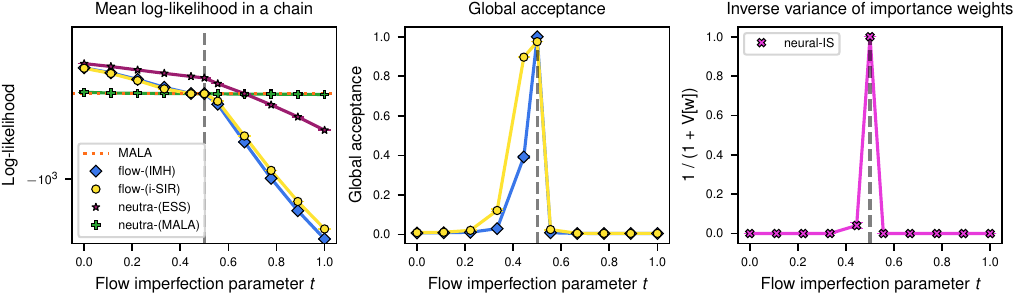}
		\label{fig:3flows:wierd_energies}
    \end{subfigure}
    \begin{subfigure}{\linewidth}
		\centering
		\includegraphics[width=\linewidth]{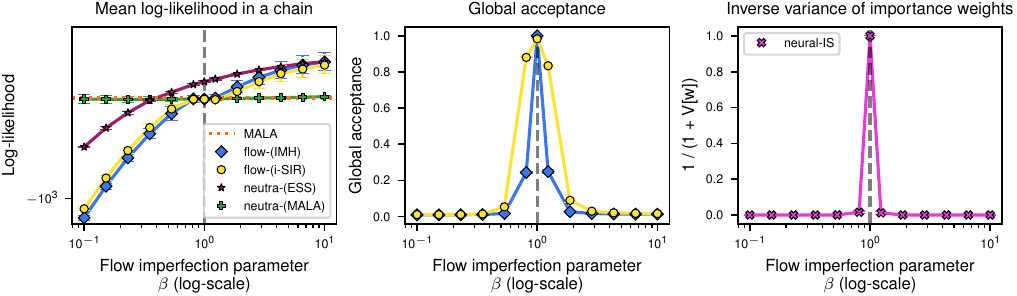}
		\label{fig:funnel:wierd_energies}
    \end{subfigure}%
    \caption{\textit{(Left)} Likelihoods levels \textit{(Middle)} / Global acceptance of pure \propalgs~methods / \textit{(Right)} Participation ratio of importance sampling, depending on the quality of the flow - \textbf{(Top)} Interpolated gaussians flow \textbf{(Bottom)} Neal's Funnel flow}
    \label{fig:3flows_funnel:wierd_energies}
\end{figure}

Figure \ref{fig:3flows_funnel:wierd_energies} shows how global \propalgs~samplers - IMH and i-SIR - behave depending on the quality of the flow. The middle and right columns show that the acceptance of IMH/i-SIR and the participation ratio of IS are crashing as soon as $\pi$ and $\rho$ aren't aligned perfectly. Note the the crashing rate of i-SIR is slower than the one of IMH. Fig. \ref{fig:3flows_funnel:wierd_energies} also highlights that \propalgs~algorithms as well as \reparamalgs~with ESS produce samples with very different likelihoods compared to \reparamalgs~methods which are known to perform better (see Fig. \ref{fig:3flows:main} and Fig. \ref{fig:funnel:main} of the main paper).

\subsection{Experimental details for the high dimensional Gaussian mixture} \label{app:highdim_mog}

\paragraph{Design of $\pi$}

$\pi$ is a mixture of 4 isotropic Gaussians i.e., $\pi = 1/4 \sum_{i=1}^4 \mathcal{N}(\mu_i, I_d)$ where the $\mu_i$ are defined as $\mu_1 = a \times (1, 1, 1, \ldots, 1, 1, 1)$, $\mu_2 = a \times (-1, -1, -1, \ldots, 1, 1, 1)$, $\mu_2 = - \mu_2$ and $\mu_4 = -\mu_1$ and $a = 0.5919$. This specific value of $a$ guarantees that if $X \sim \mathcal{N}(\mu_i,I_d)$ then $\forall j \neq i, \mathbb{P}(\|X - \mu_j\| < \|X - \mu_i\|) \leq 10^{-10}$ for any dimension $d$.

\paragraph{Training the flow}

\noindent The flows used here are RealNVPs. The base of the flow is $\rho = \mathcal{N}(0, \sigma^2 I_d)$ where $\sigma^2$ is the maximum variance of $\pi$ along each dimension \footnote{Using the total variation formula $\forall i \in \{1, \ldots, d\}, \sigma^2 = \sigma_i^2 = 1 + \sum_{j=1}^4 \frac{1}{4} (\mu_j)_i^2$}. The flow was trained using Adam \cite{Kingma2014} optimizer at a progressive learning rate. All the coupling layers had 3 hidden layers initialized with very small weights $(\simeq 10^{-6})$. The batch size was 8192 and the decay rate was $0.99$. The other hyper-parameters can be found in table \ref{table:highdim_mog:flow_params}. An indicator of the flows qualities can be found in Fig. \ref{fig:highdim_mog:global_acc_avg_kl}. 

\begin{table}[t]
    \caption{RealNVP training hyperparameters for the Gaussian mixture experiment}
    \label{table:highdim_mog:flow_params}
    \begin{center}
        \begin{small}
            \begin{tabular}{c | c | c | c | c | c}
                \toprule
                \thead{Dimension} & \# \thead{Iteration} & \thead{Patience of learning \\ rate scheduler} & \thead{Learning rate} & \thead{Size of hidden layers} & \thead{\# RealNVP blocks} \\
                \midrule
                16 & 2500 & 100 & $10^{-2}$ & 64 & 4 \\
                32 & 2640 & 106 & $7.36 \times 10^{-3}$ & 76 & 4 \\
                64 & 2900 & 120 & $3.98 \times 10^{-3}$ & 102 & 4 \\
                128 & 3440 & 146 & $1.17 \times 10^{-3}$ & 153 & 5 \\
                256 & 11250 & 200 & $10^{-4}$ & 256 & 8 \\
                \bottomrule
            \end{tabular}
        \end{small}
    \end{center}
\end{table}

\begin{figure}[t]
    \begin{subfigure}{0.50\linewidth}
        \centering
    \includegraphics[width=0.80\linewidth]{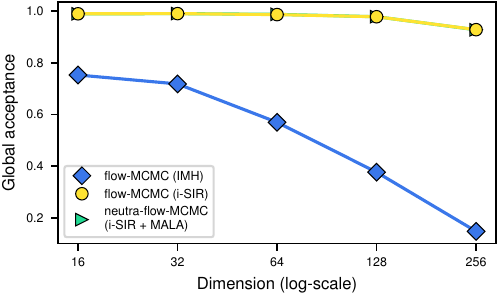}
        \label{fig:highdim_mog:global_acc}
    \end{subfigure}
    \begin{subfigure}{0.50\linewidth}
        \centering
        \includegraphics[width=0.80\linewidth]{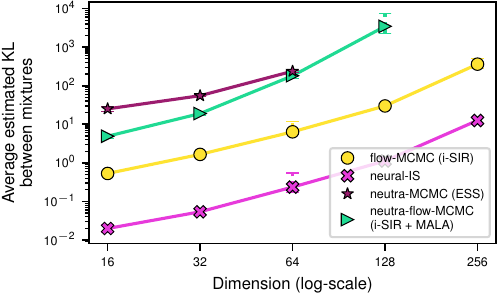}
        \label{fig:highdim_mog:avg_kl}
    \end{subfigure}
    \caption{\textbf{(Left)} Global acceptance of the RealNVP flow obtained after training on the mixture of Gaussians in increasing dimension \textbf{(Right)} Averaged forward Kullback-Leiber for the Gaussian mixture.}
    \label{fig:highdim_mog:global_acc_avg_kl}
\end{figure}

\paragraph{More \reparamalgs~algorithms for the 2D mixture}

\begin{figure}
    \centering
    \includegraphics[width=0.80\linewidth]{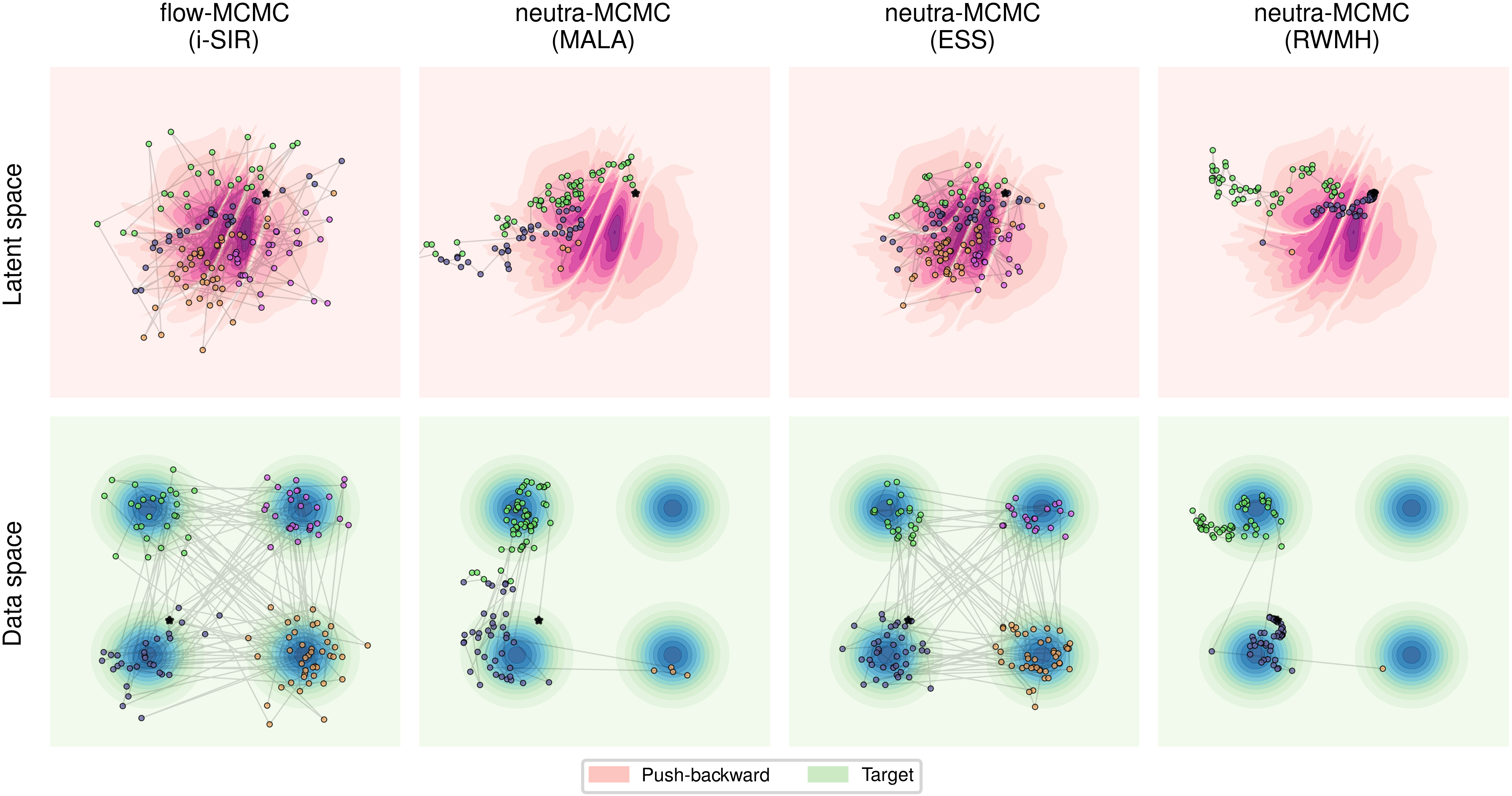}
    \caption{Extended version of Fig. \ref{fig:highdim_mog:recap}}
    \label{fig:highdim_mog:samples_neutra_extended}
\end{figure}

Figure \ref{fig:highdim_mog:samples_neutra_extended} extends Fig. \ref{fig:highdim_mog:recap} with more \reparamalgs~methods. It shows that changing the reparametrized sampler could allow switching modes in the latent space. For instance, ESS is able to do so.
Our conclusion is that using local samplers in this pathological latent space make mixing between modes much longer and less accurate compared to using \propalgs~methods. 

\paragraph{The case of mode collapse}

\begin{figure}[t]
    \centering
    \begin{subfigure}{0.33\linewidth}
        \centering
        \includegraphics[width=\linewidth]{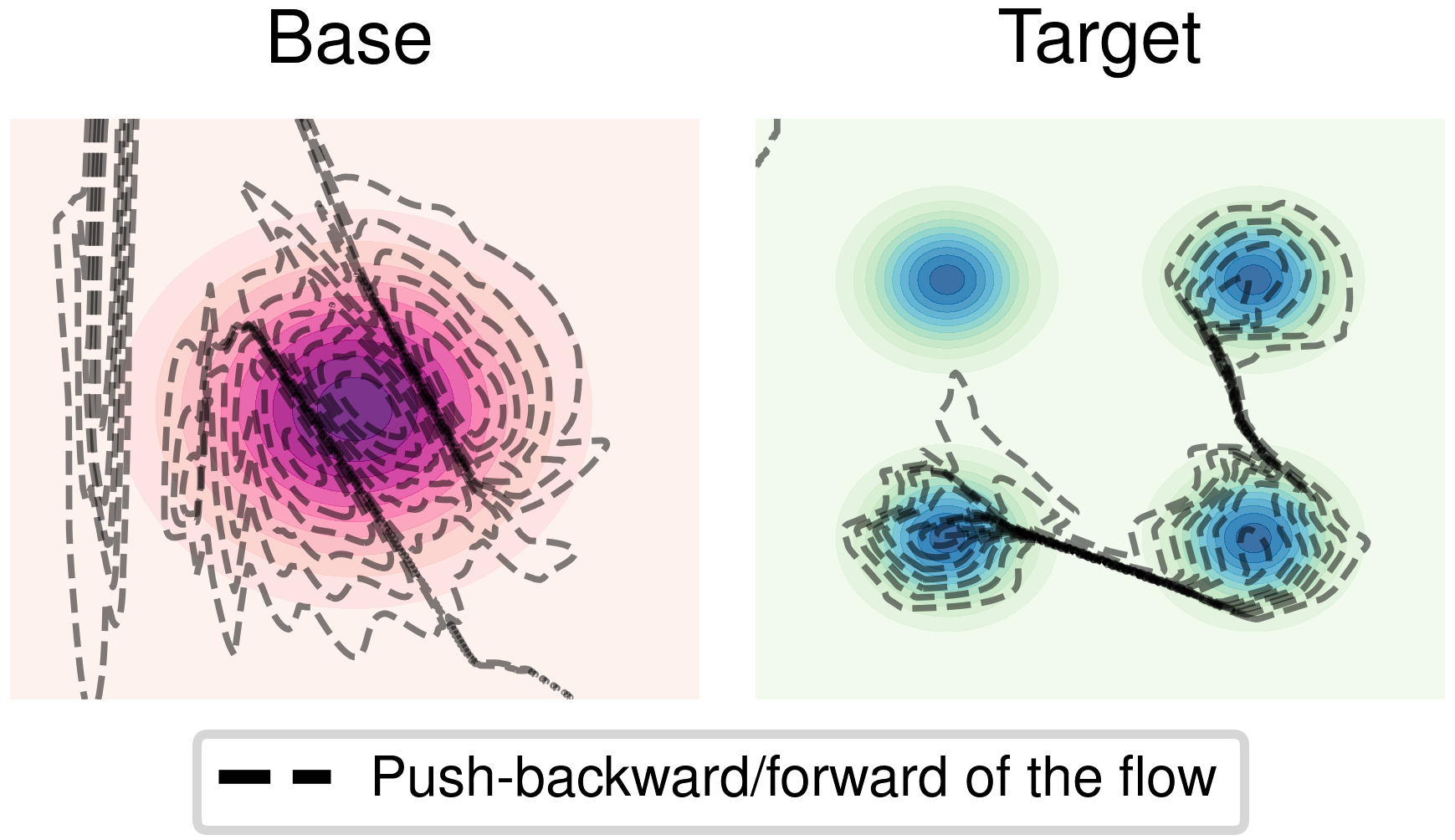}
    \end{subfigure}%
    \begin{subfigure}{0.33\linewidth}
        \centering
        \includegraphics[width=\linewidth]{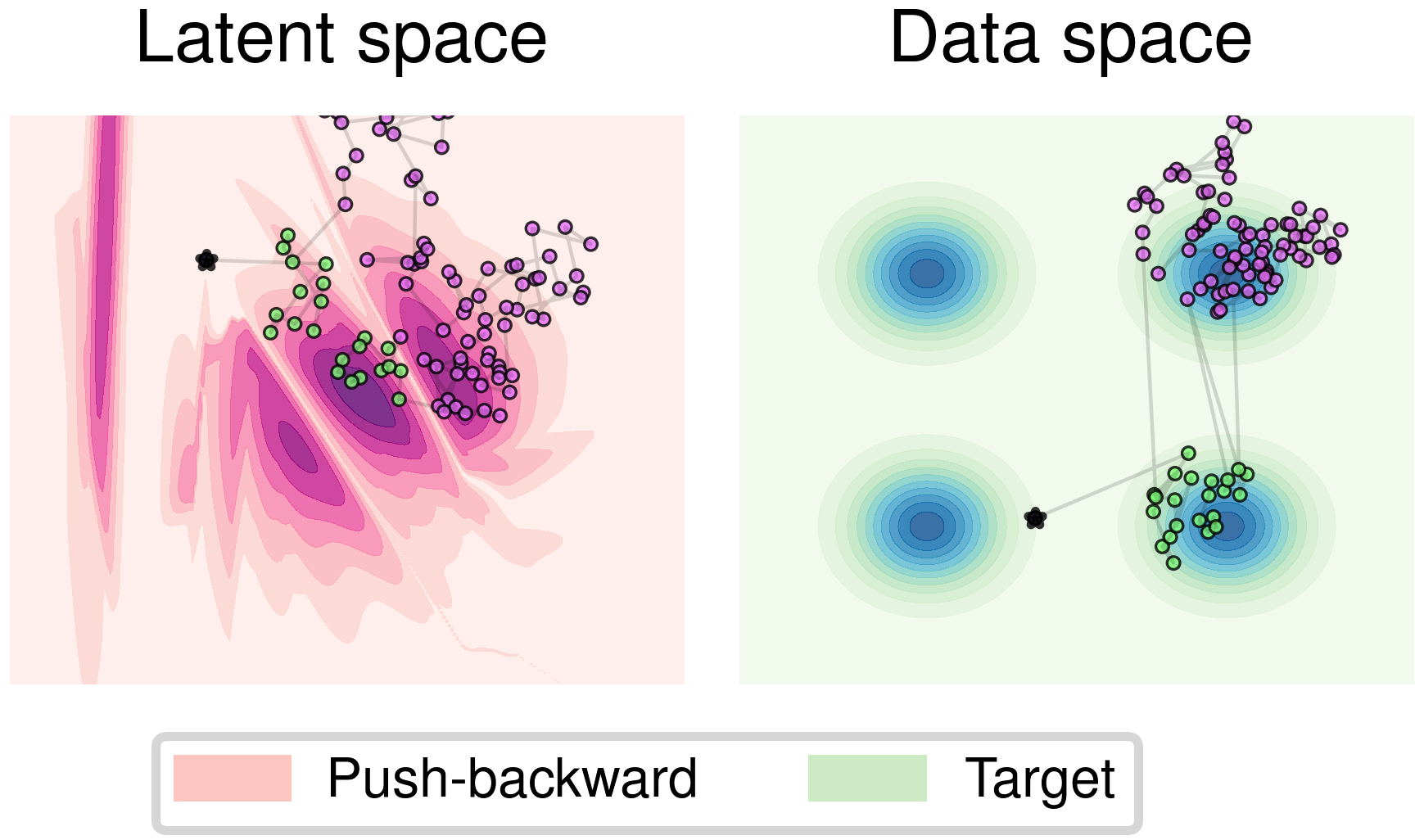}
    \end{subfigure}%
    \begin{subfigure}{0.33\linewidth}
        \centering
        \includegraphics[width=\linewidth]{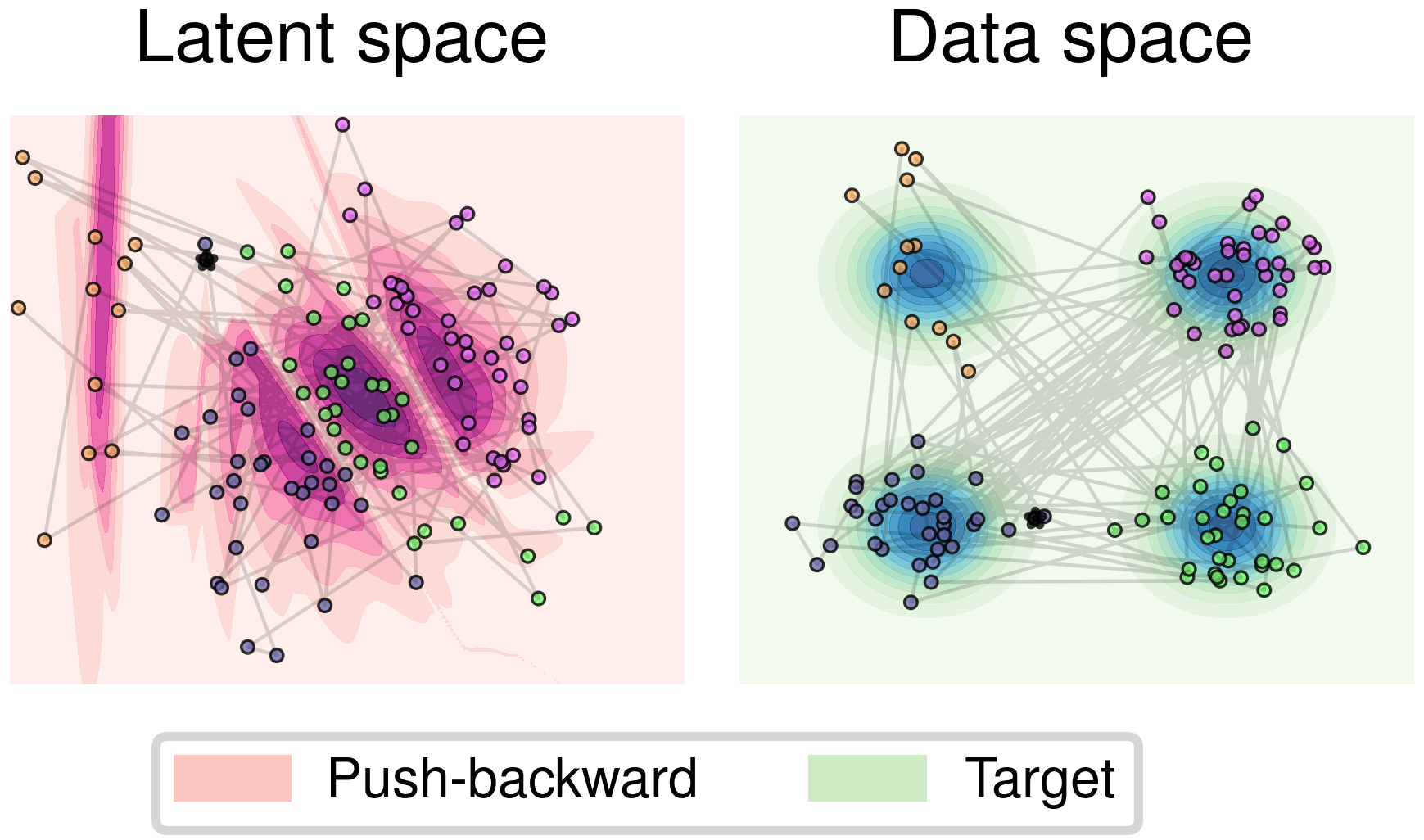}
    \end{subfigure}
    \caption{\textbf{An illustration of sampling with a mode-collapsed flow. (Left)} RealNVP trained with backward KL torwards a 2d target mixture of 4 Gaussians. \textbf{(Middle) \reparamalgs} and \textbf{(Right) \propalgs} Samplers using the flow on the left. The 128-step MCMC chain is colored according to the closest mode in the data space (bottom row) with corresponding location in the latent space (top row). MALA's step-size was chosen to reach 75\% of acceptance. \label{fig:mode-collapse}}
\end{figure}

As recalled in the introduction, normalizing flows trained with the backward KL objective notoriously suffer from mode collapse. Based on this observation, we extended the conclusions of Fig. \ref{fig:highdim_mog:recap} Left using a flow which suffered from mode collapse. Fig. \ref{fig:mode-collapse} shows that \propalgs~methods sometimes reach the uncovered mode while \reparamalgs~never do. One can also think of the following thought experiment: consider a target as a two-component mixture of Gaussian and a Gaussian base distribution. In a mode-collapsed situation, we can imagine that an affine flow sends the base distribution perfectly to one of the modes. All types of NF-enhanced samplers will fail. On the one hand, \propalgs~will never manage to propose in the other mode. On the other hand, the push-backward of the target sampled by \reparamalgs~will be an affine transformation of the original Gaussian mixture, which a local sampler cannot properly sample.

\paragraph{\propalgs~methods exploit modes}

If figure \ref{fig:highdim_mog:recap} shows the exploration capability that \propalgs~methods have and \reparamalgs~methods lack, it doesn't show how much the mass within the modes is exploited. By cutting the Markov chains depending on the closest mode and computing the mean and covariance of each part, we can compute the forward KL between each mode and the Gaussian approximated in each part. Averaging those metrics leads to Fig. \ref{fig:highdim_mog:global_acc_avg_kl} which shows that in moderate dimensions the flow is good enough so purely global methods can exploit the modes. Here, we can see the
dimension dependence of all methods.

\paragraph{Sampling details}

Sampling procedure is the same as in Sec. \ref{sec:sampling_details_general} and targets the mean score between the average KL and the mode mixing metric. Note that unlike Sec. \ref{sec:sampling_details_general}, 1024 chains were used. Only some hyper-parameters change and can be found in table \ref{table:highdim_mog:spl_params}.

\begin{table}[t]
    \caption{Sampling hyper-parameters for the Gaussian mixture experiment.}
    \label{table:highdim_mog:spl_params}
    \begin{center}
        \begin{small}
            \begin{tabular}{c | c | c | c | c | c | c }
                \toprule
                \thead{Dimension} & \thead{$n$} & \thead{$N$ \\ (\propalgs)} & \thead{$N$ \\ (\neutraflow)} & \thead{$N$ \\ (IS)} & \thead{$n_{local}$ \\ (\propalgs)} & \thead{$n_{local}$ \\ (\neutraflow)}  \\ 
                \midrule
                16 & 700 & 120 & 120 & 21000 & 5 & 5 \\
                32 & 900 & 140 & 140 & 31500 & 5 & 5 \\
                64 & 1100 & 160 & 160 & 44000 & 5 & 5 \\
                128 & 1300 & 180 & 180 & 58500 & 5 & 5 \\
                256 & 1500 & 200 & 200 & 75000 & 5 & 5 \\
                \bottomrule
            \end{tabular}
        \end{small}
    \end{center}
\end{table}

\subsection{Experimental details on the banana distribution} \label{app:banana}

\paragraph{Design of $\pi$}

The banana distribution depends on two positive parameters $a$ and $b$ and is expressed as follows

$$
    X \sim \pi(a,b) \iff  \forall i \in \{0, \ldots, d/2\}, \begin{cases}
        X_{2i} &= a Z_{2i}\\
        X_{2i+1} &= Z_{2i+1} + b a^2 Z^2_{2i} - a^2 b
    \end{cases}\eqsp,
$$

where $Z \sim \mathcal{N}(0, I_d)$ and $d$ is even. This leads to a natural bijection $T_{a,b}$ between $\mathcal{N}(0,I_d)$ and $\pi(a,b)$

$$T_{a,b} : z \mapsto \begin{pmatrix}
\vdots\\
a z_{2i} \\
z_{2i+1} + b a^2 z^2_{2i} - a^2 b \\
\vdots
\end{pmatrix}\eqsp, ~~ T_{a,b}^{-1} : x \mapsto \begin{pmatrix}
\vdots \\
x_{2i} / a \\
x_{2i+1} - b x_{2i}^2 + a^2 b \\
\vdots
\end{pmatrix}\eqsp,
$$

$$\log |\det(J_{T_{a,b}}(z))| = \frac{d}{2} \log(a)\eqsp, ~~~ \log |\det(J_{T_{a,b}^{-1}}(x))| = -\frac{d}{2} \log(a)\eqsp.$$

In the following, we take $a = 10$ and $b = 0.02$ (see Fig. \ref{fig:banana:dist}).

\begin{figure}[t]
    \centering
    \includegraphics[width=0.5\linewidth]{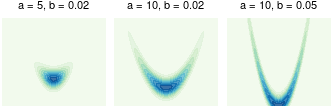}
    \caption{Banana distribution with different parameters}
    \label{fig:banana:dist}
\end{figure}

\paragraph{Flows, training and sampling}

The normalizing flows used are RealNVPs \cite{Dinh2017density} with 3 layers deep neural networks initialized with small weights ($\simeq 10^{-6}$). The specific architecture of the neural networks and the training hyperparameters depending on the dimension of the problem can be found in table \ref{table:banana:flow_params} (the learning rate scaling is $0.98$). The sampling algorithms with i-SIR using $N = 60$ particles, $n_{local} = 10$ and the target acceptance of local steps is 75\%. 128 random projections where used to compute the sliced total variation. There are 128 MCMC chains of length 1024 sampled in parallel.

\begin{table}[t]
    \caption{RealNVP hyper-parameters for the Banana experiment.}
    \label{table:banana:flow_params}
    \begin{center}
        \begin{small}
            \begin{tabular}{c | c | c | c | c | c}
                \toprule
                \thead{Dimension} & \# \thead{Iteration} & \thead{Patience of learning \\ rate scheduler} & \thead{Learning rate} & \thead{Size of hidden layers} & \thead{\# RealNVP blocks} \\
                \midrule
                16 & 1000 & 75 & $10^{-2}$ & 64 & 3 \\
                32 & 1193 & 82 & $8 \times 10^{-3}$ & 74 & 3 \\
                64 & 1580 & 96 & $5 \times 10^{-3}$ & 96 & 3 \\
                128 & 2354 & 125 & $2 \times 10^{-3}$ & 139 & 5 \\
                256 & 3903 & 183 & $3 \times 10^{-4}$ & 226 & 7 \\
                512 & 7000 & 300 & $1 \times 10^{-5}$ & 400 & 12 \\
                \bottomrule
            \end{tabular}
        \end{small}
    \end{center}
\end{table}

\section{Mixing time for IMH for log concave distribution} \label{app:proof_cond}

\subsection{General Theorem}

Here we consider Independent Metropolis-Hastings (IMH) with a target $\pi$ on $\rset^d$ and a proposal $Q$ with density denoted by $q$ with respect to the Lebesgue measure. We denote by $P$ the resulting Markov kernel given for any $x \in \rset^d$ and Borel set $\msa$ by
\begin{equation}
    P(x,\msa) = \int_{\msa} Q(\rmd y) \mathrm{acc}(x, y) + \left(1 - \int Q(\rmd y) \mathrm{acc}(x, y) \right)  \delta_x(\msa) \eqsp.
\end{equation}
We define the mixing time of $P$ with respect to an initial distribution $\mu$ and a precision target $\epsilon >0$ as
\begin{equation}
    \tau_{\mix}(\mu,\epsilon) = \inf \{n \in \mathbb{N} ~:~ \norm{\mu P^n - \pi}_{\TV} \leq \epsilon \} \eqsp,
\end{equation}
where the total variation distance between two distributions $P_1$ and $P_2$ is defined as
\begin{equation}
    \norm{P_1 - P_2}_{\TV} = \sup_{A \in \mathcal{B}(\rset^d)} \abs{P_1(A) - P_2(1)}\eqsp.
\end{equation}
The quantity $\tau_{\mix}(\mu,\epsilon)$ provides the minimum number of MCMC steps needed to bring the total variation distance between the chain and the target $\pi$ below a precision $\epsilon$ when the chain has initial distribution $\mu$. Now we introduce the $s$-conductance which will be later use to bound the mixing time. Let $s \in (0, 1/2]$, we define the $s$-conductance as
\begin{equation}
\phi_s = \inf_{\mss, \pi(\mss) \in (s,1/2]} \frac{\int_{\mss} P(x,\mss^c) \pi(\rmd x)}{\pi(\mss) - s} \eqsp,
\end{equation}
where $\mss^c$ denotes the complement of $\mss$. The $s$-conductance quantifies the probability of transitioning from a set $\mss$ charged with a probability $s$ to its complementary. This quantity is of interest since it can be used for bounding the mixing time  very useful because of Theorem \ref{th:lovasz} and its corollary taken from \cite{Lovasz1993}.
\begin{equation}
\tau_{mix}(\mu,\epsilon) = \inf \{n \in \mathbb{N} ~:~ \norm{P^n(\mu,.) - \pi}_{\TV} \leq \epsilon \} \eqsp, \text{ where } \epsilon > 0 \eqsp.
\end{equation}
$\tau_{mix}(\mu,\epsilon)$ provides the minimum number of MCMC steps needed to bring the total variation distance between the chain and the target $\pi$ bellow an accuracy $\epsilon$ when the chain started with distribution $\mu$. 

\begin{theorem}[Corollary 1.5 \citep{Lovasz1993}]\label{th:lovasz}
A Markov chain with $\pi$ as unique invariant distribution and $\mu$ a $\beta$-warm start with respect to $\pi$ (i.e., for any Borel set $\mse$, $\mu(\mse) \leq \beta \pi(\mse)$) verifies that for all $n \in \mathbb{N}$,
\begin{equation}
    \norm{P^n(\mu,.) - \pi}_{\TV} \leq \beta s + \beta \left(1 - \frac{\phi_s^2}{2}\right)^n \leq \beta s + \beta \exp\left(-\frac{n}{2} \phi_s^2\right) \eqsp.
  \end{equation}
  In particular,  for $0 < \epsilon < 1$, then
\begin{equation}
    \tau_{\mix}(\mu,\epsilon) \leq \frac{2}{\phi_s^2} \log\left(\frac{2\beta}{\epsilon}\right) \eqsp.
\end{equation}
\end{theorem}

Here we follow now a common strategy to derive quantitative mixing times for Metropolis-Hastings algorithms applied to log-concave target by applying \Cref{th:lovasz} \cite{Dwivedi2018, Mou2019, Chen2020, Chewi2020, Wu2021, Narayanan2022}.

 We denote by $w$ the weight function $w = \pi / q$ and $\tilde{w} = \log w$ the log-weight function. 
We establish first  conditions on the proposal $Q$ and the log-weight function  which allows us to get lower bounds on the conductance for $P$ and therefore its mixing time. Then, we specialize our result to the case $Q= \gauss(0,\sigma^2\Idd)$ and $\pi$ positive and strongly log-concave, see. Assumption \ref{ass:strong_convex_bis}. Without loss of generality, we assume that the unique mode of $\pi$ is 0.

\begin{assumption}\label{ass:iso_log_lip}
    \begin{itemize}
    \item $0$ is a mode for $\pi$, i.e., $\argmax_{x \in \rset^d} \pi(x) = 0$.
	\item $\pi$ satisfies an isoperimetric inequality with isoperimetric constant $\psi(\pi)$ i.e., for any partition $\{\mss_1,\mss_2,\mss_3\}$ of $\mathbb{R}^d$ \footnote{i.e., $\mss_1 \cup \mss_2 \cup \mss_3 = \mathbb{R}^d$ but for any $i \neq j, \mss_i \cap \mss_j = \emptyset$}, we have
	\begin{equation}\label{eq:isoperimetric}
 		\pi(\mss_3) \geq \psi(\pi) \dist(\mss_1,\mss_2) \pi(\mss_1) \pi(\mss_2) \eqsp,
	\end{equation}
        where $\dist(S_1,S_2) = \inf \{\norm{x - y} ~:~ x \in \mss_1, y \in \mss_2\}$
	\item The log-weight function is locally Lipschitz, i.e., for all $R \geq 0$ there exist $C_R$ such that for all $(x,y) \in \ball(0,R)^2$,
	\begin{equation}\label{eq:Lipschitz}
            \abs{\tilde{w}(x) - \tilde{w}(y)} \leq C_R \norm{x - y} \eqsp.
	\end{equation}
    \end{itemize}
  \end{assumption}

Note that if $\pi$ is log-concave, then \eqref{eq:isoperimetric} holds; see \cite{Bobkov1999}.

Our assumption on $Q$ is the following. Let $\alpha \in (1/2,1)$ and $s \in (0,1/2]$.

\begin{assumption}[$\alpha,s$]\label{ass:high_prob_cond_delta}
  There exist $\Delta \in (0,1)$, $\delta_{\Delta} >0$ and $R \geq 0$ satisfying
    \begin{equation}\label{eq:high_prob}
        \int_{\msb_{\Delta}} \pi(\rmd x) Q(\rmd y) \geq 1 - \delta_{\Delta} \eqsp,
    \end{equation}
    where 
    \begin{equation}
        \msb_{\Delta} = \left\lbrace (x,y) \in \ball(0,R)^2 ~:~ \norm{x-y} \leq  R_{\Delta}  \right\rbrace \eqsp, \text{ where } R_{\Delta} = \frac{-\log(\Delta)}{C_R} \eqsp,
    \end{equation}
    and
    \begin{equation}\label{eq:cond_delta}
        \displaystyle \frac{1}{1-\alpha} \delta_{\Delta} \leq \min\left(\frac{s}{4}, \frac{2\alpha - 1}{64 C_R} \psi(\pi) s\right) \eqsp.
    \end{equation}
\end{assumption}

\begin{theorem}[Conductance lower bound for IMH]\label{th:cond_imh}
    Let $1/2 < \alpha < 1$ and $0 < s < 1/2$. Assume assumption \ref{ass:iso_log_lip} and \ref{ass:high_prob_cond_delta}$(\alpha,s)$ hold. Then, it holds
    \begin{equation}
        \phi_s \geq \min\left(\frac{1-\alpha}{2}, \frac{(1-\alpha) (2\alpha - 1)}{128 C_R} \psi(\pi)\right)\eqsp.
    \end{equation}
\end{theorem}

\begin{proof}
 Define the set $\mse_{\Delta,\alpha}$ by
$$\mse_{\Delta,\alpha} = \left\lbrace x \in \ball(0,R) ~:~ \int 1_{(x,y) \in B_{\Delta}} Q(\rmd y) \geq \alpha \right\rbrace \eqsp.$$
Using  (\refeq{eq:high_prob}), we have by definition that
\begin{equation}
  \label{eq:1}
  \pi(\mse_{\Delta,\alpha}) \geq 1 - \frac{1}{1-\alpha} \delta_{\Delta}
\end{equation}
Indeed, it would not hold, using the law of total probability, we would have $1 - \delta_{\Delta} < [1 - \delta_{\Delta} / (1-\alpha)] 1 + [\delta_{\Delta} / (1-\alpha)] \alpha = 1 - \delta_{\Delta}$. Let $\mss$ be a measurable subset of $\mathbb{R}^d$ with $s \leq \pi(\mss) \leq 1/2$ and define
\begin{align*}
	\mss_1 &= \{x \in \mss ~:~ P(x, \mss^c) \leq 1 - \alpha\}\eqsp, \\
	\mss_2 &= \{x \in \mss^c ~:~ P(x, \mss) \leq 1 - \alpha\}\eqsp, \\
	\mss_3 &= (\mss_1 \cup \mss_2)^c\eqsp.
\end{align*}

If $\pi(\mss_1) < \pi(\mss)/2$ or $\pi(\mss_2) < \pi(\mss^c)/2$, then we may conclude from the reversibility of $P$ that
\begin{equation}
  \label{eq:first_min}
    \int_\mss P(x,\mss^c) \pi(\rmd x) = \frac{1}{2} \int_S P(x,\mss^c)\pi(\rmd x) + \frac{1}{2} \int_{\mss^c} P(x,S) \pi(\rmd x) > \frac{1}{2} \frac{\pi(\mss)}{2} (1-\alpha) = \frac{1-\alpha}{4} \pi(\mss)\eqsp.
  \end{equation}
  
  In the following, we assume that $\pi(\mss_1) \geq  \pi(\mss)/2$ or $\pi(\mss_2) \geq  \pi(\mss^c)/2$. Then it follows from the definition of the total variation distance that if $x \in \mse_{\Delta,\alpha} \cap \mss_1$ and $y \in \mse_{\Delta,\alpha} \cap \mss_2$ such that $$\norm{P(x,.) - P(y,.)}_{\TV} \geq 2 \alpha -1\eqsp.$$

Moreover using the fact that $x \in \mse_{\Delta,\alpha} \cap S_1$ and $y \in \mse_{\Delta,\alpha} \cap S_2$ and the expression of the Markov kernel, we have that
\begin{align*}
	\norm{P(x,.) - P(y,.)}_{\TV} &= \int \abs{\mathrm{acc}(x,z) - \mathrm{acc}(y,z)} Q(\rmd z) \\
	&= \int \abs{\min\left(1,\frac{\pi(z)q(x)}{\pi(x)q(z)}\right) - \min\left(1,\frac{\pi(z)q(y)}{\pi(y)q(z)}\right)} Q(\rmd z) \\
	&= \int \abs{\min\left(1,\frac{w(z)}{w(x)}\right) - \min\left(1,\frac{w(z)}{w(y)}\right)} Q(\rmd z) \\
	&\leq \int \abs{\tilde{w}(x) - \tilde{w}(y)} Q(\rmd z) \leq \abs{\tilde{w}(x) - \tilde{w}(y)} 
\leq C_R \norm{x - y}\eqsp,
\end{align*}
where we used the Lipschitz behavior of $t \mapsto \min(1,\mathrm{e}^t)$ and property (\refeq{eq:Lipschitz}) in the last lines. This leads to $2 \alpha -1 \leq C_R \norm{x -  y}$ and more generally to
\begin{equation}
    \dist(\mse_{\Delta,\alpha} \cap \mss_1, \mse_{\Delta,\alpha} \cap \mss_2) \geq \frac{2 \alpha -1}{C_R}\eqsp.
\end{equation}

Using the isoperimetric inequality (\refeq{eq:isoperimetric}), we have that
\begin{align*}
	\pi(\mse_{\Delta,\alpha}^c \cup \mss_3) &= \pi((\mse_{\Delta,\alpha} \cap \mss_1)^c \cap (\mse_{\Delta,\alpha} \cap \mss_2)^c) \\
	&\geq \frac{2 \alpha -1}{C_R} \psi(\pi) \pi(\mse_{\Delta,\alpha} \cap \mss_1) \pi( \mse_{\Delta,\alpha} \cap \mss_2)\eqsp.
\end{align*}

Using \eqref{eq:cond_delta}-\eqref{eq:1}, we have with the union bound, $s \in (0,1/2]$ and the condition
$\pi(\mss_1) \geq  \pi(\mss)/2$ or $\pi(\mss_2) \geq  \pi(\mss^c)/2 \geq 1/4$, 
\begin{align*}
	\pi(\mss_3) + \frac{1}{1-\alpha} \delta_{\Delta} &\geq \pi(\mss_3) + \pi(\mse_{\Delta,\alpha}^c) \\
	&\geq  \frac{2\alpha-1}{C_R} \psi(\pi) \pi(\mse_{\Delta,\alpha} \cap \mss_1) \pi( \mse_{\Delta,\alpha} \cap \mss_2) \\
	&\geq \frac{2\alpha-1}{C_R} \psi(\pi) (\pi(\mss_1) - \pi(\mse_{\Delta,\alpha}^c)) (\pi(\mss_2) - \pi(\mse_{\Delta,\alpha}^c)) \\
	&\geq \frac{2\alpha-1}{C_R} \psi(\pi) \left(\pi(\mss_1) - \frac{1}{1-\alpha} \delta_{\Delta}\right) \left(\pi(\mss_2) - \frac{1}{1-\alpha} \delta_{\Delta}\right) \\
	&\geq \frac{2\alpha-1}{C_R} \psi(\pi) \left(\frac{\pi(\mss)}{2} - \frac{1}{1-\alpha} \delta_{\Delta}\right) \left(\frac{1}{4}  - \frac{1}{1-\alpha} \delta_{\Delta}\right) \\
	&\geq \frac{2\alpha-1}{C_R} \psi(\pi) \frac{\pi(\mss)}{4} (1-\alpha)\eqsp.
\end{align*}

Using (\refeq{eq:cond_delta}) again, we get

$$\pi(\mss_3) \geq \frac{2\alpha-1}{64C_R} \psi(\pi) \pi(\mss)\eqsp.$$

Consequently,  using that $P$ is reversible, we get 
\begin{align*}
	\int_{\mss} P(x,\mss^c)\pi(\rmd x) &\geq \frac{1}{2} \left(\int_{\mss \cap \mss_3} P(x,\mss^c) \pi(\rmd x) + \int_{\mss^c \cap \mss_3} P(x,\mss) \pi(\rmd x) \right) \\
	&\geq \frac{1}{2} \left(\int_{\mss \cap \mss_3} (1-\alpha) \pi(\rmd x) + \int_{\mss^c \cap \mss_3} (1-\alpha) \pi(\rmd x) \right) \\
	&\geq \frac{1-\alpha}{2} \pi(\mss_3) \\
	&\geq \frac{(1-\alpha)(2\alpha-1)}{128 C_R} \psi(\pi) \pi(\mss)\eqsp,
\end{align*}
along with (\refeq{eq:first_min}) and the definition of the $s$-conductance, we get
$$\phi_s \geq \min\left(\frac{(1-\alpha)}{2}, \frac{(1-\alpha)(2\alpha-1)}{128 C_R} \psi(\pi)\right)\eqsp.$$

\end{proof}

\begin{corollary}[Mixing time upper bound for IMH]\label{th:cond_imh_cor}
    Let $1/2 < \alpha < 1$, $0 < \epsilon < 1$ and $\mu$ a $\beta$-warm initial distribution with respect to $\pi$. Assume assumption \ref{ass:iso_log_lip} and \ref{ass:high_prob_cond_delta} hold, then we have the following upper bound on the mixing time
    \begin{equation}
        \tau_{mix}(\mu,\epsilon) \leq \frac{8}{(1-\alpha)^2} \log\left(\frac{2\beta}{\epsilon}\right) \max\left(1, \frac{64^2 C_R^2}{\psi(\pi)^2(2\alpha - 1)^2}\right) \eqsp.
    \end{equation}
\end{corollary}

\begin{proof}
    Apply Theorem \ref{th:lovasz} taking $s = \epsilon / (2\beta)$.
\end{proof}

\begin{theorem}\label{th:nice_th}
   Take $Q = \mathcal{N}(0,\sigma^2 I_d)$ where $\sigma > 0$ and assume \ref{ass:iso_log_lip}. Let $0 < \epsilon < 1$ and $\mu$ a $\beta$-warm distribution with respect to $\pi$. Assume that there exist $R_1 > 0$ and $0 < c_1 < 1/16$ such that
   $$\pi(\ball(0,R_1)) \geq 1 - \frac{c_1 \epsilon}{2\beta}\eqsp,$$
   then if 
   $$C_{R_{\epsilon,\beta}} \leq \frac{\psi(\pi)}{32}\eqsp,$$
   where
   $$
    R_{\epsilon,\beta} = \max\left(2\sigma\sqrt{d} \left(1 + \max\left(\left(\frac{-\log(\frac{c_2\epsilon}{2\beta})}{d}\right)^{1/4}, \sqrt{\frac{-\log(\frac{c_2 \epsilon}{2\beta})}{d}}\right)\right), R_1\right)\eqsp,
   $$
   with $c_2 = 1/16 - c_1$,
   then
   $$
        \tau_{mix}(\mu,\epsilon) \leq 128 \log\left(\frac{2\beta}{\epsilon}\right) \max\left(1, \frac{128^2 C_{R_{\epsilon,\beta}}^2}{\psi(\pi)^2}\right) \eqsp.
   $$
\end{theorem}

\begin{proof}
    Let $s = \epsilon / (2\beta)$. \citep[Lemma~1]{Dwivedi2018} states that for any $c > 0$, there exists $R_s$ such that $\phi(\ball(0,R_s)) \geq 1 - cs$ where $\phi$ is a $m$-strongly log-concave distribution and
    $$R_s = \sqrt{\frac{d}{m}} \left(2 + 2 \max\left(\left(\frac{-\log(cs)}{d}\right)^{1/4}, \sqrt{\frac{-\log(cs)}{d}}\right)\right)\eqsp.$$
    Applying this result to $Q$ which is $1/\sigma^2$-strongly log-concave  leads to $Q(\ball(0,R_{2,s})) \geq 1 - c_2 s$ where $c_2 > 0$ and
    $$R_{2,s} = \sigma \sqrt{d} \left(2 + 2 \max\left(\left(\frac{-\log(c_2 s)}{d}\right)^{1/4}, \sqrt{\frac{-\log(c_2 s)}{d}}\right)\right)\eqsp.$$
    Let $R_{\Delta} = R_1 + R_{2,s}$ and $R = \max(R_1, R_{2,s})$ in assumption \ref{ass:high_prob_cond_delta}, this leads to $\ball(0,R_1) \times \ball(0,R_{2,s}) \subset B_{\Delta}$ so
    $$\int_{\msb_{\Delta}} \pi(\rmd x) Q(\rmd y) \geq \int_{\ball(0,R_1) \times \ball(0,R_{2,s})} \pi(\rmd x) Q(\rmd y) \geq (1 - c_1s) (1 - c_2s)\eqsp.$$
    Because $c_1 > 0$ and $c_2 > 0$, we have that $\int_{\msb_{\Delta}} \pi(\rmd x) Q(\rmd y) \geq 1 - (c_1 + c_2) s = 1 - \delta_{\Delta}$ thus checking condition (\ref{eq:high_prob}) of assumption \ref{ass:high_prob_cond_delta}. Checking (\ref{eq:cond_delta}) of assumption \ref{ass:high_prob_cond_delta} amounts to check 
    \begin{align*}
        \frac{1}{1-\alpha} \delta_{\Delta} \leq \min\left(\frac{1}{4}, \frac{2\alpha-1}{64 C_R} \psi(\pi)\right) &\iff 
            \begin{cases}
                \delta_{\Delta} &\leq \frac{1-\alpha}{4}s \\
                \delta_{\Delta} &\leq \frac{(1-\alpha)(2\alpha-1)}{64 C_R} \psi(\pi) s
            \end{cases} \\
            &\iff \begin{cases}
                (c_1 + c_2) &\leq \frac{1-\alpha}{4} \\
                C_R &\leq \frac{(1-\alpha)(2\alpha-1)}{64 (c_1 + c_2)} \psi(\pi)
            \end{cases}
    \end{align*}
    Maximising $(1-\alpha)(2\alpha-1)$ with $\alpha = 3/4$ leads to
    $$
        \begin{cases}
            (c_1 + c_2) &\leq \frac{1}{16} \\
            C_R &\leq \frac{1}{512 (c_1 + c_2)} \psi(\pi)
        \end{cases}\eqsp,
    $$
    and taking $c_2 = 1/16 - c_1$, we get $C_R \leq \psi(\pi) / 32$.
\end{proof}

\begin{assumption}
\label{ass:strong_convex_bis}
The target $\pi$ is positive and $- \log \pi$ is $m$-strongly convex on $\rset^d$: for any $x,y \in\rset^d$ and $t \in [0,1]$,
$$-\log \pi(t x + (1-t)y) \leq -t \log \pi(x) -(1-t) \log \pi(y) \\- \frac{m t(1-t)}{2} \abs{x-y} \eqsp.$$
\end{assumption}

\begin{corollary}\label{th:nice_th_cor}
    Let $\pi$ be a $m$-strongly log-concave distribution (Assumption \ref{ass:strong_convex_bis}) with a mode at 0 and $Q = \mathcal{N}(0,\sigma^2 I_d)$ with $\sigma > 0$. Let $0 < \epsilon < 1$ and $\mu$ a $\beta$-warm distribution with respect to $\pi$. Assume that the log-weight function checks the local Lipschitz condition (\ref{eq:Lipschitz}) from assumption \ref{ass:iso_log_lip} then provided that
    $$C_{R_{\epsilon,\beta}} \leq \frac{\log 2 \sqrt{m}}{32}\eqsp,$$
    with
    $$
        R_{\epsilon,\beta} = \max\left(\sigma\sqrt{d} \mathrm{r}\left(\frac{\epsilon}{272},\beta,d\right), \sqrt{\frac{d}{m}} \mathrm{r}\left(\frac{\epsilon}{17},\beta,d\right)\right)\eqsp,
    $$
    where
    $$
        \mathrm{r}(\epsilon,\beta,d) =  2\left(1 + \max\left(\left(\frac{-\log(\frac{\epsilon}{2\beta})}{d}\right)^{1/4}, \sqrt{\frac{-\log(\frac{\epsilon}{2\beta})}{d}}\right)\right)\eqsp,
    $$
    then
    $$
        \tau_{mix}(\mu,\epsilon) \leq 128 \log\left(\frac{2\beta}{\epsilon}\right) \max\left(1, \frac{128^2 C_{R_{\epsilon,\beta}}^2}{\log(2)^2 m}\right) \eqsp.
    $$
\end{corollary}

\begin{proof}
    We use the fact that $\pi$ is a $m$-strongly log-concave distribution to deduce that $\psi(\pi) = \log 2 \sqrt{m}$ using \citep[Theorem~4.4]{Cousins2014}. Moreover, as the mode of $\pi$ is 0, we apply \citep[Lemma~1]{Dwivedi2018} again to obtain $R_1$ and $c_1$ for Theorem \ref{th:nice_th} with $c_1 = 1/17 < 1/16$.
\end{proof}

Proposition \ref{th:asympto_bound_main} in the main text corresponds to the asymptotic version of corollary \ref{th:nice_th_cor}.

\paragraph{Illustration with a Gaussian target}

Take $\pi = \mathcal{N}(0,I_d)$ and $Q = \mathcal{N}(0,(1+\lambda)^2 I_d)$ with $\lambda > 0$. $\lambda$ is the error term of the proposal. $\pi$ is $m$-strongly log-concave with $m=1$. The importance weight function $w$ can be computed 
$$w(x) = \frac{\pi(x)}{q(x)} = \sigma^{d/2} \exp\left(-\frac{1}{2} x^T \left(I_d - \frac{1}{(1+\lambda)^2} I_d\right)x\right)\eqsp,$$
so the log-weight function is
$$\tilde{w}(x) = -\frac{1}{2} x^T \left(I_d - \frac{1}{(1+\lambda)^2} I_d \right) x + \frac{d}{2} \log(\sigma)\eqsp.$$
Let $R > 0$ and $(x,y) \in \ball(0,R)^2$,
\begin{align*}
	\abs{\tilde{w}(x) - \tilde{w}(y)} &= \frac{1}{2} \abs{1 - \frac{1}{(1+\lambda)^2}} \abs{\sum_{k=1}^d (x_k^2 - y_k^2)} \\
	&\leq \frac{1}{2} \abs{1 - \frac{1}{(1+\lambda)^2}}  \sum_{k=1}^d \abs{x_k^2 - y_k^2} \\
	&\leq R \abs{1 - \frac{1}{(1+\lambda)^2}}  \sum_{k=1}^d \abs{x_k - y_k} \\
	&\leq R \abs{1 - \frac{1}{(1+\lambda)^2}}  \sqrt{d} \norm{x-y}\eqsp,
\end{align*}
where in the last line we used the Cauchy-Schwartz inequality $\sum_{k=1}^d \abs{x_k} = \langle \abs{x},1 \rangle \leq \abs{ \langle \abs{x},1 \rangle} \leq \norm{x} \sqrt{d}$. This gives us $C_R = R \sqrt{d} \abs{1 - 1/(\lambda+1)^2}$. Using corollary (\ref{th:nice_th_cor}), we find that asymptotically $R_{\epsilon,\beta} \isEquivTo{d \to \infty} 2\sqrt{d} (\lambda + 1)$ which leads to
$$
    C_{R_{\epsilon,\beta}} \isEquivTo{d \to \infty} 2 d (\lambda + 1) \abs{1 - \frac{1}{(1+\lambda)^2}}\eqsp.
$$
While checking the hypothesis of corollary (\ref{th:nice_th_cor}), the error of $Q$ need to scale as an inverse law of the dimension $d$
$$
    \lambda \isEquivTo{d \to \infty} \sqrt{\frac{K}{2d} + 1} - 1 \isEquivTo{d \to \infty} \frac 1 d \text{ where } K = \frac{\log2}{32}\eqsp,
$$
in order to keep the mixing time bellow a constant despite the increase in dimension.

\paragraph{Illustration with a non-isotropic Gaussian target}

Take $\pi = \mathcal{N}(0,\Sigma)$ with $\Sigma = \mathrm{diag}(c_1^2,\ldots,c_d^2)$ with $c_i > 0$ for all $i \in \{1, \ldots, d\}$ such $m \leq c_1 \leq \ldots \leq L$ with $m,L > 0$. Moreover, we set $Q = \mathcal{N}(0,\sigma^2 I_d)$ with $\sigma > 0$. For sake of simplicity, we assume that $L = m^{-1}$ and $m < 1$. Note that $m$ is the strong log-concave constant of $\pi$. With the same reasoning as the previous example, we have that 
$$w(x) = \frac{\pi(x)}{q(x)} = \sigma^{d/2} \exp\left(-\frac{1}{2} x^T \left(\Sigma^{-1} - \frac{1}{\sigma^2} I_d\right)x\right)\eqsp,$$
leading to
$$\tilde{w}(x) = -\frac{1}{2} x^T \left(\Sigma^{-1} - \frac{1}{\sigma^2} I_d \right) x + \frac{d}{2} \log(\sigma)\eqsp.$$
Let $R > 0$ and $(x,y) \in \ball(0,R)^2$,
\begin{align*}
	\abs{\tilde{w}(x) - \tilde{w}(y)} &= \frac{1}{2} \abs{\sum_{k=1}^d \left(\frac{1}{c_i^2} - \frac{1}{\sigma^2}\right)(x_k^2 - y_k^2)} \\
	&\leq R \max_{i \in \{1,\ldots,d\}} \abs{\frac{1}{c_i^2} - \frac{1}{\sigma^2}}  \sum_{k=1}^d \abs{x_k - y_k} \\
	&\leq R \tilde{c}_{\sigma}  \sqrt{d} \norm{x-y}\eqsp,
\end{align*}
where $\tilde{c}_{\sigma} = \max_{i \in \{1,\ldots,d\}} \abs{\frac{1}{c_i^2} - \frac{1}{\sigma^2}}$. This leads to
$$C_R = \sqrt{d} R \tilde{c}_{\sigma}\eqsp.$$
Using, corollary \ref{th:nice_th_cor} we find that asymptotically
$$C_R \isEquivTo{d \to \infty} 2 d \max\left(\sigma, \frac{1}{\sqrt{m}}\right)\tilde{c}_{\sigma}\eqsp.$$
We now choose $\sigma$ to minimize either the forward KL or the backward KL between $\mathcal{N}(0,\Sigma)$ and $\mathcal{N}(0,\sigma^2I_d)$.
\begin{align*}
    D_{KL}(\mathcal{N}(0,\Sigma), \mathcal{N}(0,\sigma^2I_d)) &=\frac{1}{2} \left(\frac{1}{\sigma^2} \sum_{i=1}^d c_i^2 - d + d \log \sigma^2 - \sum_{i=1}^d \log c_i^2\right)\eqsp, \\
    \frac{\rmd D_{KL}(\mathcal{N}(0,\Sigma), \mathcal{N}(0,\sigma^2I_d))}{\rmd \sigma^2}(\sigma^2) &= \frac{1}{2\sigma^2} \left(d - \frac{1}{\sigma^2} \sum_{i=1}^d c_i^2\right)\eqsp, \\
    D_{KL}(\mathcal{N}(0,\sigma^2 I_d), \mathcal{N}(0,\Sigma)) &= \frac{1}{2} \left(\sigma^2 \sum_{i=1}^d \frac{1}{c_i^2} - d + \sum_{i=1}^d c_i^2 - d \log \sigma^2\right)\eqsp, \\
    \frac{\rmd D_{KL}(\mathcal{N}(0,I_d), \mathcal{N}(0,\Sigma))}{\rmd \sigma^2}(\sigma^2) &= \frac{1}{2} \left(\sum_{i=1}^d \frac{1}{c_i^2} - \frac{d}{\sigma^2}\right)\eqsp,
\end{align*}
leading to $\sigma_f^2 = \sum_{i=1}^d c_i^2 / d$ being the minimizer of the forward KL and $\sigma_b^2 = d / \sum_{i=1}^d (1/c_i^2)$ being the minimizer of the backward KL. As mentionned in Sec. \ref{sec:exp_intuitions}, $\sigma_f^2$ corresponds to an over-spread Gaussian while $\sigma_b^2$ corresponds to an over-concentrated Gaussian. Using that $m \leq c_i^2 \leq L$ for all $i \in \{1, \ldots, d\}$, we can compute $\max\left(\sigma_f, 1/\sqrt{m}\right) = 1/\sqrt{m}$ and $\max\left(\sigma_b, 1/\sqrt{m}\right) = 1/\sqrt{m}$. Using the same property, we have that $\tilde{c}_{\sigma_f} \leq 1/m - m$ and $\tilde{c}_{\sigma_b} \leq 1/m - m$. 
Using the upper bound for the mixing time from corollary \ref{th:nice_th_cor} and ignoring all the constants, we get an upper bound in $\mathcal{O}\left(d^2 (1-m)^2 / m^4\right)$ for both $\sigma_f$ and $\sigma_b$ which is much more dimension sensitive than MALA's bound $\mathcal{O}(\sqrt{d} / m^2)$ \cite{Dwivedi2018}.

\section{Additional details on real world examples}

\subsection{Computational considerations} \label{app:computational}

\begin{figure}
    \centering
    \includegraphics[width=0.5\linewidth]{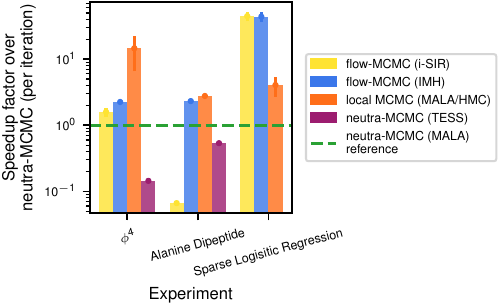}
    \caption{Speedup factor compared to \reparamalgs~on real world experiments}
    \label{fig:real_exp:speedup}
\end{figure}

\reparamalgs~methods should have a much higher computational cost compared to other methods. Figure \ref{fig:real_exp:speedup} shows that in two of our real world experiments, \reparamalgs~methods were significantly more expensive than \propalgs~methods. However, in the molecular system experiment (Sec. \ref{sec:aldp}) \propalgs~methods were more costly : that's because the distribution of this molecular system is very expensive to evaluate and breaks the parallel capabilities of \propalgs~methods such as i-SIR. Note that using ESS always led to high computational costs.

\subsection{Experimental details on logistic regression} \label{app:log_reg}

\paragraph{Logistic regression}

Consider a training dataset $\mathcal{D} = \{(x_k,y_k)\}_{j=1}^M$ where $x_k \in \mathbb{R}^d$ and $y_k \in \{-1,1\}$. In our case $\mathcal{D}$ is the german credit dataset ($M = 750$) where each $x_k$ represent 24 details about a person who takes a credit (age, sex, job, \ldots) and $y_k$ is $-1$ if this person is considered having bad credits risks and $1$ otherwise. We standardised this dataset so that each component of $x_k$ is between -1 and 1 and also added a constant $1$ component.

We consider that the likelihood of a pair $(x,y)$ is given by $p(y | x, w) = \mathrm{Bernoulli}(y ; \sigma(x^T w))$ where $w \in \mathbb{R}^d$ is a weight vector and $\sigma$ is the sigmoid function. Given a prior distribution $p(w)$, we sample the posterior distribution $p(w | \mathcal{D})$ and compute the posterior predictive distribution $p(y | x, \mathcal{D}) = \int p(y | x, w, \mathcal{D}) p(w | \mathcal{D}) dw \simeq 1/n \sum_{i=1}^n p(y | x, w_i, \mathcal{D}) p(w_i | \mathcal{D})$ for $(x,y) \in \mathcal{D}_{test}$.

Here, we consider a sparse prior suggested by \cite{Carvalho2009} which can be written as follows, $w = \tau \beta \circ \lambda$ and 
$$
    p(w) = p(\tau, \beta, \lambda) = \mathrm{Gamma}(\tau;\alpha = 0.5, \beta = 0.5)  \times \prod_{i=1}^d \mathrm{Gamma}(\lambda_i; \alpha = 0.5, \beta = 0.5)  \times \mathrm{Normal}(\beta_i; 0, 1)\eqsp.
$$
We log-transform the Gamma distributions by replacing $\mathrm{Gamma}$ with $\mathrm{Gamma}_{log}$ to ensure the positivity of the global scale $\tau$ and the local scale $\lambda$.
$$
    \mathrm{Gamma}_{log}(y;\alpha,\beta) = \mathrm{Gamma}(\exp{y};\alpha,\beta) \times \exp{y}\eqsp.
$$

\paragraph{Flow and training}

 The normalizing flow at stake is an Inverse Autoregressive Flow (IAF) \cite{Papamakarios2017} trained with the procedure described in \cite{Hoffman2019neutra} on a train dataset : it is a 3 layers deep flow using a residual neural networks wide of 51 neurons and deep of 2 layers with elu activation function. The flow was trained by optimizing the backward Kullback-Leiber with a learning rate of $10^{-2}$ (scaled by 10\% every 1000 optimization steps) using Adam optimizer for 5000 steps with a batch size of size 4096. 
 
\paragraph{Sampling details}

We used Pyro's \cite{Bingham2019} implementation of Hamilotonian Monte Carlo (HMC) to sample the previously described model. The warmup phase \footnote{In the following, warmup samples will be always discarded} lasted 64 steps and the MCMC chains were 256 steps long - this number was chosen as it guarantees a $\hat{R}$ close to 1. We automatically adapted both the step size and the mass matrix while the length of the trajectory (for the Leapfrog integrator) was frozen to 8 . Samplers involving i-SIR used $N = 100$ particles and importance sampling leveraged $N \times 256 = 25600$ particles. The global/local samplers interleaved 20 local steps between global steps.

\subsection{Experimental details on alanine dipeptide} \label{app:aldp}

\begin{figure}[t]
    \centering
    \includegraphics[width=0.25\linewidth]{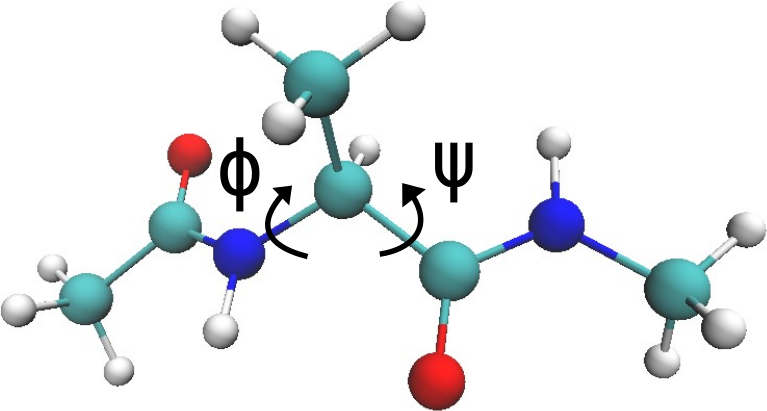}
    \caption{Visualization of alanine dipeptide and the dihedral angles $\phi$ and $\psi$ for the Ramachandran plot (taken from \cite{Midgley2022})}
    \label{fig:aldp:mol}
\end{figure}

The alanine dipeptide experiment aimed at sampling from the Boltzmann distribution on the 3D coordinates of the 22 atoms alanine dipeptide molecule (shown in Fig. \ref{fig:aldp:mol}). The details of the distribution can be found in \cite{Midgley2022}, we used their neural spline normalizing flow trained using FAB as the common flow in all flow-based samplers. The ground truth was obtained through parallel tempering MD simulations from the same paper. We used 512 chains of length 600 using the same hyperparameters as the mixture of Gaussians : $N = 128$, $n_{local} = 5$ and a target acceptance of 75\%. 

\subsection{Experimental details on $\phi^4$} \label{app:phi_four}

\paragraph{Target distribution}
Following \cite{Gabrie2022} we consider a 1d $\phi^4$ on a grid of $d$ points with Dirichlet boundary conditions at both extremeties of the field; that is $\phi_0 = \phi_{d+1} = 0$ where one configuration is the $d$-dimensional vector $( \phi_i )_{i=1}^d$. The negative logarithm of the density is given by
\begin{equation}
    - \ln \pi(\phi) = \beta \left( \frac{ad}{2}  \sum_{i=1}^{d+1} (\phi_i-\phi_{i-1})^2 + \frac{1}{4ad} \sum_{i=1}^d (1-\phi_i^2)^2 \right)
\end{equation}
with a flow a colored normal distribution as base built by keeping the quadratic terms from above: 
\begin{equation}
    - \ln \base(\phi) = \beta \left(\frac{ad}{2} \sum_{i=1}^{d+1} (\phi_i-\phi_{i-1})^2 + \frac{1}{2 ad} \sum_{i=1}^d \phi_i^2\right),
\end{equation}
also with Dirichlet boundary conditions at 0. We chose parameter values for which the system is bimodal: $a=0.1$ and inverse temperature $\beta=20$ (see Fig. \ref{fig:phi_four:recap} in the main paper).

\paragraph{Training of the normalizing flow}

The flows at stake are again RealNVPs and their hyper-parameters are given in table \ref{table:phi_four:flow_params}. We used Adam optimizer to minimize the forward KL. This forward KL was approximated by taken exactly 50\% of samples in each mode by running MALA locally. The size of those chains is $\frac{\text{batch size}}{\text{\# of parallel MCMC chains}}$.

\paragraph{Sampling parameters}

The sampling parameters were selected using a grid-search which optimized the quality of the mode weight estimation. The starting sample of each chain was taken to be in a single mode \footnote{This is achieved by running a gradient descent starting from a constant unit configuration}. We used 256 MCMC chains of length 512 where the first half of the chain was discarded for warmup purposes. The sampling hyperparameters are detailed in table \ref{table:phi_four:spl_params}.

\begin{table}[H]
    \caption{RealNVP hyperparameters for the $\phi^4$ experiment.}
    \label{table:phi_four:flow_params}
    \begin{center}
        \begin{small}
            \begin{tabular}{c | c | c | c | c | c | c | c | c}
                \toprule
                \thead{Dimension} & \thead{Depth of \\ the flow} & \thead{\# of layers \\ in NN} & \thead{\# of neuron \\ per layer} & \thead{\# of \\ training steps} & \thead{Learning rate} & \thead{Learning rate \\ decay factor} & \thead{Batch size} & \thead{\# MCMC \\ chains} \\
                \midrule
                64 & 5 & 3 & 128 & $10^4$ & $10^{-2}$ & 0.98 & 4096 & 64 \\
                128 & 6 & 256 & 3 & $2 \times 10^4$ & $10^{-2}$ & 0.98 & 4096 & 64 \\
                256 & 10 & 512 & 4 & $3.5 \times 10^4$ & $5 \times 10^{-3}$ & 0.98 & 8192 & 64 \\
                512 & 15 & 512 & 5 & $6 \times 10^4$ & $5 \times 10^{-4}$ & 0.99 & 8192 & 32 \\
                \bottomrule
            \end{tabular}
        \end{small}
    \end{center}
\end{table}

\begin{table}[H]
    \caption{Sampling hyperparameters for the $\phi^4$ experiment.}
    \label{table:phi_four:spl_params}
    \begin{center}
        \begin{small}
            \begin{tabular}{c | c | c | c | c | c | c }
                \toprule
                \thead{Dimension} & \thead{$n$} & \thead{$N$ \\ (\propalgs)} & \thead{$N$ \\ (\neutraflow)} & \thead{$N$ \\ (IS)} & \thead{$n_{local}$ \\ (\propalgs)} & \thead{$n_{local}$ \\ (\neutraflow)} \\ 
                \midrule
                64 & 256 ($\times 2$) & 60 & 60 & 3840 & 25 & 25 \\
                128 & 256 ($\times 2$) & 80 & 80 & 5120 & 25 & 25 \\
                256 & 256 ($\times 2$) & 100 & 100 & 6400 & 25 & 25 \\
                512 & 256 ($\times 2$) & 120 & 120 & 7680 & 25 & 25 \\
                \bottomrule
            \end{tabular}
        \end{small}
    \end{center}
\end{table}

\section{Sampling image distributions}\label{app:cifar10}

\begin{figure}[t]
    \centering
    \includegraphics{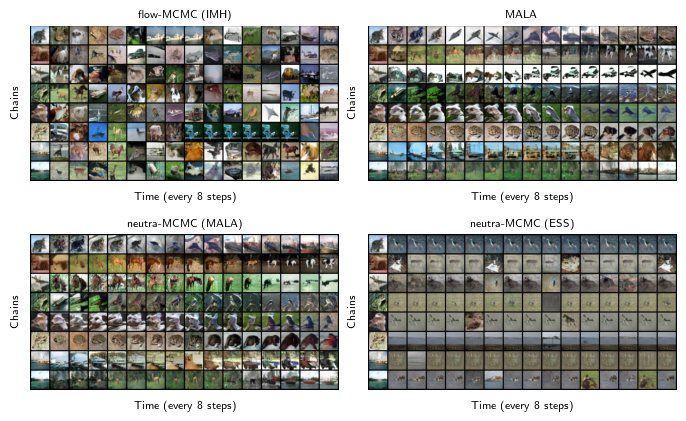}
    \caption{MCMC chains sampled from a SN-GAN as an energy-based model}
    \label{fig:samples}
\end{figure}

\begin{figure}[t]
    \centering
    \includegraphics[width=0.75\linewidth]{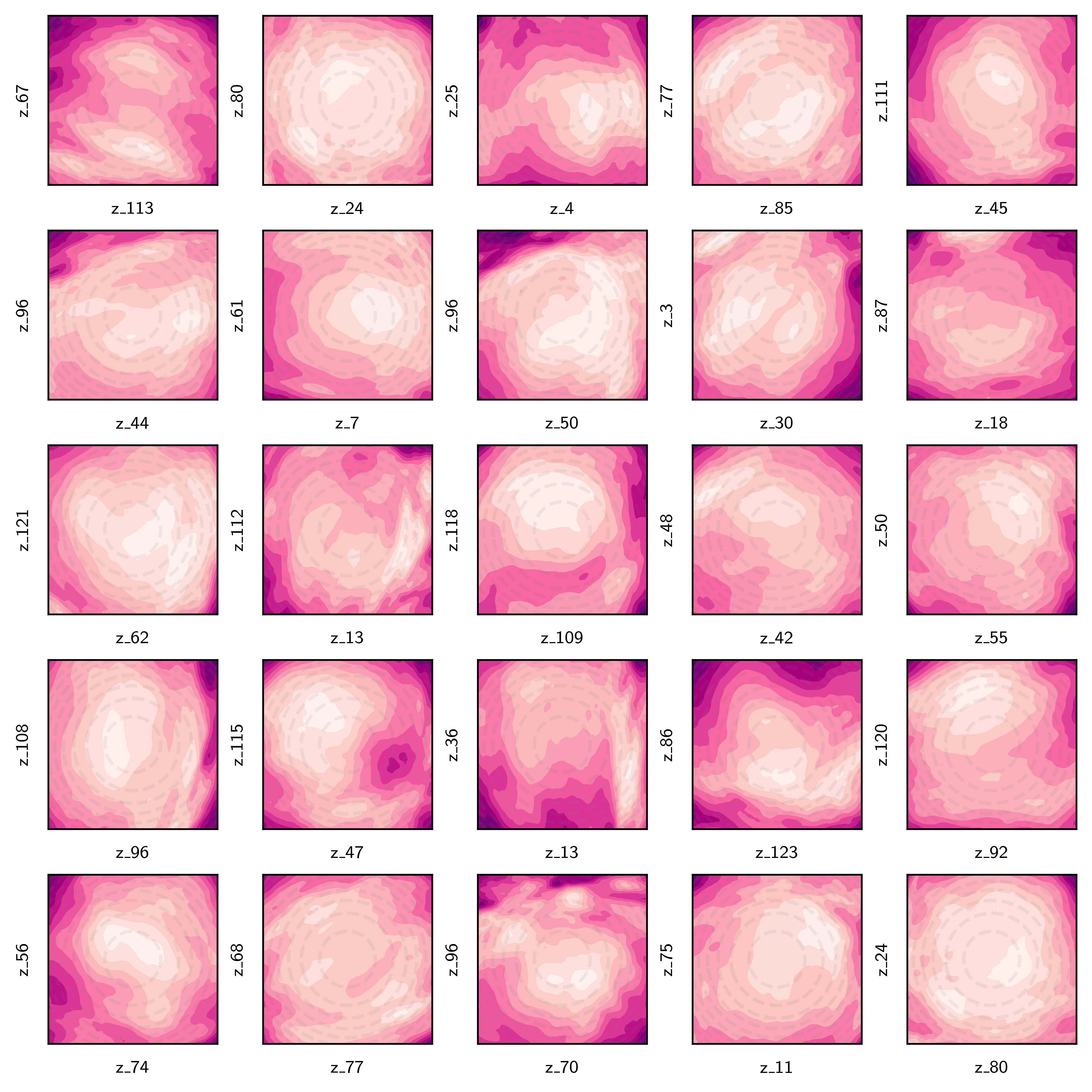}
    \caption{Random 2D slices of $E_{GAN}$. The purple colors reperensent the level lines of $E_{GAN}(z)$ and the grey lines are the level lines of $z \mapsto -\log p_0(z)$}
    \label{fig:density}
\end{figure}

\paragraph{Design of the target distribution}

We trained a SN-GAN \cite{sngan} on the CIFAR10 dataset for 100.000 epochs (implementation from \href{https://github.com/kwotsin/mimicry}{https://github.com/kwotsin/mimicry}). The latent space of the SNGAN is of dimension 128. We use a multivariate centered standard Gaussian as a base distribution. Following \cite{gan_ebm}, we define a probability measure as an energy-based model on the latent space of the GAN
$$
    p(z) = \exp(-E_{GAN}(z)) / Z_{\theta} \text { with } E_{GAN}(z) = -\log p_0(z) - \operatorname{logit}(D(G(z))),
$$
where $p_0$ is the base distribution, $G$ is the generator, $D$ is the discriminator and $\operatorname{logit}(y) = \log(y / (1-y))$ is the inverse sigmoid function. Slices of $E_{GAN}$ can be found on Fig. \ref{fig:density} : $p(z)$ is multimodal and has many bad geometries. The samples of $p(z)$ can be transformed in images (which belong to $\mathbb{R}^{3 \times 32 \times 32}$) by pushing them through $G$.

\paragraph{Sampling procedure}

We sampled $p(z)$ using 4 different MCMC algorithms. Fig. \ref{fig:samples} depicts 8 chains for each sampler started in the same state. The flow based samplers use a RealNVP normalizing flow which has 8 layers each one having a MLP wide of 128 neurons and deep of 3 layers. This normalizing flow was trained using the backward KL loss for 512 iterations using batches of size 1024. The chains are 128 samples long (on Fig. \ref{fig:samples}, the chains are subsampled every 8 steps) and have the same hyper-parameters as in the four Gaussians experiment in dimension 128.

Fig. \ref{fig:samples} shows that \propalgs~algorithm mixes between the modes of this moderately-high dimensional distribution ($d=128$) more often than its \reparamalgs~counterparts. This observation is consistent with the results of Sec. \ref{sec:exp_intuitions} and Sec. \ref{sec:realtasks}.

\section{Amount of computation and type of resources}

We used a single type of GPU, the Nvidia A100. Each experiment of Section~\ref{sec:exp_intuitions} took about 2 hours to run when distributed on 4 GPUs. Moreover, the training of the 5 RealNVPs for Section~\ref{sec:mog} took 6 hours on 4 GPUs. Regarding the real world experiments (see Section~\ref{sec:realtasks}), training the flow for the molecular experiment took 24 hours on a single GPU and 6 hours to perform the sampling part on 4 GPUs. Training the flows for the $\phi^4$ experiments took 12 hours on 2 GPUs and 1 hours on 4 GPUs for sampling. Finally, training the flow for the bayesian logistic regression took less than a minute but sampling with HMC lasted 6 hours on 4 GPUs. In total, this work required 55 hours of parallel GPU run-time.

%% file: ms.bbl
\begin{thebibliography}{80}
\providecommand{\natexlab}[1]{#1}
\providecommand{\url}[1]{\texttt{#1}}
\expandafter\ifx\csname urlstyle\endcsname\relax
  \providecommand{\doi}[1]{doi: #1}\else
  \providecommand{\doi}{doi: \begingroup \urlstyle{rm}\Url}\fi

\bibitem[Abbott et~al.(2022{\natexlab{a}})Abbott, Albergo, Botev, Boyda,
  Cranmer, Hackett, Matthews, Racanière, Razavi, Rezende, Romero-López,
  Shanahan, and Urban]{abbott_aspects_2022}
Abbott, R., Albergo, M.~S., Botev, A., Boyda, D., Cranmer, K., Hackett, D.~C.,
  Matthews, A. G. D.~G., Racanière, S., Razavi, A., Rezende, D.~J.,
  Romero-López, F., Shanahan, P.~E., and Urban, J.~M.
\newblock Aspects of scaling and scalability for flow-based sampling of lattice
  {QCD}, November 2022{\natexlab{a}}.
\newblock URL \url{http://arxiv.org/abs/2211.07541}.
\newblock arXiv:2211.07541 [cond-mat, physics:hep-lat].

\bibitem[Abbott et~al.(2022{\natexlab{b}})Abbott, Albergo, Boyda, Cranmer,
  Hackett, Kanwar, Racanière, Rezende, Romero-López, Shanahan, Tian, and
  Urban]{abbott_gauge-equivariant_2022}
Abbott, R., Albergo, M.~S., Boyda, D., Cranmer, K., Hackett, D.~C., Kanwar, G.,
  Racanière, S., Rezende, D.~J., Romero-López, F., Shanahan, P.~E., Tian, B.,
  and Urban, J.~M.
\newblock Gauge-equivariant flow models for sampling in lattice field theories
  with pseudofermions.
\newblock \emph{Physical Review D}, 106\penalty0 (7):\penalty0 074506, October
  2022{\natexlab{b}}.
\newblock ISSN 2470-0010, 2470-0029.
\newblock \doi{10.1103/PhysRevD.106.074506}.
\newblock URL \url{https://link.aps.org/doi/10.1103/PhysRevD.106.074506}.

\bibitem[Agapiou et~al.(2017)Agapiou, Papaspiliopoulos, Sanz-Alonso, and
  Stuart]{Agapiou2017}
Agapiou, S., Papaspiliopoulos, O., Sanz-Alonso, D., and Stuart, A.~M.
\newblock Importance sampling: Intrinsic dimension and computational cost.
\newblock \emph{Statistical Science}, 32\penalty0 (3):\penalty0 405--431, 2017.
\newblock ISSN 08834237, 21688745.
\newblock URL \url{http://www.jstor.org/stable/26408299}.

\bibitem[Albergo et~al.(2019)Albergo, Kanwar, and Shanahan]{Albergo2019}
Albergo, M.~S., Kanwar, G., and Shanahan, P.~E.
\newblock {Flow-based generative models for Markov chain Monte Carlo in lattice
  field theory}.
\newblock \emph{Physical Review D}, 100\penalty0 (3):\penalty0 034515, aug
  2019.
\newblock ISSN 2470-0010.
\newblock \doi{10.1103/PhysRevD.100.034515}.
\newblock URL \url{http://arxiv.org/abs/2106.05934
  https://link.aps.org/doi/10.1103/PhysRevD.100.034515}.

\bibitem[Andrieu et~al.(2010)Andrieu, Doucet, and Holenstein]{Andrieu2010}
Andrieu, C., Doucet, A., and Holenstein, R.
\newblock Particle {M}arkov chain {M}onte {C}arlo methods.
\newblock \emph{Journal of the Royal Statistical Society: Series B},
  72\penalty0 (3):\penalty0 269--342, 2010.

\bibitem[Arbel et~al.(2021)Arbel, Matthews, and Doucet]{arbel_annealed_2021}
Arbel, M., Matthews, A., and Doucet, A.
\newblock Annealed {Flow} {Transport} {Monte} {Carlo}.
\newblock In \emph{Proceedings of the 38th {International} {Conference} on
  {Machine} {Learning}}, pp.\  318--330. PMLR, July 2021.
\newblock URL \url{https://proceedings.mlr.press/v139/arbel21a.html}.
\newblock ISSN: 2640-3498.

\bibitem[Bingham et~al.(2019)Bingham, Chen, Jankowiak, Obermeyer, Pradhan,
  Karaletsos, Singh, Szerlip, Horsfall, and Goodman]{Bingham2019}
Bingham, E., Chen, J.~P., Jankowiak, M., Obermeyer, F., Pradhan, N.,
  Karaletsos, T., Singh, R., Szerlip, P.~A., Horsfall, P., and Goodman, N.~D.
\newblock Pyro: Deep universal probabilistic programming.
\newblock \emph{J. Mach. Learn. Res.}, 20:\penalty0 28:1--28:6, 2019.
\newblock URL \url{http://jmlr.org/papers/v20/18-403.html}.

\bibitem[Bobkov(1999)]{Bobkov1999}
Bobkov, S.~G.
\newblock {Isoperimetric and Analytic Inequalities for Log-Concave Probability
  Measures}.
\newblock \emph{The Annals of Probability}, 27\penalty0 (4):\penalty0 1903 --
  1921, 1999.
\newblock \doi{10.1214/aop/1022874820}.
\newblock URL \url{https://doi.org/10.1214/aop/1022874820}.

\bibitem[Bogachev et~al.(2005)Bogachev, Kolesnikov, and Medvedev]{Bogachev2005}
Bogachev, V.~I., Kolesnikov, A.~V., and Medvedev, K.~V.
\newblock Triangular transformations of measures.
\newblock \emph{Sbornik: Mathematics}, 196\penalty0 (3):\penalty0 309, apr
  2005.
\newblock \doi{10.1070/SM2005v196n03ABEH000882}.
\newblock URL \url{https://dx.doi.org/10.1070/SM2005v196n03ABEH000882}.

\bibitem[Bovier et~al.(2005)Bovier, Klein, and Gayrard]{Bovier2005}
Bovier, A., Klein, M., and Gayrard, V.
\newblock Metastability in reversible diffusion processes ii. precise
  asymptotics for small eigenvalues.
\newblock \emph{Journal of the European Mathematical Society}, 7:\penalty0
  69--99, 10 2005.
\newblock \doi{10.4171/JEMS/22}.

\bibitem[Brenier(1991)]{Brenier1991}
Brenier, Y.
\newblock Polar factorization and monotone rearrangement of vector-valued
  functions.
\newblock \emph{Communications on Pure and Applied Mathematics}, 44\penalty0
  (4):\penalty0 375--417, 1991.
\newblock \doi{https://doi.org/10.1002/cpa.3160440402}.
\newblock URL
  \url{https://onlinelibrary.wiley.com/doi/abs/10.1002/cpa.3160440402}.

\bibitem[Brown \& Jones(2021)Brown and Jones]{Brown2021}
Brown, A. and Jones, G.~L.
\newblock Exact convergence analysis for metropolis-hastings independence
  samplers in wasserstein distances, 2021.

\bibitem[Cabezas \& Nemeth(2022)Cabezas and Nemeth]{Cabezas2022}
Cabezas, A. and Nemeth, C.
\newblock Transport elliptical slice sampling, 2022.
\newblock URL \url{https://arxiv.org/abs/2210.10644}.

\bibitem[Carvalho et~al.(2009)Carvalho, Polson, and Scott]{Carvalho2009}
Carvalho, C.~M., Polson, N.~G., and Scott, J.~G.
\newblock Handling sparsity via the horseshoe.
\newblock In van Dyk, D. and Welling, M. (eds.), \emph{Proceedings of the
  Twelth International Conference on Artificial Intelligence and Statistics},
  volume~5 of \emph{Proceedings of Machine Learning Research}, pp.\  73--80,
  Hilton Clearwater Beach Resort, Clearwater Beach, Florida USA, 16--18 Apr
  2009. PMLR.
\newblock URL \url{https://proceedings.mlr.press/v5/carvalho09a.html}.

\bibitem[Chandrasekaran et~al.()Chandrasekaran, Dadush, and
  Vempala]{Karthekeyan2010}
Chandrasekaran, K., Dadush, D., and Vempala, S.
\newblock \emph{Thin Partitions: Isoperimetric Inequalities and a Sampling
  Algorithm for Star Shaped Bodies}, pp.\  1630--1645.
\newblock \doi{10.1137/1.9781611973075.133}.
\newblock URL \url{https://epubs.siam.org/doi/abs/10.1137/1.9781611973075.133}.

\bibitem[Che et~al.(2020)Che, ZHANG, Sohl-Dickstein, Larochelle, Paull, Cao,
  and Bengio]{gan_ebm}
Che, T., ZHANG, R., Sohl-Dickstein, J., Larochelle, H., Paull, L., Cao, Y., and
  Bengio, Y.
\newblock Your gan is secretly an energy-based model and you should use
  discriminator driven latent sampling.
\newblock In Larochelle, H., Ranzato, M., Hadsell, R., Balcan, M., and Lin, H.
  (eds.), \emph{Advances in Neural Information Processing Systems}, volume~33,
  pp.\  12275--12287. Curran Associates, Inc., 2020.
\newblock URL
  \url{https://proceedings.neurips.cc/paper_files/paper/2020/file/90525e70b7842930586545c6f1c9310c-Paper.pdf}.

\bibitem[Chen et~al.(2019)Chen, Behrmann, Duvenaud, and Jacobsen]{Chen2019}
Chen, R. T.~Q., Behrmann, J., Duvenaud, D.~K., and Jacobsen, J.-H.
\newblock Residual flows for invertible generative modeling.
\newblock In Wallach, H., Larochelle, H., Beygelzimer, A., d\textquotesingle
  Alch\'{e}-Buc, F., Fox, E., and Garnett, R. (eds.), \emph{Advances in Neural
  Information Processing Systems}, volume~32. Curran Associates, Inc., 2019.
\newblock URL
  \url{https://proceedings.neurips.cc/paper/2019/file/5d0d5594d24f0f955548f0fc0ff83d10-Paper.pdf}.

\bibitem[Chen et~al.(2020)Chen, Dwivedi, Wainwright, and Yu]{Chen2020}
Chen, Y., Dwivedi, R., Wainwright, M.~J., and Yu, B.
\newblock Fast mixing of metropolized hamiltonian monte carlo: Benefits of
  multi-step gradients.
\newblock \emph{Journal of Machine Learning Research}, 21\penalty0
  (92):\penalty0 1--72, 2020.
\newblock URL \url{http://jmlr.org/papers/v21/19-441.html}.

\bibitem[Chewi et~al.(2021)Chewi, Lu, Ahn, Cheng, Gouic, and
  Rigollet]{Chewi2020}
Chewi, S., Lu, C., Ahn, K., Cheng, X., Gouic, T.~L., and Rigollet, P.
\newblock Optimal dimension dependence of the metropolis-adjusted langevin
  algorithm.
\newblock In Belkin, M. and Kpotufe, S. (eds.), \emph{Proceedings of Thirty
  Fourth Conference on Learning Theory}, volume 134 of \emph{Proceedings of
  Machine Learning Research}, pp.\  1260--1300. PMLR, 15--19 Aug 2021.
\newblock URL \url{https://proceedings.mlr.press/v134/chewi21a.html}.

\bibitem[Cornish et~al.(2020)Cornish, Caterini, Deligiannidis, and
  Doucet]{Cornish2020}
Cornish, R., Caterini, A., Deligiannidis, G., and Doucet, A.
\newblock Relaxing bijectivity constraints with continuously indexed
  normalising flows.
\newblock In III, H.~D. and Singh, A. (eds.), \emph{Proceedings of the 37th
  International Conference on Machine Learning}, volume 119 of
  \emph{Proceedings of Machine Learning Research}, pp.\  2133--2143. PMLR,
  13--18 Jul 2020.
\newblock URL \url{https://proceedings.mlr.press/v119/cornish20a.html}.

\bibitem[Cousins \& Vempala()Cousins and Vempala]{Cousins2014}
Cousins, B. and Vempala, S.
\newblock \emph{A Cubic Algorithm for Computing Gaussian Volume}, pp.\
  1215--1228.
\newblock \doi{10.1137/1.9781611973402.90}.
\newblock URL \url{https://epubs.siam.org/doi/abs/10.1137/1.9781611973402.90}.

\bibitem[Craiu \& Lemieux(2007)Craiu and Lemieux]{Craiu2007}
Craiu, R.~V. and Lemieux, C.
\newblock Acceleration of the multiple-try metropolis algorithm using
  antithetic and stratified sampling.
\newblock \emph{Statistics and Computing}, 17\penalty0 (2):\penalty0 109--120,
  Jun 2007.
\newblock ISSN 1573-1375.
\newblock \doi{10.1007/s11222-006-9009-4}.
\newblock URL \url{https://doi.org/10.1007/s11222-006-9009-4}.

\bibitem[Del~Debbio et~al.(2021)Del~Debbio, Rossney, and
  Wilson]{del_debbio_efficient_2021}
Del~Debbio, L., Rossney, J.~M., and Wilson, M.
\newblock Efficient {Modelling} of {Trivializing} {Maps} for {Lattice}
  \${\textbackslash}phi{\textasciicircum}4\$ {Theory} {Using} {Normalizing}
  {Flows}: {A} {First} {Look} at {Scalability}.
\newblock \emph{Physical Review D}, 104\penalty0 (9):\penalty0 094507, November
  2021.
\newblock ISSN 2470-0010, 2470-0029.
\newblock \doi{10.1103/PhysRevD.104.094507}.
\newblock URL \url{http://arxiv.org/abs/2105.12481}.
\newblock arXiv:2105.12481 [hep-lat].

\bibitem[Del~Moral et~al.(2006)Del~Moral, Doucet, and
  Jasra]{del_moral_sequential_2006}
Del~Moral, P., Doucet, A., and Jasra, A.
\newblock Sequential {Monte} {Carlo} samplers.
\newblock \emph{Journal of the Royal Statistical Society: Series B (Statistical
  Methodology)}, 68\penalty0 (3):\penalty0 411--436, June 2006.
\newblock ISSN 1369-7412, 1467-9868.
\newblock \doi{10.1111/j.1467-9868.2006.00553.x}.
\newblock URL
  \url{https://onlinelibrary.wiley.com/doi/10.1111/j.1467-9868.2006.00553.x}.

\bibitem[Ding \& Zhang(2021)Ding and Zhang]{ding_deepbar_2021}
Ding, X. and Zhang, B.
\newblock {DeepBAR}: {A} {Fast} and {Exact} {Method} for {Binding} {Free}
  {Energy} {Computation}.
\newblock \emph{The Journal of Physical Chemistry Letters}, 12\penalty0
  (10):\penalty0 2509--2515, March 2021.
\newblock ISSN 1948-7185, 1948-7185.
\newblock \doi{10.1021/acs.jpclett.1c00189}.
\newblock URL \url{https://pubs.acs.org/doi/10.1021/acs.jpclett.1c00189}.

\bibitem[Dinh et~al.(2016)Dinh, Sohl-Dickstein, and Bengio]{Dinh2017density}
Dinh, L., Sohl-Dickstein, J., and Bengio, S.
\newblock Density estimation using real nvp, 2016.
\newblock URL \url{https://arxiv.org/abs/1605.08803}.

\bibitem[Dua \& Graff(2017)Dua and Graff]{Dua2019}
Dua, D. and Graff, C.
\newblock {UCI} machine learning repository, 2017.
\newblock URL \url{http://archive.ics.uci.edu/ml}.

\bibitem[Duane et~al.(1987)Duane, Kennedy, Pendleton, and
  Roweth]{duane_hybrid_1987}
Duane, S., Kennedy, A.~D., Pendleton, B.~J., and Roweth, D.
\newblock Hybrid {Monte} {Carlo}.
\newblock \emph{Physics Letters B}, 195\penalty0 (2):\penalty0 216--222,
  September 1987.
\newblock ISSN 0370-2693.
\newblock \doi{10.1016/0370-2693(87)91197-X}.
\newblock URL
  \url{https://www.sciencedirect.com/science/article/pii/037026938791197X}.

\bibitem[Dwivedi et~al.(2018)Dwivedi, Chen, Wainwright, and Yu]{Dwivedi2018}
Dwivedi, R., Chen, Y., Wainwright, M.~J., and Yu, B.
\newblock Log-concave sampling: Metropolis-hastings algorithms are fast!
\newblock In Bubeck, S., Perchet, V., and Rigollet, P. (eds.),
  \emph{Proceedings of the 31st Conference On Learning Theory}, volume~75 of
  \emph{Proceedings of Machine Learning Research}, pp.\  793--797. PMLR, 06--09
  Jul 2018.
\newblock URL \url{https://proceedings.mlr.press/v75/dwivedi18a.html}.

\bibitem[Gabrié et~al.(2022)Gabrié, Rotskoff, and Vanden-Eijnden]{Gabrie2022}
Gabrié, M., Rotskoff, G.~M., and Vanden-Eijnden, E.
\newblock Adaptive monte carlo augmented with normalizing flows.
\newblock \emph{Proceedings of the National Academy of Sciences}, 119\penalty0
  (10):\penalty0 e2109420119, 2022.
\newblock \doi{10.1073/pnas.2109420119}.
\newblock URL \url{https://www.pnas.org/doi/abs/10.1073/pnas.2109420119}.

\bibitem[Gelman \& Rubin(1992)Gelman and Rubin]{Gelman1992}
Gelman, A. and Rubin, D.~B.
\newblock {Inference from Iterative Simulation Using Multiple Sequences}.
\newblock \emph{Statistical Science}, 7\penalty0 (4):\penalty0 457 -- 472,
  1992.
\newblock \doi{10.1214/ss/1177011136}.
\newblock URL \url{https://doi.org/10.1214/ss/1177011136}.

\bibitem[Girolami \& Calderhead(2011)Girolami and Calderhead]{Girolami2011}
Girolami, M. and Calderhead, B.
\newblock Riemann manifold langevin and hamiltonian monte carlo methods.
\newblock \emph{Journal of the Royal Statistical Society. Series B (Statistical
  Methodology)}, 73\penalty0 (2):\penalty0 123--214, 2011.
\newblock ISSN 13697412, 14679868.

\bibitem[Grumitt et~al.(2022)Grumitt, Dai, and Seljak]{Grumitt2022}
Grumitt, R. D.~P., Dai, B., and Seljak, U.
\newblock Deterministic langevin monte carlo with normalizing flows for
  bayesian inference, 2022.
\newblock URL \url{https://arxiv.org/abs/2205.14240}.

\bibitem[Hackett et~al.(2021)Hackett, Hsieh, Albergo, Boyda, Chen, Chen,
  Cranmer, Kanwar, and Shanahan]{hackett_flow-based_2021}
Hackett, D.~C., Hsieh, C.-C., Albergo, M.~S., Boyda, D., Chen, J.-W., Chen,
  K.-F., Cranmer, K., Kanwar, G., and Shanahan, P.~E.
\newblock Flow-based sampling for multimodal distributions in lattice field
  theory, July 2021.
\newblock URL \url{http://arxiv.org/abs/2107.00734}.
\newblock arXiv:2107.00734 [cond-mat, physics:hep-lat].

\bibitem[Hoffman et~al.(2019)Hoffman, Sountsov, Dillon, Langmore, Tran, and
  Vasudevan]{Hoffman2019neutra}
Hoffman, M.~D., Sountsov, P., Dillon, J.~V., Langmore, I., Tran, D., and
  Vasudevan, S.
\newblock {NeuTra-lizing Bad Geometry in Hamiltonian Monte Carlo Using Neural
  Transport}.
\newblock In \emph{1st Symposium on Advances in Approximate Bayesian Inference,
  2018 1–5}, 2019.
\newblock URL \url{http://arxiv.org/abs/1903.03704}.

\bibitem[Invernizzi et~al.(2022)Invernizzi, Krämer, Clementi, and
  Noé]{invernizzi_skipping_2022}
Invernizzi, M., Krämer, A., Clementi, C., and Noé, F.
\newblock Skipping the {Replica} {Exchange} {Ladder} with {Normalizing}
  {Flows}.
\newblock \emph{The Journal of Physical Chemistry Letters}, 13\penalty0
  (50):\penalty0 11643--11649, December 2022.
\newblock \doi{10.1021/acs.jpclett.2c03327}.
\newblock URL \url{https://doi.org/10.1021/acs.jpclett.2c03327}.
\newblock Publisher: American Chemical Society.

\bibitem[Izmailov et~al.(2020)Izmailov, Kirichenko, Finzi, and
  Wilson]{izmailov_semi-supervised_2020}
Izmailov, P., Kirichenko, P., Finzi, M., and Wilson, A.~G.
\newblock Semi-{Supervised} {Learning} with {Normalizing} {Flows}.
\newblock \emph{Proceedings of the 37 th International Conference on Machine
  Learning}, PMLR 119, 2020.

\bibitem[Jerfel et~al.(2021)Jerfel, Wang, Wong-Fannjiang, Heller, Ma, and
  Jordan]{jerfel_variational_2021}
Jerfel, G., Wang, S., Wong-Fannjiang, C., Heller, K.~A., Ma, Y., and Jordan,
  M.~I.
\newblock Variational refinement for importance sampling using the forward
  kullback-leibler divergence.
\newblock In \emph{Uncertainty in {Artificial} {Intelligence}}, pp.\
  1819--1829. PMLR, 2021.

\bibitem[Jia \& Seljak(2019)Jia and Seljak]{jia_normalizing_2019}
Jia, H. and Seljak, U.
\newblock Normalizing {Constant} {Estimation} with {Gaussianized} {Bridge}
  {Sampling}, December 2019.
\newblock URL \url{http://arxiv.org/abs/1912.06073}.
\newblock arXiv:1912.06073 [astro-ph, stat].

\bibitem[Karamanis et~al.(2022)Karamanis, Beutler, Peacock, Nabergoj, and
  Seljak]{karamanis_accelerating_2022}
Karamanis, M., Beutler, F., Peacock, J.~A., Nabergoj, D., and Seljak, U.
\newblock Accelerating astronomical and cosmological inference with
  {Preconditioned} {Monte} {Carlo}.
\newblock \emph{Monthly Notices of the Royal Astronomical Society},
  516\penalty0 (2):\penalty0 1644--1653, September 2022.
\newblock ISSN 0035-8711, 1365-2966.
\newblock \doi{10.1093/mnras/stac2272}.
\newblock URL \url{http://arxiv.org/abs/2207.05652}.
\newblock arXiv:2207.05652 [astro-ph, physics:physics].

\bibitem[Kingma \& Ba(2014)Kingma and Ba]{Kingma2014}
Kingma, D.~P. and Ba, J.
\newblock Adam: A method for stochastic optimization, 2014.
\newblock URL \url{https://arxiv.org/abs/1412.6980}.

\bibitem[Kobyzev et~al.(2021)Kobyzev, Prince, and Brubaker]{Kobyzev2021}
Kobyzev, I., Prince, S.~J., and Brubaker, M.~A.
\newblock Normalizing flows: An introduction and review of current methods.
\newblock \emph{{IEEE} Transactions on Pattern Analysis and Machine
  Intelligence}, 43\penalty0 (11):\penalty0 3964--3979, nov 2021.
\newblock \doi{10.1109/tpami.2020.2992934}.
\newblock URL \url{https://doi.org/10.1109%2Ftpami.2020.2992934}.

\bibitem[Laddha \& Vempala(2020)Laddha and Vempala]{Laddha2020}
Laddha, A. and Vempala, S.
\newblock Convergence of gibbs sampling: Coordinate hit-and-run mixes fast,
  2020.
\newblock URL \url{https://arxiv.org/abs/2009.11338}.

\bibitem[Liu et~al.(2000)Liu, Liang, and Wong]{Liu2000}
Liu, J.~S., Liang, F., and Wong, W.~H.
\newblock The multiple-try method and local optimization in metropolis
  sampling.
\newblock \emph{Journal of the American Statistical Association}, 95\penalty0
  (449):\penalty0 121--134, 2000.
\newblock ISSN 01621459.
\newblock URL \url{http://www.jstor.org/stable/2669532}.

\bibitem[Lov{\'a}sz(1999)]{Lovasz1999}
Lov{\'a}sz, L.
\newblock Hit-and-run mixes fast.
\newblock \emph{Mathematical Programming}, 86\penalty0 (3):\penalty0 443--461,
  Dec 1999.
\newblock ISSN 1436-4646.
\newblock \doi{10.1007/s101070050099}.
\newblock URL \url{https://doi.org/10.1007/s101070050099}.

\bibitem[Lovász \& Simonovits(1993)Lovász and Simonovits]{Lovasz1993}
Lovász, L. and Simonovits, M.
\newblock Random walks in a convex body and an improved volume algorithm.
\newblock \emph{Random Structures \& Algorithms}, 4\penalty0 (4):\penalty0
  359--412, 1993.
\newblock \doi{https://doi.org/10.1002/rsa.3240040402}.
\newblock URL
  \url{https://onlinelibrary.wiley.com/doi/abs/10.1002/rsa.3240040402}.

\bibitem[Lovász \& Vempala(2007)Lovász and Vempala]{Vempala2007}
Lovász, L. and Vempala, S.
\newblock The geometry of logconcave functions and sampling algorithms.
\newblock \emph{Random Structures \& Algorithms}, 30\penalty0 (3):\penalty0
  307--358, 2007.
\newblock \doi{https://doi.org/10.1002/rsa.20135}.
\newblock URL \url{https://onlinelibrary.wiley.com/doi/abs/10.1002/rsa.20135}.

\bibitem[Lüscher(2009)]{luscher_trivializing_2009}
Lüscher, M.
\newblock Trivializing {Maps}, the {Wilson} {Flow} and the {HMC} {Algorithm}.
\newblock \emph{Communications in Mathematical Physics}, 293\penalty0
  (3):\penalty0 899, November 2009.
\newblock ISSN 1432-0916.
\newblock \doi{10.1007/s00220-009-0953-7}.
\newblock URL \url{https://doi.org/10.1007/s00220-009-0953-7}.

\bibitem[Mahmoud et~al.(2022)Mahmoud, Masters, Lee, and
  Lill]{mahmoud_accurate_2022}
Mahmoud, A.~H., Masters, M., Lee, S.~J., and Lill, M.~A.
\newblock Accurate {Sampling} of {Macromolecular} {Conformations} {Using}
  {Adaptive} {Deep} {Learning} and {Coarse}-{Grained} {Representation}.
\newblock \emph{Journal of Chemical Information and Modeling}, 62\penalty0
  (7):\penalty0 1602--1617, April 2022.
\newblock ISSN 1549-9596, 1549-960X.
\newblock \doi{10.1021/acs.jcim.1c01438}.
\newblock URL \url{https://pubs.acs.org/doi/10.1021/acs.jcim.1c01438}.

\bibitem[McNaughton et~al.(2020)McNaughton, Milo{\v{s}}evi{\'{c}}, Perali, and
  Pilati]{McNaughton2020}
McNaughton, B., Milo{\v{s}}evi{\'{c}}, M.~V., Perali, A., and Pilati, S.
\newblock {Boosting Monte Carlo simulations of spin glasses using
  autoregressive neural networks}.
\newblock \emph{Physical Review E}, 101\penalty0 (Mc):\penalty0 1--13, 2020.
\newblock ISSN 24700053.
\newblock \doi{10.1103/PhysRevE.101.053312}.
\newblock URL \url{http://arxiv.org/abs/2002.04292}.

\bibitem[Midgley et~al.(2022)Midgley, Stimper, Simm, Schölkopf, and
  Hernández-Lobato]{Midgley2022}
Midgley, L.~I., Stimper, V., Simm, G. N.~C., Schölkopf, B., and
  Hernández-Lobato, J.~M.
\newblock Flow annealed importance sampling bootstrap, 2022.
\newblock URL \url{https://arxiv.org/abs/2208.01893}.

\bibitem[Miyato et~al.(2018)Miyato, Kataoka, Koyama, and Yoshida]{sngan}
Miyato, T., Kataoka, T., Koyama, M., and Yoshida, Y.
\newblock Spectral normalization for generative adversarial networks.
\newblock In \emph{International Conference on Learning Representations}, 2018.
\newblock URL \url{https://openreview.net/forum?id=B1QRgziT-}.

\bibitem[Mou et~al.(2019)Mou, Ho, Wainwright, Bartlett, and Jordan]{Mou2019}
Mou, W., Ho, N., Wainwright, M.~J., Bartlett, P.~L., and Jordan, M.~I.
\newblock Sampling for bayesian mixture models: Mcmc with polynomial-time
  mixing, 2019.
\newblock URL \url{https://arxiv.org/abs/1912.05153}.

\bibitem[Murray et~al.(2010)Murray, Adams, and MacKay]{Murray2010}
Murray, I., Adams, R., and MacKay, D.
\newblock Elliptical slice sampling.
\newblock In Teh, Y.~W. and Titterington, M. (eds.), \emph{Proceedings of the
  Thirteenth International Conference on Artificial Intelligence and
  Statistics}, volume~9 of \emph{Proceedings of Machine Learning Research},
  pp.\  541--548, Chia Laguna Resort, Sardinia, Italy, 13--15 May 2010. PMLR.
\newblock URL \url{https://proceedings.mlr.press/v9/murray10a.html}.

\bibitem[Müller et~al.(2019)Müller, Mcwilliams, Rousselle, Gross, and
  Novák]{muller_neural_2019}
Müller, T., Mcwilliams, B., Rousselle, F., Gross, M., and Novák, J.
\newblock Neural {Importance} {Sampling}.
\newblock \emph{ACM Transactions on Graphics}, 38\penalty0 (5):\penalty0 1--19,
  October 2019.
\newblock ISSN 0730-0301, 1557-7368.
\newblock \doi{10.1145/3341156}.
\newblock URL \url{https://dl.acm.org/doi/10.1145/3341156}.

\bibitem[Narayanan \& Srivastava(2022)Narayanan and Srivastava]{Narayanan2022}
Narayanan, H. and Srivastava, P.
\newblock On the mixing time of coordinate hit-and-run.
\newblock \emph{Combinatorics, Probability and Computing}, 31\penalty0
  (2):\penalty0 320–332, 2022.
\newblock \doi{10.1017/S0963548321000328}.

\bibitem[Natarovskii et~al.(2021)Natarovskii, Rudolf, and
  Sprungk]{Natarovskii2021}
Natarovskii, V., Rudolf, D., and Sprungk, B.
\newblock Geometric convergence of elliptical slice sampling.
\newblock In Meila, M. and Zhang, T. (eds.), \emph{Proceedings of the 38th
  International Conference on Machine Learning}, volume 139 of
  \emph{Proceedings of Machine Learning Research}, pp.\  7969--7978. PMLR,
  18--24 Jul 2021.
\newblock URL \url{https://proceedings.mlr.press/v139/natarovskii21a.html}.

\bibitem[Neal et~al.(2011)]{Neal2011}
Neal, R.~M. et~al.
\newblock Mcmc using hamiltonian dynamics.
\newblock \emph{Handbook of markov chain monte carlo}, 2\penalty0
  (11):\penalty0 2, 2011.

\bibitem[Nicoli et~al.(2022)Nicoli, Anders, Funcke, Hartung, Jansen, Kessel,
  Nakajima, and Stornati]{nicoli_machine_2022}
Nicoli, K.~A., Anders, C., Funcke, L., Hartung, T., Jansen, K., Kessel, P.,
  Nakajima, S., and Stornati, P.
\newblock Machine {Learning} of {Thermodynamic} {Observables} in the {Presence}
  of {Mode} {Collapse}.
\newblock In \emph{Proceedings of {The} 38th {International} {Symposium} on
  {Lattice} {Field} {Theory} — {PoS}({LATTICE2021})}, pp.\  338, May 2022.
\newblock \doi{10.22323/1.396.0338}.
\newblock URL \url{http://arxiv.org/abs/2111.11303}.
\newblock arXiv:2111.11303 [hep-lat].

\bibitem[Nijkamp et~al.(2021)Nijkamp, Gao, Sountsov, Vasudevan, Pang, Zhu, and
  Wu]{Nijkamp2020}
Nijkamp, E., Gao, R., Sountsov, P., Vasudevan, S., Pang, B., Zhu, S.-C., and
  Wu, Y.~N.
\newblock Mcmc should mix: Learning energy-based model with neural transport
  latent space mcmc.
\newblock In \emph{International Conference on Learning Representations}, 2021.

\bibitem[No{\'{e}} et~al.(2019)No{\'{e}}, Olsson, K{\"{o}}hler, and
  Wu]{Noe2019}
No{\'{e}}, F., Olsson, S., K{\"{o}}hler, J., and Wu, H.
\newblock {Boltzmann generators: Sampling equilibrium states of many-body
  systems with deep learning}.
\newblock \emph{Science}, 365\penalty0 (6457), 2019.
\newblock ISSN 10959203.
\newblock \doi{10.1126/science.aaw1147}.
\newblock URL \url{https://www.science.org/doi/10.1126/science.aaw1147}.

\bibitem[Papamakarios et~al.(2017)Papamakarios, Pavlakou, and
  Murray]{Papamakarios2017}
Papamakarios, G., Pavlakou, T., and Murray, I.
\newblock Masked autoregressive flow for density estimation.
\newblock In Guyon, I., Luxburg, U.~V., Bengio, S., Wallach, H., Fergus, R.,
  Vishwanathan, S., and Garnett, R. (eds.), \emph{Advances in Neural
  Information Processing Systems}, volume~30. Curran Associates, Inc., 2017.
\newblock URL
  \url{https://proceedings.neurips.cc/paper/2017/file/6c1da886822c67822bcf3679d04369fa-Paper.pdf}.

\bibitem[Papamakarios et~al.(2021)Papamakarios, Nalisnick, Rezende, Mohamed,
  and Lakshminarayanan]{Papamakarios2021}
Papamakarios, G., Nalisnick, E., Rezende, D.~J., Mohamed, S., and
  Lakshminarayanan, B.
\newblock Normalizing flows for probabilistic modeling and inference.
\newblock \emph{Journal of Machine Learning Research}, 22\penalty0
  (57):\penalty0 1--64, 2021.
\newblock URL \url{http://jmlr.org/papers/v22/19-1028.html}.

\bibitem[Parno \& Marzouk(2018)Parno and Marzouk]{Parno2018transport_mcmc}
Parno, M.~D. and Marzouk, Y.~M.
\newblock Transport map accelerated markov chain monte carlo.
\newblock \emph{{SIAM}/{ASA} Journal on Uncertainty Quantification}, 6\penalty0
  (2):\penalty0 645--682, jan 2018.
\newblock \doi{10.1137/17m1134640}.
\newblock URL \url{https://doi.org/10.1137%2F17m1134640}.

\bibitem[Rezende \& Mohamed(2015)Rezende and Mohamed]{rezende_variational_2015}
Rezende, D. and Mohamed, S.
\newblock Variational {Inference} with {Normalizing} {Flows}.
\newblock In \emph{Proceedings of the 32nd {International} {Conference} on
  {Machine} {Learning}}, pp.\  1530--1538. PMLR, June 2015.
\newblock URL \url{https://proceedings.mlr.press/v37/rezende15.html}.
\newblock ISSN: 1938-7228.

\bibitem[Robert \& Casella(2005)Robert and Casella]{Robert2005}
Robert, C.~P. and Casella, G.
\newblock \emph{Monte Carlo Statistical Methods (Springer Texts in
  Statistics)}.
\newblock Springer-Verlag, Berlin, Heidelberg, 2005.
\newblock ISBN 0387212396.

\bibitem[Roberts \& Rosenthal(2011)Roberts and Rosenthal]{Roberts2011}
Roberts, G.~O. and Rosenthal, J.~S.
\newblock Quantitative non-geometric convergence bounds for independence
  samplers.
\newblock \emph{Methodology and Computing in Applied Probability}, 13\penalty0
  (2):\penalty0 391--403, Jun 2011.
\newblock ISSN 1573-7713.
\newblock \doi{10.1007/s11009-009-9157-z}.
\newblock URL \url{https://doi.org/10.1007/s11009-009-9157-z}.

\bibitem[Roberts \& Tweedie(1996)Roberts and Tweedie]{Roberts1996}
Roberts, G.~O. and Tweedie, R.~L.
\newblock Exponential convergence of langevin distributions and their discrete
  approximations.
\newblock \emph{Bernoulli}, 2\penalty0 (4):\penalty0 341--363, 1996.
\newblock ISSN 13507265.
\newblock URL \url{http://www.jstor.org/stable/3318418}.

\bibitem[Rubinstein \& Kroese(2017)Rubinstein and
  Kroese]{rubinstein_simulation_2017}
Rubinstein, R.~Y. and Kroese, D.~P.
\newblock \emph{Simulation and the {Monte} {Carlo} method}.
\newblock Wiley series in probability and statistics. John Wiley \& Sons, Inc,
  Hoboken, New Jersey, third edition edition, 2017.
\newblock ISBN 978-1-118-63220-8 978-1-118-63238-3.

\bibitem[Samsonov et~al.(2022)Samsonov, Lagutin, Gabri{\'e}, Durmus, Naumov,
  and Moulines]{Samsonov2022}
Samsonov, S., Lagutin, E., Gabri{\'e}, M., Durmus, A., Naumov, A., and
  Moulines, E.
\newblock Local-global mcmc kernels: the best of both worlds.
\newblock In \emph{Advances in Neural Information Processing Systems}, 2022.

\bibitem[Stimper et~al.(2022)Stimper, Midgley, Simm, Schölkopf, and
  Hernández-Lobato]{vincent_stimper_2022_6993124}
Stimper, V., Midgley, L.~I., Simm, G. N.~C., Schölkopf, B., and
  Hernández-Lobato, J.~M.
\newblock Alanine dipeptide in an implicit solvent at 300k, August 2022.
\newblock URL \url{https://doi.org/10.5281/zenodo.6993124}.

\bibitem[Tabak \& Vanden-Eijnden(2010)Tabak and Vanden-Eijnden]{Tabak2010}
Tabak, E.~G. and Vanden-Eijnden, E.
\newblock {Density estimation by dual ascent of the log-likelihood}.
\newblock \emph{Communications in Mathematical Sciences}, 8\penalty0
  (1):\penalty0 217--233, 2010.
\newblock ISSN 15396746.
\newblock \doi{10.4310/CMS.2010.v8.n1.a11}.

\bibitem[Tjelmeland(2004)]{Tjelmeland2004}
Tjelmeland, H.
\newblock Using all metropolis--hastings proposals to estimate mean values.
\newblock Technical report, 2004.

\bibitem[Tokdar \& Kass(2010)Tokdar and Kass]{Tokdar2010}
Tokdar, S.~T. and Kass, R.~E.
\newblock Importance sampling: a review.
\newblock \emph{WIREs Computational Statistics}, 2\penalty0 (1):\penalty0
  54--60, 2010.
\newblock \doi{https://doi.org/10.1002/wics.56}.
\newblock URL
  \url{https://wires.onlinelibrary.wiley.com/doi/abs/10.1002/wics.56}.

\bibitem[Vempala(2005)]{Vempala2005}
Vempala, S.
\newblock Geometric random walks: a survey.
\newblock \emph{Combinatorial and computational geometry}, 52\penalty0
  (573-612):\penalty0 2, 2005.

\bibitem[Wang(2022)]{Wang2022}
Wang, G.
\newblock {Exact convergence analysis of the independent Metropolis-Hastings
  algorithms}, 2022.
\newblock URL \url{https://doi.org/10.3150/21-BEJ1409}.

\bibitem[Wirnsberger et~al.(2020)Wirnsberger, Ballard, Papamakarios,
  Abercrombie, Racanière, Pritzel, Jimenez~Rezende, and
  Blundell]{wirnsberger_targeted_2020}
Wirnsberger, P., Ballard, A.~J., Papamakarios, G., Abercrombie, S., Racanière,
  S., Pritzel, A., Jimenez~Rezende, D., and Blundell, C.
\newblock Targeted free energy estimation via learned mappings.
\newblock \emph{The Journal of Chemical Physics}, 153\penalty0 (14):\penalty0
  144112, October 2020.
\newblock ISSN 0021-9606, 1089-7690.
\newblock \doi{10.1063/5.0018903}.
\newblock URL \url{https://aip.scitation.org/doi/10.1063/5.0018903}.

\bibitem[Wong et~al.(2022)Wong, Gabrié, and Foreman-Mackey]{wong_flowmc_2022}
Wong, K. W.~K., Gabrié, M., and Foreman-Mackey, D.
\newblock {flowMC}: {Normalizing}-flow enhanced sampling package for
  probabilistic inference in {Jax}, November 2022.
\newblock URL \url{http://arxiv.org/abs/2211.06397}.
\newblock arXiv:2211.06397 [astro-ph].

\bibitem[Wu et~al.(2022)Wu, Schmidler, and Chen]{Wu2021}
Wu, K., Schmidler, S., and Chen, Y.
\newblock Minimax mixing time of the metropolis-adjusted langevin algorithm for
  log-concave sampling.
\newblock \emph{Journal of Machine Learning Research}, 23\penalty0
  (270):\penalty0 1--63, 2022.
\newblock URL \url{http://jmlr.org/papers/v23/21-1184.html}.

\bibitem[Yang \& Liu(2021)Yang and Liu]{Yang2021}
Yang, X. and Liu, J.~S.
\newblock Convergence rate of multiple-try metropolis independent sampler,
  2021.
\newblock URL \url{https://arxiv.org/abs/2111.15084}.

\end{thebibliography}
